\documentclass[10pt,journal,compsoc]{IEEEtran}

\usepackage{amsmath}
\usepackage{booktabs} 
\usepackage{amssymb}
\usepackage{nicefrac}       
\usepackage{microtype}      
\usepackage{xspace}
\usepackage{enumitem}
\usepackage{subcaption}
\usepackage{algorithm}
\usepackage{algpseudocode} 
\usepackage{graphicx}
\usepackage[numbers]{natbib}
\usepackage{amsthm}
\usepackage{xcolor}
\usepackage{color, colortbl} 

\definecolor{LightCyan}{rgb}{0.88,1,0.88}

\definecolor{DarkGreen}{rgb}{0.0,0.3,0.0}
\definecolor{DarkGreen2}{rgb}{0.0,0.7,0.0}
\usepackage[colorlinks=true,linkcolor=DarkGreen2,filecolor=DarkGreen2,urlcolor=DarkGreen,citecolor=DarkGreen2,bookmarks=false,breaklinks=true]{hyperref}

\newtheorem{theo}{Theorem}

\usepackage{multirow}

\makeatletter
\DeclareRobustCommand\onedot{\futurelet\@let@token\bmv@onedotaux}
\def\bmv@onedotaux{\ifx\@let@token.\else.\null\fi\xspace}
\def\eg{\emph{e.g}\onedot} 
\def\ie{\emph{i.e}\onedot} 
 
\def\etc{\emph{etc}\onedot} 
\def\wrt{w.r.t\onedot}

\def\aka{a.k.a\onedot}

\makeatother

\newcommand{\ZB}{{}}
\newcommand{\ZBL}{{}}

\usepackage{multibib}
\newcites{latex}{References}

\makeatletter
\def\ps@myheadings{%
    \let\@oddfoot\@empty\let\@evenfoot\@empty
    \def\@evenhead{\thepage\hfil\slshape\leftmark}%
    \def\@oddhead{{\slshape\rightmark}\hfil\thepage}%
    \let\@mkboth\@gobbletwo
    \let\sectionmark\@gobble
    \let\subsectionmark\@gobble
    }
  \if@titlepage
  \renewcommand\maketitle{\begin{titlepage}%
  \let\footnotesize\small
  \let\footnoterule\relax
  \let \footnote \thanks
  \null\vfil
  \vskip 60\p@
  \begin{center}%
    {\LARGE \@title \par}%
    \vskip 3em%
    {\large
     \lineskip .75em%
      \begin{tabular}[t]{c}%
        \@author
      \end{tabular}\par}%
      \vskip 1.5em%
    {\large \@date \par}
  \end{center}\par
  \@thanks
  \vfil\null
  \end{titlepage}%
  \setcounter{footnote}{0}%
}
\else
\renewcommand\maketitle{\par
  \begingroup
    \renewcommand\thefootnote{\@fnsymbol\c@footnote}%
    \def\@makefnmark{\rlap{\@textsuperscript{\normalfont\color{black}
    }}}%
    \long\def\@makefntext##1{\parindent 1em\noindent
            \hb@xt@1.8em{%
                \hss\@textsuperscript{\normalfont
                }}##1}%
    \if@twocolumn
      \ifnum \col@number=\@ne
        \@maketitle
      \else
        \twocolumn[\@maketitle]%
      \fi
    \else
      \newpage
      \global\@topnum\z@   
      \@maketitle
    \fi
    \thispagestyle{plain}
    \ifx\myappendix\undefined
    \@thanks
    \fi
  \endgroup
  \setcounter{footnote}{0}%
}
\makeatother

\begin{document}

\title{Exploiting Field Dependencies for Learning on Categorical Data}

\author{Zhibin~Li, 
        Piotr Koniusz$^*$, 
        Lu Zhang, 
        Daniel Edward Pagendam, 
        Peyman Moghadam
\IEEEcompsocitemizethanks{\IEEEcompsocthanksitem 
\vspace{-0.2cm}Z. Li is a research fellow for Machine Learning \& Artificial Intelligence Future Science Platform, CSIRO, Brisbane, Australia. E-mail: lzb5600@gmail.com
\protect\\
\vspace{-0.3cm}
\IEEEcompsocthanksitem P. Koniusz, D. E. Pagendam and P. Moghadam are with Data61, CSIRO, Australia. E-mails: firstname.lastname@data61.csiro.au\protect\\
\vspace{-0.3cm}
\IEEEcompsocthanksitem P. Koniusz ($^*\!$corresponding author: piotr.koniusz@data61.csiro.au) is also with the Australian National University, Canberra, Australia.\protect\\
\vspace{-0.3cm}
\IEEEcompsocthanksitem P. Moghadam is also with the Queensland University of Technology, Brisbane, Australia. 
\protect\\
\vspace{-0.3cm}
\IEEEcompsocthanksitem L. Zhang is with the Queensland Brain Institute, University of Queensland, Brisbane, Australia. E-mail: l.zhang3@uq.edu.au}
\vspace{-0.1cm}
\thanks{Code: \url{https://github.com/csiro-robotics/MDL}}
\thanks{Manuscript submitted in Jun 2022. Manuscript  accepted by the IEEE Transactions on Pattern Analysis and Machine Intelligence in July 2023.}
\vspace{-0.3cm}
}

\markboth{Journal of \LaTeX\ Class Files,~Vol.~14, No.~8, June~2022}%
{Shell \MakeLowercase{\textit{Z. Li et al.}}: Exploiting Field Dependencies for Learning on Categorical Data}

\IEEEtitleabstractindextext{%
\begin{abstract}
Traditional approaches for learning on categorical data underexploit the dependencies between columns (\aka fields) in a dataset because they rely on the embedding of data points driven alone by the classification/regression loss. In contrast, we propose a novel method for learning on categorical data with the goal of exploiting dependencies between fields. Instead of modelling statistics of features globally (\ie, by the covariance matrix of features), we learn a global field dependency matrix that captures dependencies between fields and then we refine the global field dependency matrix at the instance-wise level with different weights (so-called local dependency modelling) \wrt each field to improve the modelling of the field dependencies. Our algorithm exploits the meta-learning paradigm, \ie, the dependency matrices are refined in the inner loop of the meta-learning algorithm without the use of labels, whereas the outer loop intertwines the updates of the embedding matrix (the matrix performing projection) and global dependency matrix in a supervised fashion (with the use of labels). Our method is simple yet it outperforms several state-of-the-art methods on six popular dataset benchmarks. Detailed ablation studies provide additional insights into our method. 
\end{abstract}

\begin{IEEEkeywords}
Categorical data, meta-learning, embedding, representation learning.
\end{IEEEkeywords}}
\maketitle
\IEEEpeerreviewmaketitle

\IEEEraisesectionheading{\section{Introduction}\label{sec:intro}}

\IEEEPARstart{L}{earning} with categorical data is one of the fundamental tasks in many applications. For example, in on-line advertising, Click-Through Rate (CTR) prediction relies on the categorical data~\cite{chapelle2014simple}. Other applications such as the web search~\cite{agichtein2006learning}, recommender systems~\cite{koren2009matrix}  and even disease diagnosis \cite{LIU2023120267} also depend heavily on the categorical data.  As presented in Table~\ref{data_exm}, the categorical data can be organised into different columns \aka {\it~fields}, \eg, Shoes, Lipstick and Computer are the categorical features under the field ``Product''. With the growing number of instances, the number of categorical features often increases to a problematic scale.

\begin{table}[httb!]
\caption{Example of categorical data for ad-click prediction.}
\label{data_exm}
\centering
\begin{tabular}{c|ccc}
\toprule
Clicked &  Gender & Advertiser & Product  \\ \midrule
No      &  Male   & Nike       & Shoes    \\
No      &  Male   & L'Oréal    & Lipstick     \\
Yes     &  Female & Apple      & Computer \\ \bottomrule
\end{tabular}
\end{table}

Categorical data poses a fundamental challenge regarding how to embed the discrete categorical features into some expressive latent (feature) space suitable for the learning and inference steps.  In supervised learning, the usual approach is to assign an embedding vector \aka \textit{embedding} to each categorical feature, and then feed these vectors into subsequent modules, \eg, Multi-Layer Perceptron (MLP)~\cite{cheng2016wide}. Categorical features are firstly transformed into binary features through one-hot encoding and then used to select embedding vectors from the so-called {\it embedding matrix}\footnote{The embedding matrix performs embedding (projection). Note that  multiplying such a matrix with a one-hot vector simply selects  from the matrix a vector corresponding to non-zero element of one-hot vector.} Due to the sparsity of categorical features, for each instance, only the embedding vectors corresponding to the non-zero features (one per field) will form the input to the model.  The embedding matrix that performs embedding is optimised together with other model parameters in an end-to-end manner. However, the quality of embeddings is typically driven by  classification or regression loss, whereas our approach seeks to explicitly model and exploit dependencies between fields. Meta-learning \cite{finn2017model,metz2018meta} offers a unique opportunity to (i) exploit instance-wise dependencies between fields which can also help optimise the embeddings for the supervised learning task, (ii) perform the unsupervised instance-wise dependency refinement not only in training but also at the test stage to adapt to the test data without the use of test label we care to predict.

\vspace{0.1cm}
\noindent\textbf{Motivation.}
Features in a dataset are often correlated, \ie, in Table~\ref{data_exm}, Shoes and Nike are correlated with each other. Numerous prediction tasks~\cite{adc}, \eg, ad-click, demographics, app preference prediction and intent analysis \cite{ZhangSZXLYY22} exhibit such field correlations. Indeed, any column in Table~\ref{data_exm} can be chosen for the target prediction, whereas remaining columns may serve as input features. It is intuitive to assume that these columns share some dependencies which help infer values of one column from the remaining columns. Such an observation inspires us to explore such dependencies with the goal of improving learning/inference on categorical data.

Modelling feature dependencies is notoriously hard and costly for sparse high-dimensional data~\cite{pourahmadi2013high,Lin_2018_ECCV,sfa_yifei,simon_isice}. The quality of dependencies also varies across different instances and prediction tasks. To address these issues, we design a method that can: (i) model the field dependencies instead of feature dependencies and (ii) refine the field dependencies for a better fit to each instance and the specific task. In what follows, an efficient algorithm for learning the dependencies is proposed. To this end, we employ meta-learning to optimise the so-called dependency loss. This process is supervised by the task-driven loss (\eg, logistic loss for binary classification) with the goal of finding a global field dependency matrix expressing dependencies between fields globally and an embedding matrix whose role is to perform embedding. In doing so, we learn a global field dependency matrix for all instances which can further be refined to capture the field dependencies per instance (so-called local dependencies) and help reduce the task-specific loss. We note that our meta-learning based algorithm can rapidly adapt to the new incoming test data (test tasks) during the inference stage by simply performing a few of gradient descent steps \wrt the dependency loss, which does not rely on target labels. {\ZBL The use of global field dependencies greatly eases  learning of categorical feature dependencies as global field dependencies can be captured by significantly fewer parameters compared to categorical feature dependencies. Learning categorical feature dependencies is particularly hard for sparse and high-dimensional data resulting in overfitting and high computational and memory costs. 
However,  the global field dependencies may fail to capture  finer relations compared to the categorical feature dependencies. 
In our case, the deficiency of global field dependencies in modelling finer-grained feature dependencies is alleviated by the meta-learned fine-tuning scheme. Meta-learning  
helps fine-tune from the global field dependencies level down to the instance-wise feature dependencies, which helps the task-driven (supervised) loss. The global field dependencies and tuning process are performed jointly by meta-learning  to optimally help the task-driven loss.}


\vspace{0.2cm}
In summary, our contributions are three-fold:

\vspace{0.1cm}
\renewcommand{\labelenumi}{\roman{enumi}.}
\begin{enumerate}[leftmargin=0.6cm]
\item To the best of our knowledge, we are the first to explicitly exploit field dependencies for supervised learning on large-scale categorical data. To this end, we propose to capture instance-specific field dependencies (local dependencies) by an efficient optimisation algorithm.
\item We design an effective meta-learning algorithm for learning an optimally-dependent embedding matrix and field dependencies that are refined \wrt different instances. This also enables us to utilise the dependencies in an unsupervised manner during the test stage.
\item We conduct extensive experiments on six datasets of varying scale to verify the effectiveness of our method, with ablation studies providing additional insights.
\end{enumerate}

\section{Methodology}
In this section, we briefly introduce the background required to understand our method, \eg, the embedding step. Next, we provide details of dependency modelling among fields. Finally, we show how dependency modelling facilitates the subsequent learning tasks via our meta-learning algorithm.
\begin{figure*}
  \centering
  \includegraphics[width=1.0\linewidth]{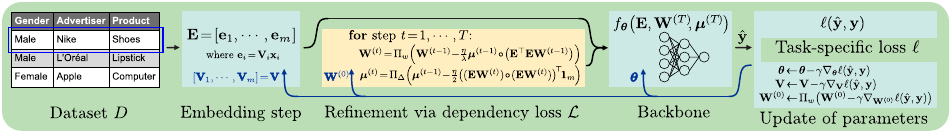}
  \caption{The proposed framework. The inner and outer loops of Algorithm \ref{algo} are indicated by yellow and blue panels. Blue lines carry parameters of the model updated in the outer loop. See equations of Embedding step \eqref{ebd}, Refinement step \eqref{grdupdtw}, \eqref{grdupdtmu} via the dependency loss $\mathcal{L}$, Backbone \eqref{bkb} and the Task-specific loss \eqref{expctedloss} with Update of parameters (Step 9 of Alg. \ref{algo}).}\label{framework}
\end{figure*}
\subsection{Background}
For the $m$-field data, each instance can be represented as a sparse vector: 
$$\mathbf{x}=[\mathbf{x}_1;\mathbf{x}_2;\cdots;\mathbf{x}_m]\in\mathbb{R}^{d_1+\cdots+d_m},$$
where each $\mathbf{x}_i\in\mathbb{R}^{d_i}$ is a one-hot-encoding column vector for $i$-th field, and $d_i$ denotes cardinality of $i$-th field. Stacking all vectors $\mathbf{x}_1,\cdots,\mathbf{x}_m$ vertically we get the one-hot-encoding vector $\mathbf{x}\in\mathbb{R}^{d_1+\cdots+d_m}$. Each $\mathbf{x}_i$ is usually transformed into a dense representation $\mathbf{e}_i\in\mathbb{R}^{k}$ by an associated embedding matrix $\mathbf{V}_i\in\mathbb{R}^{k\times d_i}$ and obtain:
\begin{equation}\label{ebd}
    \mathbf{E} = [\mathbf{e}_1,\mathbf{e}_2,\cdots,\mathbf{e}_m]\in\mathbb{R}^{k\times m},\;\; \mathbf{e}_i=\mathbf{V}_i\mathbf{x}_i, \forall i=1,2,\cdots,m,
\end{equation}
where the embedding $\mathbf{E}$ is obtained through the  {\it embedding matrix} $\mathbf{V}=[\mathbf{V}_1,\mathbf{V}_2,\cdots,\mathbf{V}_m]\in\mathbb{R}^{k\times(d_1+\cdots+d_m)}$. The embedding dimension $k$ is a hyper-parameter. The embedding $\mathbf{E}$ is then feed into subsequent modules and the embedding matrix $\mathbf{V}$ is learned altogether with other parameters in subsequent modules.   

The embedding process in \eqref{ebd} can be generalised for data consisting of not just categorical features but also  numerical features, \eg, one may apply $x_i\in\mathbb{R}^1$ and $\mathbf{V}_i\in\mathbb{R}^{k\times 1}$ if $i$-th field consists of numerical feature $x_i$ rather than vector $\mathbf{x}_i$. 

\subsection{Modelling the Field Dependencies}

It has been shown that categorical features enjoy dependency which could be used to improve learning and inference~\cite{adc}. Thus, we propose to model the dependencies among fields. More precisely, we model dependencies among dense embeddings obtained through \eqref{ebd} as embeddings capturing useful dependencies are more informative and helpful for the subsequent learning tasks. We use a simple loss function to 
express the $k$-th field of instance $\mathbf{x}$ as a weighted linear combination of remaining fields of instance $\mathbf{x}$. We have:
\begin{equation}\label{singleloss}
    \mathcal{L}_k(\mathbf{w}_k)  = \frac{1}{2}||\mathbf{e}_k - \sum_{i\neq k}   w_{ik}\mathbf{e}_i||_2^2,
\end{equation}
where index $i\!\in\![1,m]\!\setminus\!\{k\}$, set $[1,m]\equiv\{1,\cdots,m\}$ and $||\cdot||_2$ is the L2 norm.
As such, the smaller this loss function is, the stronger dependency between $\mathbf{e}_k$ and $\{\mathbf{e}_i| i\neq k\}$. Each $w_{ik}$ is a learnable coefficient of vector $\mathbf{w}_k$ except we let $w_{kk} = -1$. By adopting a weighted combination of \eqref{singleloss}, we obtain the dependency capturing loss \aka \textbf{dependency loss}:
\begin{equation}\label{u_loss}
\begin{aligned}
    \mathcal{L}( \mathbf{W}, \boldsymbol{\mu} ) &=\frac{1}{\lambda}\sum_{k=1}^{m} \mu_k \mathcal{L}_k(\mathbf{w}_k), 
\end{aligned}
\end{equation}
with $\mathbf{W} = [\mathbf{w}_1,\cdots,\mathbf{w}_m]\in\mathbb{R}^{m\times m}$. Moreover, when we minimize \eqref{u_loss}, we impose constraints $\mu_k \geq 0, \forall k \in[1,m]$ and $ \sum_{k=1}^m \mu_k=\lambda$. 
In this paper, we refer to $\mathbf{W}$ as the \textbf{dependency matrix}. The vector $\boldsymbol{\mu} = [\mu_1,\cdots,\mu_m]^\top$ contains  non-negative weight coefficients to be learned from data (one may think of them as attention scores stored in a vector). Hyper-parameter $\lambda>0$ determines the divergence of $\boldsymbol{\mu}/\lambda$ from a uniform distribution $\frac{1}{m}\mathbf{1}_m$ (vector containing $m$ coefficients, each equal $\frac{1}{m}$). This could be interpreted as a prior which prevents the dependency loss from overfitting to a single field (one extreme) and overreliance on a dense combination of all fields (another extreme). In proof of Theorem~\ref{theorem1} in Section \ref{optmsec}, we explain how $\lambda$ impacts this prior  during optimisation. In what follows, for brevity we often write  $\mathcal{L}_k$ and $\mathcal{L}$ in place of $\mathcal{L}_k( \mathbf{w}_k )$ and  $\mathcal{L}( \mathbf{W}, \boldsymbol{\mu} )$ respectively. 

Notice that the dependency loss $\mathcal{L}$ can be rewritten into a simpler matrix-based form:

\begin{equation}\label{siml}
\begin{aligned}
    \mathcal{L} = & \frac{1}{2\lambda} \langle\boldsymbol{\mu}\circ(\mathbf{E} \mathbf{W}), \mathbf{E} \mathbf{W}\rangle, 
\end{aligned}
\end{equation}
with the following constraints imposed when minimizing \eqref{siml}:
\begin{equation}
\label{constr}
\begin{aligned}
                & w_{kk} = -1, \forall k \in[1,m],\\
               & \mu_k \geq 0, \forall k \in[1,m],\\
               & \sum_{k=1}^m \mu_k=\lambda.
\end{aligned}
\end{equation}

Symbol $\circ$ in  \eqref{siml} is the element-wise product with broadcast. Specifically, we multiply $k$-th column of $(\mathbf{E} \mathbf{W})$ by $\mu_k$. Moreover, $\langle\cdot,\cdot\rangle$ denotes the inner product, and 
$\lambda>0$ in \eqref{siml} 
has the same role as already defined for \eqref{u_loss}.

\subsection{Optimising the Field Dependencies for the Supervised Task with Meta-Dependency Learning}\label{optmsec}

For simplicity, we use supervised classification task as an example to introduce our method. Using the dependency loss as a regularisation term in addition to the classification loss could be sub-optimal because it is hard to find a good stable trade-off between such both terms. Moreover, relying solely on a global dependency matrix $\mathbf{W}$ to capture field dependencies for all instances is sub-optimal as verified by our empirical study as
such a dependency matrix disregards instance-wise feature variations in each field. On the other hand, simply assuming different dependency matrices for difference instances involves a huge number of parameters, which would result in poor generalisation to unseen data due to overfitting.

In order to address these issues, we design a meta-learning algorithm that  refines the dependency matrix instance-wise via the dependency loss whose role is to transfer unsupervised knowledge with the goal of improving the effectiveness of classification task.

\subsubsection{Projected Gradient Descent for Refining the Field Dependencies}
A \textbf{global dependency matrix} $\mathbf{W}^{(0)}$  serves as the initial state of the dependency matrices for all instances. We use a meta-learning algorithm from Section~\ref{metal} to learn this global dependency matrix. Subsequently, starting from the global dependency matrix $\mathbf{W}^{(0)}$, for every instance, we update the dependency matrix over a fixed number of gradient descent steps  on the dependency loss $\mathcal{L}$. Notice that the number of gradient descent steps is a hyper-parameter. Specifically, we do not minimise the dependency loss $\mathcal{L}$ in full to prevent overfitting to individual instances.   

We use the projected gradient descent \cite{levitin1966constrained} for optimising $\mathcal{L}$ \wrt $\mathbf{W}$ and $\boldsymbol{\mu}$ due to the constraints \eqref{constr} imposed on \eqref{siml}. 
The projected gradient descent step $t$ for updating the dependency matrix is given as:
\begin{equation}\label{grdupdtw}
    \mathbf{W}^{(t)} =  \Pi_w\Big(\mathbf{W}^{(t-1)}  - \frac{\eta}{\lambda}    \boldsymbol{\mu}^{(t-1)}\circ \big(\mathbf{E}^\top \mathbf{E} \mathbf{W}^{(t-1)}\big) \Big),
\end{equation}
where $\Pi_w(\cdot)$ defines the projection that sets the diagonal elements of its argument  $\mathbf{W}^{(t-1)} - \frac{\eta}{\lambda}    \boldsymbol{\mu}^{(t-1)}\circ \big(\mathbf{E}^\top \mathbf{E} \mathbf{W}^{(t-1)}\big)$ to $-1$ and $\eta$ is the learning rate. The optimisation starts from the global dependency matrix $\mathbf{W}^{(0)}$ and $\boldsymbol{\mu}^{(0)}$ in step $1$.

When optimising \eqref{siml} \wrt $\boldsymbol{\mu}$, one can  see that setting $\mu_i=\lambda$ and  $\mu_j\!=\!0,\!\forall j\!\neq\!i$  is optimal if $i=\arg\min_k \mathcal{L}_k$. This however results in the sparsest solution leading to the case where only a single column of the dependency matrix $\mathbf{W}$ being refined in subsequent steps, which is an oversimplification of underlying dependency pattern. In order to flexibly refine $\boldsymbol{\mu}$, we use following projected gradient descent step:
\begin{equation}\label{grdupdtmu}
    \boldsymbol{\mu}^{(t)} =  \Pi_\Delta\Big(\boldsymbol{\mu}^{(t-1)}  - \frac{\eta}{2}  \big((\mathbf{E}\mathbf{W}^{(t)})\circ(\mathbf{E}\mathbf{W}^{(t)})\big)^\top \mathbf{1}_m\Big).
\end{equation}
The optimisation starts from $\boldsymbol{\mu}^{(0)}=\mathbf{1}_m$ where $\mathbf{1}_m$ is a vector of length $m$ containing all coefficients equal one. $\Pi_\Delta(\cdot)$ defines the projection onto a simplex $\Delta = \{\boldsymbol{\mu}|\mu_k \geq 0,\forall k\in[1,m], \sum_{k=1}^m \mu_k=\lambda\}$ where $\lambda$ is a hyper-parameter controlling the flexibility of $\boldsymbol{\mu}$ (see the proof of Theorem~\ref{theorem1} for details). We also multiply the learning rate by $\lambda$ compared to updates of $\mathbf{W}$ in \eqref{grdupdtw} for faster convergence. Let
\begin{equation}
\hat{\boldsymbol{\mu}}^{(t)} = \boldsymbol{\mu}^{(t-1)}  - \frac{\eta}{2}  \big((\mathbf{E}\mathbf{W}^{(t-1)})\circ(\mathbf{E}\mathbf{W}^{(t-1)})\big)^\top \mathbf{1}_m, 
\end{equation}
then the projection $\Pi_\Delta(\hat{\boldsymbol{\mu}}^{(t)})$ is obtained by solving the following problem (based on the definition of projected gradient descent):
\begin{equation}\label{qpp}
\begin{aligned}
\boldsymbol{\mu}^{(t)} = \Pi_\Delta\big(\hat{\boldsymbol{\mu}}^{(t)}\big)=&
{\arg\min}_{\boldsymbol{\mu} \in \mathbb{R}^{m}}  \frac{1}{2}\|\hat{\boldsymbol{\mu}}^{(t)}-\boldsymbol{\mu}^{(t)}\|_2^{2}  \\
 & \text { s.t. } \mu_k \geq 0, \forall k \in[1,m],\\
               & \sum_{k=1}^m \mu_k=\lambda,
\end{aligned}
\end{equation}
where \eqref{qpp} is a Quadratic Programming problem that can be solved via  existing solvers~\cite{mosek,cplex2009v12}. However, most of solvers involve an iterative process with some approximation error. We provide a closed-form solution to the above projection problem in Algorithm~\ref{algproj}. This enables us to incorporate such a problem seamlessly and efficiently into many existing methods based on automatic differentiation and GPU acceleration~\cite{paszke2019pytorch}. A closed-form solution to a similar problem was given in previous work~\cite{laha2018controllable} where the constraint involves an additional L2 norm loss term over the simplex. 
\begin{theo}\label{theorem1}
The solution to the optimisation problem \eqref{qpp} is given by Algorithm~\ref{algproj} in log-linear time \wrt $m$.
\end{theo}
\begin{proof}
Our proof exploits the fact that \eqref{qpp} is a strictly-convex optimisation problem with affine constraints, and therefore, Slater's condition holds. In this case, the KKT conditions provide necessary and sufficient conditions for optimality.

The Lagrangian of the problem is:
\begin{equation}
L(\boldsymbol{\mu}, \boldsymbol{\alpha}, \beta)=\frac{1}{2}\|\hat{\boldsymbol{\mu}}-\boldsymbol{\mu}\|_2^{2} - \sum_{k=1}^m \alpha_k \mu_{k} -  \beta \Big(\sum_{k=1}^m   \mu_{k}-\lambda\Big),
\end{equation}
where we omit the superscript $(t)$ for brevity. 

The corresponding KKT conditions \wrt optimal solutions ($\boldsymbol{\mu}^{*}$, $\boldsymbol{\alpha}^{*}$, $\beta^{*}$) are (with applicable indices):
\begin{align}
\hat\mu_k - \mu_k^{*} - \alpha_k^{*}- \beta^{*} &=0, \label{kt1}\\  
\alpha_k^{*} & \geq 0,  \label{kt2}\\ 
\alpha_k^{*} \mu_k^{*} &=0,  \label{kt3} \\ 
\sum_{k=1}^m   \mu_{k}^{*}-\lambda &= 0 ,  \label{kt4}\\ 
\mu_{k}^{*} & \geq 0.   \label{kt5} 
\end{align}
According to the above KKT conditions, the following statements hold:
\begin{equation}
    \mu_{k}^{*} >0 \overset{\eqref{kt3}}{\implies} \alpha_k^{*} = 0 \overset{\eqref{kt1}}{\implies} \mu_k^{*} = \hat\mu_k-\beta^{*}>0. \label{kt6}
\end{equation}
We sort elements of $\hat{\boldsymbol{\mu}}$ into descending order $\hat\mu_{(1)}\geq\hat\mu_{(2)}\geq\cdots\geq\hat\mu_{(m)}$. From \eqref{kt4} we obtain:
\begin{equation}
    \beta^{*} =  \frac{\sum_{k=1}^K   \hat\mu_{(k)}-\lambda}{K}, \label{kt7}
\end{equation}
where $K$ in \eqref{kt7} denotes the cardinality of set $\{k|\hat\mu_k-\beta^{*}>0\}$. Using the above conditions, $K$ can be calculated by:
\begin{equation}
    K = \max \Big(\Big\{i|\hat\mu_{(i)} - \frac{\sum_{k=1}^i   \hat\mu_{(k)}-\lambda}{i} >0\Big\}\Big). 
\end{equation}
Then $\boldsymbol{\mu}^{*}$ can be calculated accordingly after obtaining $K$.

From above derivation we can see that for $0<\lambda<\infty$, elements of $\boldsymbol{\mu}^{*}$ are non-decreasing. For $\lambda\rightarrow 0^+$, the solution $\boldsymbol{\mu}^{*}$ tends to be sparse, and for $\lambda\rightarrow \infty$, the solution tends to be a vector with all its coefficients being equal. The scaling effect of $\lambda$ is cancelled out by $\frac{1}{\lambda}$ in \eqref{siml}. We  also show in the ablation study that as $\lambda$ grows, $\boldsymbol{\mu}^{*}/\lambda$ tends to be closer to uniform distribution $\frac{1}{m}\mathbf{1}_m$.

The time complexity is dominated by 
sorting the elements of $\hat{\boldsymbol{\mu}}$, resulting in $\mathcal{O}(m\log(m))$ time complexity. The above details complete the proof.
\end{proof}

\begin{algorithm}
	\caption{ Projection of $\hat{\boldsymbol{\mu}}$ onto a simplex $\Delta = \{\boldsymbol{\mu}|\mu_k \geq 0,\forall k\in[1,m], \sum_{k=1}^m \mu_k = \lambda\}$ to obtain $\boldsymbol{\mu}$. }  \label{algproj}
    \textbf{Input}: $\hat{\boldsymbol{\mu}}\in\mathbb{R}^{m}$, $\lambda$. \\
    \textbf{Output}: $\boldsymbol{\mu} = \Pi_\Delta\big(\hat{\boldsymbol{\mu}}\big)$.
	\begin{algorithmic}[1]
	    \State \parbox[t]{\dimexpr\linewidth-\algorithmicindent}
	    {Sort elements of $\hat{\boldsymbol{\mu}}$ into $\hat\mu_{(1)}\geq\hat\mu_{(2)}\geq\cdots\geq\hat\mu_{(m)}$.}\label{line1}
	    \State \parbox[t]{\dimexpr\linewidth-\algorithmicindent}
	    {Calculate $K = \max \big(\big\{i|\hat\mu_{(i)} - \frac{\sum_{k=1}^i   \hat\mu_{(k)}-\lambda}{i}>0\big\}\big)$.}\label{line2}
	    \State \parbox[t]{\dimexpr\linewidth-\algorithmicindent}
	    {Calculate $\beta = \frac{\sum_{k=1}^K   \hat\mu_{(k)}-\lambda}{K}$.}\label{line3}
	    \State \parbox[t]{\dimexpr\linewidth-\algorithmicindent}
	    {Calculate $\boldsymbol{\mu}$ by:
	    \begin{equation*}
	        \mu_k = 
                \begin{cases}
                    \hat\mu_k - \beta,& \text{if } \hat\mu_k > \beta,\\
                    0,              & \text{otherwise},
                \end{cases}
                \:\: \forall k \in [1,m].
	    \end{equation*}
	    }\label{line4}
	\end{algorithmic} 
\end{algorithm}

\subsubsection{The Backbone Model}
We employ a model which takes  the embedding $\mathbf{E}$, $\boldsymbol{\mu}^{(T)}$ and the dependencies matrix $\mathbf{W}^{(T)}$ as input. After $T$ update steps we have:
\begin{align}
    &f_{\boldsymbol{\theta}}\big(\mathbf{E},\mathbf{W}^{(T)},\boldsymbol{\mu}^{(T)}\big) =\label{bkb}\\ &\mathrm{MLP}\Big(\big[\langle\mathbf{e}_1,\mathbf{e}_2\rangle_{\boldsymbol{\Phi}^{1,2}},\cdots\!,\langle\mathbf{e}_{m-1},\mathbf{e}_m\rangle_{\boldsymbol{\Phi}^{m\!-\!1,m}},  \text{vec\big({\fontsize{7}{7}\selectfont $\frac{\boldsymbol{\mu}^{(T)}}{\lambda}\!\circ\!\mathbf{W}^{(T)}$}\big)}\big]\Big),\nonumber
\end{align}
%
where parameters $\boldsymbol{\theta}$ include parameters of MLP and set $\{\boldsymbol{\Phi}^{1,2},\cdots,\boldsymbol{\Phi}^{m\!-\!1,m}\}$. We use a three-layer MLP with 100 neurons in each layer and the ReLU activation function. 
$\langle\mathbf{a},\mathbf{b}\rangle_{\boldsymbol{\Phi}}
\triangleq  \mathbf{a}^\top\boldsymbol{\Phi}\mathbf{b}$
defines an operation between two vectors $\mathbf{a},\mathbf{b}\in\mathbb{R}^k$ and $\mathbf{\Phi}\in\mathbb{R}^{k\times k}$. Such a pair-wise product between $\mathbf{e}_i$ and $\mathbf{e}_j$, $\forall i,j\in[1,m]$ generates an $m(m-1)/2$ dimensional vector. $\mathbf{W}^{(T)}$ is multiplied by $\boldsymbol{\mu}^{(T)}/\lambda$ (each column by one corresponding coefficient of $\boldsymbol{\mu}^{(T)}/\lambda$) before being flattened by $\text{vec}(\cdot)$ to concatenate with the $m(m-1)/2$ dimensional vector forming the input to the MLP. In this way, the field dependencies are explicitly utilised in the model and $\boldsymbol{\mu}^{(T)}/\lambda$ helps adjust weight of each column in dependency matrix $\mathbf{W}^{(T)}$ similar to the attention mechanism. 

\subsubsection{Meta-Dependency Learning for the Supervised Task}
\label{metal}

Earlier, in \eqref{grdupdtw} and \eqref{qpp},  we have presented the update rules for the dependency matrix $\mathbf{W}^{(t)}$ and the weight vector $\boldsymbol{\mu}^{(t)}$ for each instance. In this section, we introduce our full meta-leaning algorithm which produces a global dependency matrix $\mathbf{W}^{(0)}$ and an embedding matrix $\mathbf{V}$. Firstly, we introduce the concept of Model-Agnostic Meta-Learning~\cite{finn2017model} (MAML) \cite{finn2017model} which is key to understanding our algorithm.

\begin{algorithm}[b]
	\caption{Conceptual Sketch of MAML \cite{finn2017model}.} \label{maml_basic}
    \textbf{Input}: Dataset $D$, number of steps $T$, outer step size $\gamma$,\\inner step size $\eta$.\\ 
    \textbf{Output}: Global model $\boldsymbol{\vartheta}$.
	\begin{algorithmic}[1]
	    \State {Initialise global model parameters $\boldsymbol{\vartheta}$.}
		\While {not done}
		    \State{Sample an individual task  $(\mathbf{x},\mathbf{y})\in D$;}
			\State Initialize task-specific parameters  $\boldsymbol{\vartheta}'\leftarrow\boldsymbol{\vartheta}$;
			\For {step $t=1,2,\ldots,T$}
				\State \parbox[t]{\dimexpr\linewidth-\algorithmicindent}
			{Update task-specific parameters:\\   $\qquad\qquad\;\;\boldsymbol{\vartheta}'\leftarrow\boldsymbol{\vartheta}'-\eta\nabla_{\boldsymbol{\vartheta}}\ell\big(g_{\boldsymbol{\vartheta}}(\mathbf{x}),\mathbf{y})\big)$;}
			\EndFor
			\State\parbox[t]{\dimexpr\linewidth-\algorithmicindent}
			{Update global model parameters $\boldsymbol{\vartheta}$:\\
            $\boldsymbol{\vartheta}\leftarrow\boldsymbol{\vartheta}-\gamma\nabla_{\boldsymbol{\vartheta}}\ell\big(g_{\boldsymbol{\vartheta}'}(\mathbf{x}),\mathbf{y})\big)$;
			}
		\EndWhile
	\end{algorithmic}
\end{algorithm}

MAML is often called a second-order optimisation algorithm, as it alternates between optimising the global classification model (so-called outer loop) and models of individual tasks (so-called inner loop). Algorithm \ref{maml_basic} shows a conceptual sketch of MAML: (i) The inner loop receives global model parameters $\boldsymbol{\vartheta}$ of feature encoder $g_{\boldsymbol{\vartheta}}(\mathbf{x})$ (encodes sample $\mathbf{x}$) and  adapts $\boldsymbol{\vartheta}$ to parameters $\boldsymbol{\vartheta}'$ of an individual task. (ii) The outer loop receives features of individual tasks from the feature encoder $g_{\boldsymbol{\vartheta}}(\mathbf{x})$, encoded with individual task-specific parameters $\boldsymbol{\vartheta}'$, and the global parameters $\boldsymbol{\vartheta}$ are updated so that they represent individual tasks well.

In what follows, we build on the above concept to develop our Meta-Dependency Learning (\textbf{MDL}) algorithm. Firstly, we note that MDL in Algorithm \ref{algo} differs from MAML in Algorithm \ref{maml_basic}, \ie, our inner loop performs unsupervised tasks with the goal of informing the target supervised classification task (outer loop). Such a design is an entirely novel proposition. Moreover, MAML acts on parameters $\boldsymbol{\vartheta}$ and $\boldsymbol{\vartheta}'$ of feature encoder $g_{\boldsymbol{\vartheta}}(\mathbf{x})$ of some input samples $\mathbf{x}$ (\eg, images). In contrast, the inner loop of our algorithm performs the refinement of the dependency matrix $\mathbf{W}^{(0)}$ towards individual instances captured by $\mathbf{W}^{(t)},t\geq1$, based on the dependency loss. The outer loop performs the update of the global dependency matrix $\mathbf{W}^{(0)}$, the embedding matrix $\mathbf{V}$ and the  parameters $\boldsymbol{\theta}$ of the backbone model $f_{\boldsymbol{\theta}}(\cdot)$. 
{\ZB We considered learning both $\mathbf{W}^{(0)}$ and $\boldsymbol{\mu}^{(0)}$ with MDL with either same or different learning rates, but we found the results were worse in such a case. This is likely due to the non-smooth projection $\Pi_\Delta(\cdot)$ making it harder to train an optimal $\boldsymbol{\mu}^{(0)}$. Including $\boldsymbol{\mu}^{(0)}$ in MDL could also encourage model overfitting. Setting $\boldsymbol{\mu}^{(0)}$ to be a constant vector is equivalent to assuming a uniform distribution prior over $\boldsymbol{\mu}^{(t)}$, which empirically turns out to be the better choice.} {\ZBL Note that the projection $\Pi_w(\cdot)$ will not complicate the optimisation of $\mathbf{W}^{(0)}$ because the projection only works on diagonal elements by setting them to $-1$  (constant by design). Alternatively, one may exclude the elements of diagonal from optimisation. The diagonal elements and projection were introduced by us so that we can organise the elements $\{w_{ik}: i \neq k \}$ into off-diagonal elements of the matrix $\mathbf{W}^{(0)}$ for simplicity of notation and implementation.}

\begin{figure}[t]
\vspace{-0.3cm}
  \centering
  \includegraphics[width=0.7\columnwidth]{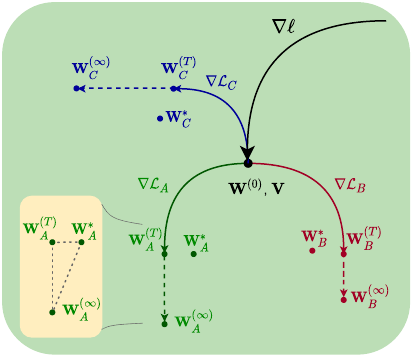}
  \caption{Diagram of MDL on three instances (A, B, and C). The task-specific loss $\ell$ is optimised \wrt the global dependency matrix $\mathbf{W}^{(0)}$ and the embedding matrix $\mathbf{V}$.  $\mathbf{W}^{(0)}$ is refined towards individual instance-wise optima ($\mathbf{W}_A^*$ , $\mathbf{W}_B^*$ , $\mathbf{W}_C^*$) by the dependency loss terms ($\mathcal{L}_A$, $\mathcal{L}_B$, $\mathcal{L}_C$), which  helps lower the task-specific loss $\ell$.  However, the optima ($\mathbf{W}_A^{\infty}$, $\mathbf{W}_B^{\infty}$, $\mathbf{W}_C^{\infty}$) obtained on the dependency loss are unlikely to coincide with the true optima ($\mathbf{W}_A^*$ , $\mathbf{W}_B^*$ , $\mathbf{W}_C^*$) for the task-specific loss $\ell$ due to overfitting and/or the different nature of both objectives. Thus, a ``good'' finite number of $T$ steps of the gradient descent on the dependency loss are required to ensure that intermediate dependency matrices ($\mathbf{W}_A^{T}$, $\mathbf{W}_B^{T}$, $\mathbf{W}_C^{T}$) get close to the true unknown dependency matrices ($\mathbf{W}_A^*$ , $\mathbf{W}_B^*$ , $\mathbf{W}_C^*$), which could be obtained if the dependency loss was replaced with the task-specific loss that has access to labels (we have no access to labels during testing). As an example, the yellow panel illustrates that $\|\mathbf{W}_A^*-\mathbf{W}_A^{(T)}\|_F<\|\mathbf{W}_A^*-\mathbf{W}_A^{(\infty)}\|_F$ for $1\leq T\ll\infty$.
  }\label{meta-learn}.
\end{figure}

\begin{algorithm}[b]
	\caption{Meta-Dependency Learning (Training).} \label{algo}
    \textbf{Input}: Dataset $D$, step size $\eta$ and $\gamma$, number of steps $T$.\\ 
    \textbf{Output}: Model parameters ${\boldsymbol{\theta}}$, embedding matrix $\mathbf{V}$, global dependency matrix $\mathbf{W}^{(0)}$.
	\begin{algorithmic}[1]
	    \State {Initialise model parameters $\boldsymbol{\theta}$, embedding matrix $\mathbf{V}$, global dependency matrix $\mathbf{W}^{(0)}$, weight coefficients $\boldsymbol{\mu}^{(0)}=\frac{\lambda}{m}\mathbf{1}_m$.}
		\While {not done}
		    \State \parbox[t]{\dimexpr\linewidth-\algorithmicindent}
		    {Sample an instance $(\mathbf{x},\mathbf{y}) \in D$;}
		    \State{Generate the embedding $\mathbf{E}$ with $\mathbf{V}$ by \eqref{ebd};}
			\For {step $t=1,2,\ldots,T$}
				\State Calculate $\mathbf{W}^{(t)}$ by \eqref{grdupdtw} with step size $\eta$;
				\State Calculate $\boldsymbol{\mu}^{(t)}$ by \eqref{grdupdtmu} with step size $\eta$;
			\EndFor
			\State \parbox[t]{\dimexpr\linewidth-\algorithmicindent}
			{Update  model parameters $\boldsymbol{\theta}$, embedding matrix $\mathbf{V}$, and global dependency matrix $\mathbf{W}^{(0)}$:
			\begin{align*}
			\boldsymbol{\theta} &\leftarrow  \boldsymbol{\theta} - \gamma\nabla_{\boldsymbol{\theta}} \ell\big(f_{\boldsymbol{\theta}}(\mathbf{E},\mathbf{W}^{(T)},\boldsymbol{\mu}^{(T)}), \mathbf{y}\big);\\
			\mathbf{V} &\leftarrow \mathbf{V} - \gamma\nabla_{\mathbf{V}} \ell\big(f_{\boldsymbol{\theta}}(\mathbf{E},\mathbf{W}^{(T)},\boldsymbol{\mu}^{(T)}), \mathbf{y}\big);\\ 
			\mathbf{W}^{(0)} &\leftarrow \Pi_w\Big( \mathbf{W}^{(0)} \!- \gamma\nabla_{\mathbf{W}^{(0)}} \ell\big(f_{\boldsymbol{\theta}}(\mathbf{E},\mathbf{W}^{(T)},\boldsymbol{\mu}^{(T)}), \mathbf{y}\big)\Big);
			\end{align*}
			}\label{algostep9} 
		\EndWhile
	\end{algorithmic}
\end{algorithm}

\begin{algorithm}[b]
	\caption{Meta-Dependency Learning (Inference).} \label{algotest}
    \textbf{Input}: Testing instance $\mathbf{x}$; step size $\eta$, number of steps $T$, model parameters  ${\boldsymbol{\theta}}$, embedding matrix $\mathbf{V}$, global dependency matrix $\mathbf{W}^{(0)}$.\\
    \textbf{Output}: Prediction $\hat y$.
	\begin{algorithmic}[1]
	\State Initialise weight coefficients $\boldsymbol{\mu}^{(0)}=\frac{\lambda}{m}\mathbf{1}_m$.
	    \State \parbox[t]{\dimexpr\linewidth-\algorithmicindent}
	    {Generate the embedding $\mathbf{E}$ with $\mathbf{V}$ by \eqref{ebd};}
		\For {step $t=1,2,\ldots,T$}
			\State Calculate $\mathbf{W}^{(t)}$ by \eqref{grdupdtw} with step size $\eta$;
			\State Calculate $\boldsymbol{\mu}^{(t)}$ by \eqref{grdupdtmu} with step size $\eta$;
		\EndFor
 		\State \parbox[t]{\dimexpr\linewidth-\algorithmicindent}
 		{$\hat y = f_{\boldsymbol{\theta}}\big(\mathbf{E},\mathbf{W}^{(T)},\boldsymbol{\mu}^{(T)}\big);$}
	\end{algorithmic}
\end{algorithm}

For simplicity, we explain the algorithm using classification as the supervised task. 
Let a training instance be   $(\mathbf{x},\mathbf{y})\in D$, where $D\sim\mathcal{D}\subset\mathcal{X}\times\mathcal{Y}$ is a dataset sampled from the data distribution $\mathcal{D}$, whereas  $\mathcal{X}$ and $\mathcal{Y}$ define  an input  and output space respectively. 
Let $f_{\boldsymbol{\theta}}(\cdot)$ be  a function parameterised by $\boldsymbol{\theta}$ taking the embedding $\mathbf{E}$, dependency matrix $\mathbf{W}^{(T)}$ and weight vector $\boldsymbol{\mu}^{(T)}$ as inputs. Then, given a loss function $\ell(\cdot)$, the learning objective on the underlying data distribution $\mathcal{D}$ is given by:
\begin{equation}\label{expctedloss}
    \min_{\boldsymbol{\theta},\mathbf{V},\mathbf{W}^{(0)}} \mathbb{E}_{(\mathbf{x},\mathbf{y}) \sim \mathcal{D}}\big[\ell(f_{\boldsymbol{\theta}}(\mathbf{E},\mathbf{W}^{(T)},\boldsymbol{\mu}^{(T)}), \mathbf{y})\big],
\end{equation}
In practice, we will use the average loss over a batch of instances sampled from dataset $D$ to approximate the expected loss over the data distribution $\mathcal{D}$ in \eqref{expctedloss}.

In the ideal case, for an instance sampled from dataset $D$, 
optimising the dependency loss $\mathcal{L}$   leads to $\mathbf{W}^{(T)}$ after $T$ steps, coinciding with optimising $\mathbf{W}$ \wrt the classification loss $\ell$ directly (without considering the dependencies). That means we can optimise $\mathcal{L}$ as a surrogate instead of optimising $\ell$ directly. Thus, we can also optimise the dependency loss for each unlabelled instance of test data, expecting this will lead to lower value of $\ell$. To this end, Meta-Dependency Learning 
finds a good embedding matrix $\mathbf{V}$ and global dependency matrix $\mathbf{W}^{(0)}$ such that adjusting the dependency matrix $\mathbf{W}^{(0)}$ on the embedding $\mathbf{E}$ towards a smaller dependency loss $\mathcal{L}$ will reduce the classification loss $\ell$.  

To put it in simply, the global dependency matrix $\mathbf{W}^{(0)}$ should model the dependencies between fields well enough to be a starting point for the refinement step, and the refinement step should help improve the global dependency matrix. Updating $\mathbf{W}^{(0)}$ for each instance helps capture finer and more useful dependencies for the supervised task even on unseen data. The embedding matrix $\mathbf{V}$  helps generate an informative  embedding $\mathbf{E}$ for each instance, thus capturing the beneficial dependencies for the  supervised task.

{\ZB There are two main advantages of using MDL over using similarity between embeddings: i) the similarity between embeddings is typically calculated pair-wise between fixed embeddings of features, while the dependencies of MDL jointly consider all features of an instance; ii) the MDL enables the model to  tune the embeddings instance-wise following dependencies in testing stage without labels, as it connects dependencies with the loss (of classifier).}

Fig.~\ref{framework}, and Algorithms~\ref{algo} and \ref{algotest} provide detailed steps of MDL. Fig.~\ref{meta-learn} explains how the dependency loss and the task-specific  loss interact with each other in MDL.

\subsubsection{Optimisation} Modern optimisation toolkits provide the automatic differentiation~\cite{abadi2016tensorflow,paszke2019pytorch} and can handle the adaptive learning rate. In practice, we  use such functionalities in step \ref{algostep9} of Algorithm~\ref{algo}. Nonetheless, thanks to the simple form of the dependency loss, its derivative can be explicitly used to avoid the automatic differentiation in the inner loop of Algorithm~\ref{algo}. Therefore, only one forward pass and one backward pass are required per one parameter update step. Moreover, the proposed algorithms can also process mini-batches to speed up training/testing. 

\subsubsection{Complexity Analysis} Let $n$, $m$, $k$ and $T$ denote the size of mini-batch, number of fields, embedding dimension and gradient descent steps for the dependency loss, respectively. The gradient descent updates \eqref{grdupdtw} and \eqref{grdupdtmu} for dependency learning require $\mathcal{O}(Tnkm^2)$ time, and  $\mathcal{O}(m^2)$ parameters in addition to the MLP.  The embedding in \eqref{ebd} can be computed fast through indexing. Given a reasonable choice of $k$ and the fact that $m$ and $T$ are usually small, learning the dependencies introduces a negligible computational time and space in additional to the cost of backbone $f_{\boldsymbol{\theta}}$. 

\begin{table*}[!h]
\caption{Statistics of benchmark datasets.}
\label{stsd}
\centering
\begin{tabular}{@{}llllllll@{}}
\toprule
Dataset   & \#instances & \#train & \#validation & \#test          & \#positive labels & \#features & \#fields \\ \midrule
KDD2012   & 149,639,105 & 95,764,025 & 23,941,007 & 29,934,073 & 6,658,036         & 2,047,519  & 11       \\
Avazu     & 40,428,967  & 32,343,173 & 4,042,896 & 4,042,898   & 6,865,066         & 2,018,012  & 22       \\
Criteo    & 45,840,617  & 36,672,493 & 4,584,061 & 4,584,063   & 11,745,438        & 395,894    & 39       \\
ML-Tag    & 2,006,859   & 1,605,487 & 200,685 & 200,687        & 668,953           & 90,445     & 3        \\
ML-Rating & 1,000,209   & 800,167 & 100,020 & 100,022          & 575,281           & 13,461     & 6        \\
Frappe    & 288,609     & 230,887 & 28,860 & 28,862            & 96,203            & 5,382      & 10       \\ \bottomrule
\end{tabular}
\end{table*}

\section{Related Works/Alternative Approaches}
{\ZB Categorical data often come in as tables, thus being a natural fit to methods dealing with tabular data~\cite{liu2020dnn2lr,liu2021mining,Popov2020Neural,arik2021tabnet,huang2020tabtransformer,somepalli2022saint}. These methods have drawn much attention in recent years. They usually consider a more general case where data consist of both numerical and categorical types, but fall short of considering the sparsity issue in large-scale categorical data, where there could be millions of categorical features observed sparsely. Corresponding benchmarks were often conducted on datasets with smaller number of features~\cite{shwartz2022tabular,gorishniy2021revisiting}. The sparseness of categorical features requires specific design to learn appropriate embeddings, \eg, as in the Factorization Machine (FM)~\cite{rendle2010factorization}.} 

FM and its variants~\cite{blondel2016higher,blondel2015convex,pan2018field} are arguably some of the most famous methods for classification/regression or ranking of categorical data due to their advantages in handling the sparse format of input. FM models typically use second-order polynomial modelling where pairs of features are weighted by the inner product of  corresponding pairs of embedding vectors. Compared to learning embeddings that are directly fed into a neural network, associating embedding vectors with coefficients of the polynomial model works better alleviating the sparsity issue. Such embedding-based models have been extended to neural network models~\cite{lu2020dual,yu2019input,xiao2017attentional,wang2020attention,cheng2020adaptive,he2017neural,guo2017deepfm,lian2018xdeepfm,hong2019interaction}.

Our MDL method utilizes the dependencies to assist the learning of embedding vectors which is completely different from polynomial modelling commonly adopted in FM-related models. The dependencies are refined in a localised way to better capture the variance of each instance.  There also exist other recent models such as a linear mixed model~\cite{liang2020lmlfm} and locally-linear models~\cite{liu2017locally,chen2019generalized} which handle the categorical data in an instance-adaptive manner. {\ZBL The specific structure of categorical data has  inspired several related methods~\cite{juan2016field,pan2018field,wang2020attention,qu2018product,chen2019rafm,li2020field} 
which use the field information  to parameterise the model in various ways. However, the dependencies between fields are overlooked in above-motioned models. 
Our MDL method exploits the sparse structure of categorical data and captures dependencies between categorical features by the auxiliary loss (dependency loss). This mechanism is used conjointly with the task-driven loss (supervised loss) applied to the target learning task. In contrast, the above models simply learn to predict the target from categorical features. 

Specifically, Field-aware Factorization Machine (FFM) \cite{juan2016field} uses an embedding vector per categorical feature per field. 
Field-weighted Factorization Machine \cite{pan2018field} uses different weights for different pairs of fields.
 Attention-over-attention Field-aware Factorization Machine \cite{wang2020attention} builds upon FFM by introducing an attention mechanism. 
Product-based Neural Networks \cite{qu2018product}  use differently weighted inner products of embedding vectors from different pairs of fields. 
Rank-aware Factorization Machine (RaFM) \cite{chen2019rafm} uses embeddings of different dimensionality based on the count of categorical features under each field (in contrast to the  embeddings with one fixed dimensionality).
Field-wise Learning (FWL) \cite{li2020field} uses embeddings with separate model parameters per field. 
}

The dependency learning we explore in this paper is a form of unsupervised learning, which is motivated by Multiple Imputation by Chained Equations (MICE)~\cite{azur2011multiple}, which masks out features followed by the imputation strategy.
{\ZB Some works~\cite{somepalli2022saint,wang2020attention,xiao2017attentional} use the attention mechanism which differs from the field dependencies we use, in the sense that the attention is not necessarily or explicitly related to the dependencies between fields.} 
Some works \cite{yoon2020vime,ucar2021subtab} also perform imputation on tabular data. They focus on self- and semi-supervised learning where imputation is applied globally as a regularisation. However, categorical data models with regularisation perform poorly  as shown in our empirical study. Obtaining categorical features, imputed through  classification, is also computationally prohibitive for high-dimensional categorical data. Some natural language processing models~\cite{zheng2014topic,zheng2015deep} also exploit inference of sequence embeddings from each other.

Our method is also in part inspired by the ability of Model-Agnostic Meta-Learning~\cite{finn2017model} (MAML) to rapidly adapt to new tasks. MAML and its variants are often called second-order optimisation models as they use non-overlapping instances  within the inner loop  and the outer loop of meta-learning, which result in second-order gradients that improve the generalisation ability of MAML. Meta-learning is often used with few-shot inductive   \cite{protonet, relationnet,simon2022meta,Shi2022Accurate3C,hongguangACCV} or transductive  \cite{lee2019meta,zhu2022ease,zhu2023protolp} classification, few-shot object detection \cite{zhang2020sopaccv,zhang2022kernelized,zhang2022time}, segmentation \cite{shaban2017oslsm,dahyun2023seg}, action  \cite{udtw_eccv22,Wang_2022_ACCV} and keypoint recognition \cite{lu2022few}. The research on the inductive bias of meta-learning is ongoing~\cite{lu2021towards} but the target models required in such an analysis are often not available. Such an issue can  somewhat be addressed by meta-learning an unsupervised weight update rule~\cite{metz2018meta}, or by leveraging  knowledge distillation from pre-trained models~\cite{lu2021towards}. These methods however require bespoke design for different data types and are generally not applicable to categorical data. In contrast, we tailor the idea of MAML to categorical data by leveraging the field dependencies, which is an entirely novel proposition.

\section{Experiments}
In this section, we present extensive experiments to verify the effectiveness of MDL. We also study the correspondence between the dependency loss and the task-specific (classification) loss, and provide numerous ablations. We choose several popular models as baselines. 

\subsection{Datasets and Preprocessing}
\label{datapre}
We use the following publicly available datasets dedicated to various applications, all of which can be formulated as binary classification or regression tasks.

\vspace{0.1cm}
\noindent\textbf{Avazu}\footnote{\url{https://www.kaggle.com/c/avazu-ctr-prediction/data}.}. This dataset contains clicks  on mobile ads; all features are categorical. The aim is to predict the click behaviour of users with user and ad attributes. The labels are 0/1 for non-clicked/clicked records. We use the same train/validation/test sets and preprocessing steps as \citet{li2020field}. For each field, features occurring less than 4 times were all aggregated into a single feature.

\vspace{0.1cm}
\noindent\textbf{Criteo}\footnote{\url{https://www.kaggle.com/c/criteo-display-ad-challenge}.}. This dataset focuses on the question: given a user and the page they are visiting, what is the probability that they will click on a given ad? Each instance in this dataset corresponds to a display ad served by Criteo. This problem is often interpreted as a binary classification task with the label 0/1 for non-clicked/clicked records. There are 13 columns of integer features (mostly count features) and 26 columns of categorical features. The semantic of the features is undisclosed. We use the same train/validation/test sets and preprocessing steps as \citet{li2020field}. Integer features were transformed into categorical features by binning following log-transformation. Features appearing less than 35 times were grouped as one feature in corresponding fields.

\begin{table*}[httb!]
\caption{Results on KDD2012, Avazu and Criteo. Asterisk `*' means the statistical significance of MDL \wrt the given baseline.}
\vspace{-0.4cm}
\setlength{\tabcolsep}{3pt}
\label{expl}
\centering
\small
\begin{tabular}{@{}llccccccccccc@{}}
\toprule
Method & & \multicolumn{3}{c}{KDD2012} &  & \multicolumn{3}{c}{Avazu} &  & \multicolumn{3}{c}{Criteo} \\ \cmidrule(lr){3-5} \cmidrule(lr){7-9} \cmidrule(l){11-13} 
\multicolumn{1}{c}{} &          & Logloss & AUC    & \#params &  & Logloss & AUC    & \#params &  & Logloss & AUC    & \#params \\ \midrule
LR      &                       & 0.1560$^{*}$ & 0.7865$^{*}$ & 2.05M    &  & 0.3819$^{*}$  & 0.7763$^{*}$ & 2.02M    &  & 0.4561$^{*}$  & 0.7941$^{*}$ & 0.40M    \\
FM &\cite{rendle2010factorization}           & 0.1523$^{*}$  & 0.8011$^{*}$  & 126.40M &  & 0.3770$^{*}$  & 0.7854$^{*}$ & 201.80M &  & 0.4420$^{*}$  & 0.8082$^{*}$  & 32.07M \\
FFM &\cite{juan2016field}       & 0.1529$^{*}$  & 0.8000$^{*}$& 47.09M   &  & 0.3737$^{*}$   & 0.7915$^{*}$ & 341.04M  &  & 0.4413$^{*}$ & 0.8107$^{*}$ & 60.57M   \\
xDFM &\cite{lian2018xdeepfm}    & 0.1516$^{*}$  & 0.8053$^{*}$ & 247.10M  &  & 0.3727$^{*}$   & 0.7915$^{*}$ & 83.10M   &  & 0.4406$^{*}$  & 0.8112$^{*}$ & 4.68M    \\
AFN &\cite{cheng2020adaptive}   & 0.1513$^{*}$  & 0.8071$^{*}$ & 264.89M  &  & 0.3754$^{*}$   & 0.7886$^{*}$ & 45.42M   &  & 0.4395$^{*}$ & 0.8125$^{*}$ & 17.55M   \\
IFM &\cite{yu2019input}         & 0.1501$^{*}$  & 0.8115$^{*}$ & 246.78M  &  & 0.3733$^{*}$   & 0.7905$^{*}$ & 82.74M   &  & 0.4403$^{*}$  & 0.8114$^{*}$ & 8.31M    \\
INN & \cite{hong2019interaction} & 0.1581$^{*}$ & 0.7924$^{*}$ & 66.21M   &  & 0.3816$^{*}$   & 0.7761$^{*}$ & 22.20M   &  & 0.4447$^{*}$  & 0.8067$^{*}$ & 4.36M    \\
LLFM & \cite{liu2017locally}     & 0.1522$^{*}$  & 0.8019$^{*}$ & 264.84M  &  & 0.3767$^{*}$   & 0.7862$^{*}$ & 532.75M  &  & 0.4425$^{*}$  & 0.8094$^{*}$ & 52.25M   \\
RaFM & \cite{chen2019rafm}       & 0.1540$^{*}$  & 0.7943$^{*}$ & 187.45M  &  & 0.3741$^{*}$   & 0.7903$^{*}$ & 87.27M   &  & 0.4415$^{*}$  & 0.8105$^{*}$ & 70.79M   \\
FWL & \cite{li2020field}         & 0.1515$^{*}$  & 0.8075$^{*}$ & 698.21M  &  & 0.3714$^{*}$   & 0.7946$^{*}$ & 357.18M  &  & 0.4390$^{*}$  & 0.8130$^{*}$ & 206.65M  \\
\rowcolor{LightCyan} MDL  & ours                          & \textbf{0.1494}  & \textbf{0.8146} & 247.15M  &  & \textbf{0.3686}  & \textbf{0.7981} & 83.47M   &  & \textbf{0.4375}  & \textbf{0.8145} & 9.23M    \\ \bottomrule
\end{tabular}
\end{table*}

\begin{table*}
\centering
\caption{Results on ML-Tag, Frappe and ML-Rating. Asterisk `*' means the statistical significance of MDL \wrt the given baseline.}
\vspace{-0.4cm}
\setlength{\tabcolsep}{3pt}
\label{expm}
\small
\begin{tabular}{@{}llcccccccccc@{}}
\toprule
Method &                                & \multicolumn{3}{c}{ML-Tag}                  &  & \multicolumn{3}{c}{Frappe}                  &  & \multicolumn{2}{c}{ML-Rating} \\ \cmidrule(lr){3-5} \cmidrule(lr){7-9} \cmidrule(l){11-12} 
       &                                & Logloss         & AUC             & \#params &  & Logloss         & AUC             & \#params &  & MSE               & \#params   \\ \midrule
LR     &                                & 0.3069$^*$      & 0.9288$^*$      & 90446   &  & 0.3111$^*$      & 0.9318$^*$      & 5383    &  & 0.8152$^*$            & 13462     \\
FM     & \cite{rendle2010factorization} & 0.2114$^*$      & 0.9606$^*$      & 7.32M   &  & 0.1249          & 0.9834$^*$      & 0.44M   &  & 0.7406$^*$            & 0.15M     \\
FFM    & \cite{juan2016field}           & 0.2055$^*$      & 0.9642$^*$      & 4.43M   &  & 0.1352$^*$      & 0.9831$^*$      & 0.87M   &  & 0.7557$^*$            & 0.82M     \\
xDFM   & \cite{lian2018xdeepfm}         & 0.2220$^*$      & 0.9632$^*$      & 3.73M   &  & 0.1388$^*$      & 0.9833$^*$      & 0.27M   &  & 0.7497$^*$            & 0.58M     \\
AFN    & \cite{cheng2020adaptive}       & 0.2287$^*$      & 0.9592$^*$      & 9.72M   &  & 0.1313$^*$          & 0.9847$^*$          & 6.24M   &  & 0.7583$^*$            & 1.67M     \\
IFM    & \cite{yu2019input}             & 0.2448$^*$      & 0.9494$^*$      & 3.70M   &  & 0.1322$^*$          & 0.9837$^*$      & 0.22M   &  & 0.7417$^*$            & 0.15M     \\
INN    & \cite{hong2019interaction}     & 0.2524$^*$      & 0.9520$^*$      & 1.00M   &  & 0.2052$^*$      & 0.9744$^*$      & 0.06M   &  & 0.7397$^*$            & 0.15M     \\
LLFM   & \cite{liu2017locally}          & 0.2293$^*$      & 0.9567$^*$      & 15.19M  &  & 0.1330$^*$          & 0.9826$^*$      & 0.45M   &  & 0.7373$^*$            & 0.65M     \\
RaFM   & \cite{chen2019rafm}            & 0.2249$^*$      & 0.9589$^*$      & 6.11M   &  & \textbf{0.1199} & 0.9848$^*$          & 0.67M   &  & 0.7547$^*$            & 2.26M     \\
FWL    & \cite{li2020field}             & 0.2131$^*$      & 0.9627$^*$      & 3.35M   &  & 0.1326$^*$      & 0.9839$^*$      & 0.43M   &  & 0.7545$^*$            & 0.31M     \\
\rowcolor{LightCyan} MDL  & ours                             & \textbf{0.1843} & \textbf{0.9691} & 7.36M   &  & 0.1219          & \textbf{0.9860} & 0.32M   &  & \textbf{0.7263}   & 0.30M     \\ \bottomrule
\end{tabular}
\end{table*}

\vspace{0.1cm}
\noindent\textbf{KDD2012}\footnote{\url{https://www.kaggle.com/c/kddcup2012-track2/overview}.}. Given the training instances derived from session logs of a search engine, the goal is to predict clicked/non-clicked ad impressions, \ie,  given the features of both an ad and a user that appear in a search session log, one has to predict whether a user will click an ad. All features are categorical. We used the ``training.txt'' file of the second track in KDD CUP 2012 and set the label to be 1 if the number of clicks is non-zero, and 0 otherwise. We use an 8:2 split for training/testing with a 20\% holdout set taken from the train set for validation.  The dataset split is the same as in ~\cite{juan2016field}, but the results are different as 1) we use different preprocessing steps (for each field, features occurring less than 10 times were all aggregated into a single feature) for better scalability of all baseline methods, and 2) we do not retrain the model combining the training and validation data after completing the hyper-parameter search.

\vspace{0.1cm}
\noindent\textbf{ML-Tag}\footnote{\parbox[t]{\dimexpr\linewidth-\algorithmicindent}{\url{https://github.com/hexiangnan/neural\_factorization\_machine/tree/master/data/ml-tag}.}}. This dataset, from \citet{he2017neural} and GroupLens \cite{harper2015movielens}, is dedicated to predicting if a user will assign a tag to a movie with fields including {\em user id}, {\em movie id} and {\em tag}.

\vspace{0.1cm}
\noindent\textbf{Frappe}\footnote{\parbox[t]{\dimexpr\linewidth-\algorithmicindent}{\url{https://github.com/hexiangnan/neural\_factorization\_machine/tree/master/data/frappe}.}}. This dataset, from \citet{he2017neural} and \citet{baltrunas2015frappe}, includes app usage logs from users. 
The task,  context-aware mobile app recommendation, is to predict whether a user will use a mobile app under certain conditions. The labels are 0/1 for unused/used category respectively. All features are categorical including {\em user id}, {\em weather}, \etc.

\vspace{0.1cm}
\noindent\textbf{ML-Rating}\footnote{\url{https://grouplens.org/datasets/movielens/1m}.}. This is the MovieLens 1M Dataset from GroupLens~\cite{harper2015movielens}. We study the task of predicting the user feedback on movies. Ratings use  5-star scale. We combine the user attributes and {\em movie id} as features, and formulate the task as a regression task predicting the users' ratings of movies. 

\vspace{0.1cm}
For Avazu, Criteo, and KDD2012 datasets we follow the splits in other papers as specified above. For the rest of datasets, we randomly split each of them into train (80\%), validation (10\%), and test (10\%) sets. The statistics of the datasets after preprocessing are presented in Table~\ref{stsd}.

\subsection{Evaluation Metrics}
In what follows, we use Logloss (\ie, binary cross-entropy loss), AUC (Area Under the ROC curve) and MSE (Mean Squared Error) as our evaluation criteria where they are applicable. 

We note that a small increase in AUC or an improvement of 0.001 for Logloss is generally considered to be significant on the above datesets~\cite{cheng2016wide,wang2017deep}. We also conduct statistical analyses to test for differences between the mean metrics obtained from our method and baselines.  For each dataset/baseline/metric, we conduct a one-way Analysis of Variance (ANOVA) with the methods considered as so-called treatments.  Pair-wise differences between the treatment mean for baselines incorporating our method and those from purely baselines are assessed by planned contrasts using Dunnett's test \cite{Dunnett1955} (testing for significance and controlling the family-wise Type I error under multiple hypothesis tests) with a significance level of ($p = 0.05$).

\subsection{Baselines, Hyper-Parameter Settings and Implementation Details}
\label{runt}
We compare our method to the following three, non-exclusive types of baselines with open source implementation on all the datasets:
1) methods that model only lower-order feature interactions, including Logistic Regression (LR), FM~\cite{rendle2010factorization}, FFM~\cite{juan2016field} and 2) methods that utilise the structure of fields or categorical data, including FWL~\cite{li2020field}, RaFM~\cite{chen2019rafm}, FFM, INN~\cite{hong2019interaction}, xDFM~\cite{lian2018xdeepfm} and 3) methods that use adaptive learning scheme, such as  AFN~\cite{cheng2020adaptive} (adaptive order), LLFM~\cite{liu2017locally} and IFM~\cite{yu2019input} (adaptive embedding). {\ZB On the four large-scale datasets, KDD2012, Avazu, Criteo and ML-Tag, we also compare our method with more diverse baselines including methods such as W\&D (Wide\&Deep)~\cite{cheng2016wide}, DeepFM~\cite{guo2017deepfm} for categorical data, NODE~\cite{Popov2020Neural}, TabNet~\cite{arik2021tabnet},
TabTr~\cite{huang2020tabtransformer}, SAINT~\cite{somepalli2022saint} for tabular data, and the recently proposed FMLP~\cite{kelongfinalmlp2023} used in CTR prediction. }

\begin{table*}
\centering
{\ZB
\setlength{\tabcolsep}{2pt}
\small
\caption{Additional results on large datasets. Asterisk `*' means the statistical significance of MDL \wrt the given baseline.}
\label{additional_results}
\begin{tabular}{@{}llccrrccrrccrrccr@{}}
\toprule
 &  & \multicolumn{3}{c}{KDD2012} &  & \multicolumn{3}{c}{Avazu} &  & \multicolumn{3}{c}{Criteo} &  & \multicolumn{3}{c}{ML-Tag} \\ \cmidrule(lr){3-5} \cmidrule(lr){7-9} \cmidrule(lr){11-13} \cmidrule(l){15-17} 
Method &  & Logloss & AUC & \#params &  & Logloss & AUC & \#params &  & Logloss & AUC & \#params &  & Logloss & AUC & \#params \\ \midrule
W\&D &\cite{cheng2016wide}& 0.1510$^*$ & 0.8073$^*$ & 66.23M &  & 0.3747$^*$ & 0.7887$^*$ & 82.83M &  & 0.4412$^*$ & 0.8100$^*$ & 8.40M &  & 0.2462$^*$ & 0.9541$^*$ & 3.73M \\
DeepFM &\cite{guo2017deepfm}& 0.1512$^*$ & 0.8065$^*$ & 66.23M &  & 0.3747$^*$ & 0.7892$^*$ & 82.83M &  & 0.4427$^*$ & 0.8092$^*$ & 4.40M &  & 0.2281$^*$ & 0.9602$^*$ & 3.73M \\
NODE &\cite{Popov2020Neural}& 0.1567$^*$ & 0.7837$^*$ & 0.20M &  & 0.3884$^*$ & 0.7651$^*$ & 0.24M &  & 0.4685$^*$ & 0.7796$^*$ & 0.29M &  & 0.2742$^*$ & 0.9423$^*$ & 0.18M \\
TabNet &\cite{arik2021tabnet}& 0.1601$^*$ & 0.7619$^*$ & 121M &  & 0.3780$^*$ & 0.7834$^*$ & 80.9M &  & 0.4597$^*$ & 0.7995$^*$ & 7.9M &  & 0.2827$^*$ & 0.9373$^*$ & 7.3M \\
TabTr &\cite{huang2020tabtransformer}& 0.1561$^*$ & 0.7833$^*$ & 145M &  & 0.3769$^*$ & 0.7854$^*$ & 81.4M &  & 0.4439$^*$ & 0.8078$^*$ & 9.7M &  & 0.2651$^*$ & 0.9473$^*$ & 9.0M \\
SAINT &\cite{somepalli2022saint}& 0.1537$^*$ & 0.7943$^*$ & 849.93M &  & 0.3817$^*$ & 0.7776$^*$ & 288.86M &  & 0.4451$^*$ & 0.8065$^*$ & 67.05M &  & 0.2864$^*$ & 0.9367$^*$ & 20.82M \\
FMLP &\cite{kelongfinalmlp2023}& 0.1508$^*$ & 0.8069$^*$ & 458.05M &  & 0.3737$^*$ & 0.7901$^*$ & 81.15M &  & 0.4452$^*$ & 0.8064$^*$ & 565.23M &  & 0.2252$^*$ & 0.9633$^*$ & 7.42M \\
\rowcolor{LightCyan} MDL & ours & \textbf{0.1494} & \textbf{0.8146} & 247.15M &  & \textbf{0.3686} & \textbf{0.7981} & 83.47M &  & \textbf{0.4375} & \textbf{0.8145} & 9.23M &  & \textbf{0.1843} & \textbf{0.9691} & 7.36M \\ \bottomrule
\end{tabular}}
\end{table*}

For all the baselines and our method, we chose hyper-parameters from a reasonably large grid of points and selected those which led to the smallest Logloss/MSE  on the validation sets. 
We applied the early-stopping strategy based on the validation sets for all models. We chose optimiser from Adagrad~\cite{duchi2011adaptive} and Adam~\cite{adamKingmaB14}. The batch size was fixed to 2048. All experiments were performed on machines using one NVIDIA Tesla P100 GPU and 64GB memory. Detailed runtimes of our method can be found in Appendices.

\subsection{Experimental Results}
\label{expsec}
The experimental results are presented in Tables~\ref{expl} and \ref{expm}, where \#params and M denotes number of parameters and Million respectively.  {\ZB We report the mean Logloss, AUC and MSE where applicable based on 5 runs. Asterisks next to numbers in tables indicate experiments where the mean metrics of our MDL method was a \textbf{statistically significant improvement upon the baselines}  ($p < 0.05$) using ANOVA (specifically, one-sided Dunnett's test for multiple comparisons \cite{Dunnett1955}). Standard deviations of multiple runs and statistical tables for the hypothesis tests under planned contrasts are provided in Appendix E and F, respectively. }

Experiments demonstrate that our method performs consistently better than other baseline methods, and is comparable to other best methods on the smaller dataset Frappe. Lower-order models LR and FM are sub-optimal for above datasets but they have a smaller computational footprint. For INN and RaFM model, we decided to bin more features to further reduce the dimensionality of features on the KDD2012 dataset so that they can be trained within reasonable time. However, this made them perform even worse than lower-order models. RaFM is a strong performer on Frappe dataset, because it uses different embeddings for features of different occurrence frequencies. Frappe contains fields whose cardinality varies significantly, and thus such an embedding scheme can be helpful. Moreover, we note that methods which consider the structure of categorical data or field information (FWL, RaFM, FFM, INN, xDFM) or using adaptive embedding (AFN, LLFM, IFM) perform better than simple LR or FM models, but they often require more modules, operations and data to learn effectively which makes them also slow. The MDL model
leads to consistently better results on most of the metrics and datasets.

{\ZB The results in Table~\ref{additional_results} show  that models for tabular data (NODE, TabNet, TabTr, SAINT) are generally inferior to models designed taking into account the sparsity of categorical data (W\&D, DeepFM, FMLP)  on the four large datasets. Our model, MDL, consistently outperforms all baselines in these experiments showing advantages of leveraging the field dependencies under the meta-learning regime. 
}
\subsection{Ablation Studies}
\label{ablsd}

\begin{figure}
\centering
\begin{subfigure}{0.24\textwidth}
  \includegraphics[trim=8 10 35 8, clip=true,width=1.0\columnwidth]{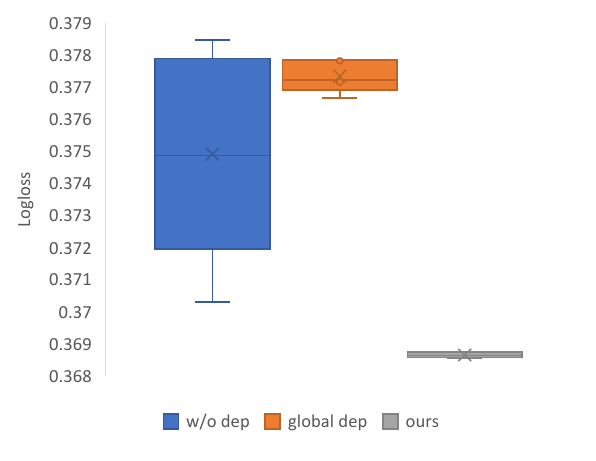}
  \caption{\label{avazu_box_log}}
\end{subfigure}
\begin{subfigure}{0.24\textwidth}
  \includegraphics[trim=8 10 35 8, clip=true,width=1.0\columnwidth]{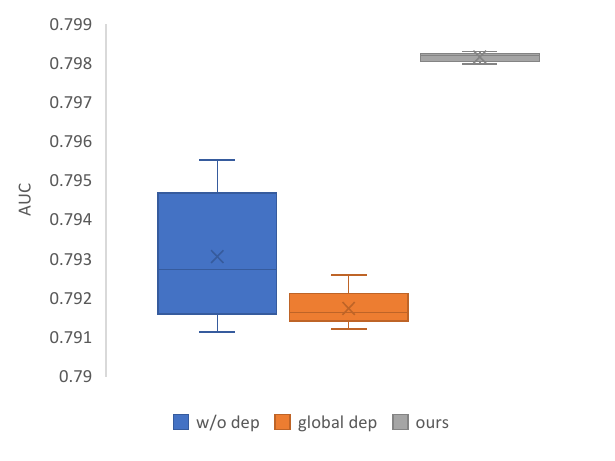}
  \caption{\label{avazu_box_auc}}
\end{subfigure}\\
\begin{subfigure}{0.24\textwidth}
  \includegraphics[trim=8 10 35 8, clip=true,width=1.0\columnwidth]{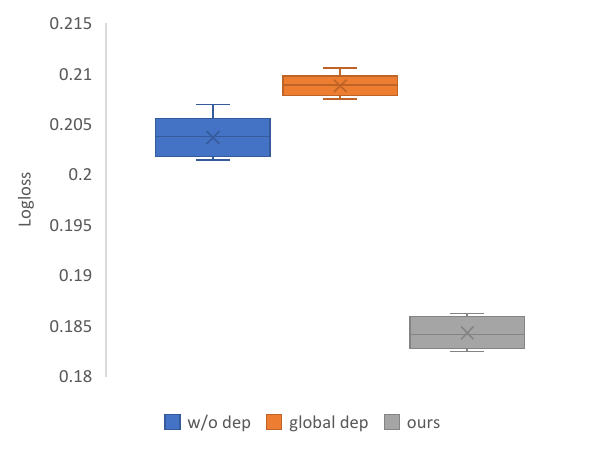}
  \caption{\label{mltag_box_log}}  
\end{subfigure}
\begin{subfigure}{0.24\textwidth}
  \includegraphics[trim=8 10 35 8, clip=true,width=1.0\columnwidth]{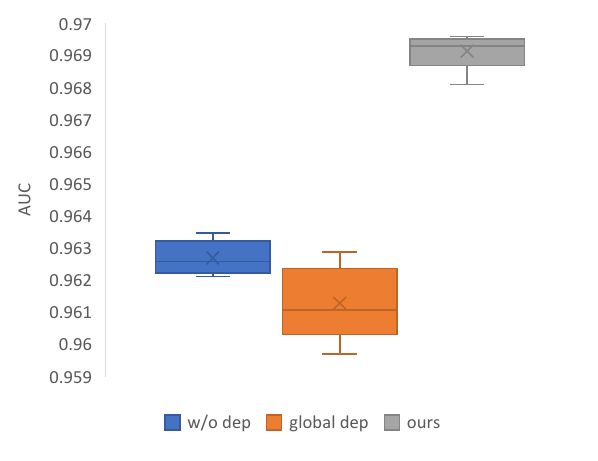}
  \caption{\label{mltag_box_auc}}
\end{subfigure}
\caption{The box plot for Logloss {\ZB and AUC} \wrt the two variants and our method on Avazu (Fig.~\ref{avazu_box_log}, Fig.~\ref{avazu_box_auc} ) {\ZB and ML-Tag (Fig.~\ref{mltag_box_log}, Fig.~\ref{mltag_box_auc})} datasets (the results are based on five runs).} \label{globaldep}
\end{figure}

We conduct detailed ablation studies on Avazu {\ZB and ML-Tag} to further verify the effectiveness of our method. All the results in this section are averaged over 5 runs on best hyper-parameters. 

\vspace{0.1cm}
\noindent\textbf{Effectiveness of Meta-Dependency Learning.}
To verify the effectiveness of Meta-Dependency Learning,  we train two variants of the proposed model: the first one (w/o dep) is the proposed model trained without the MDL steps; the second one (global dep) simply uses the dependency loss as a global regularisation term without the MDL. That is to say, the model only learns a global dependency matrix and the dependency loss is utilised as a constraint for learning the embedding matrix. {\ZB The learning objective of ``global dep'' is thus $\ell + \zeta\mathcal{L}$ where $\ell$ and $\mathcal{L}$ are the classification loss and dependency loss, respectively. We search the hyper-parameter $\zeta\geq 0$ in a reasonably large range.} The learnt global dependency and embedding matrices are used in the same way as in our method, and use the same backbone. Results on the test set are presented in Fig.~\ref{globaldep} (box plot).  Fig.~\ref{globaldep} shows that results of the two variants are significantly inferior compared to MDL.  This verifies the importance of field dependencies, and highlights that simply using the global field dependencies is sub-optimal. {\ZB Note that our method and ``global dep'' produced smaller standard deviations on the Avazu dataset compared to ``w/o dep'' as the field dependencies play a regularisation role in the learning process, thus preventing overfitting and lowering the standard deviation on the test set. Adding regularisation term $\mathcal{L}$ to the loss directly, as in $\ell + \zeta\mathcal{L}$ of ``global dep''  acts on the global loss, thus potentially resulting in overregularisation,  yielding results worse than in the case of ``w/o dep". 
In contrast, letting regularisation adapt with each  meta-learning task separately does improve results.  
For ML-Tag, the number of fields is much smaller than that in Avazu (3 versus 22), but the embedding dimension are higher (80 versus 40). The dependency matrix $\mathbf{W}$ is thus severely rank-deficient making the variance harder to control.} 

\begin{figure}
\vspace{-0.3cm}
\centering
\begin{subfigure}[b]{0.5\textwidth}
\includegraphics[trim=0 0 12 0, clip=true,width=1.0\columnwidth]{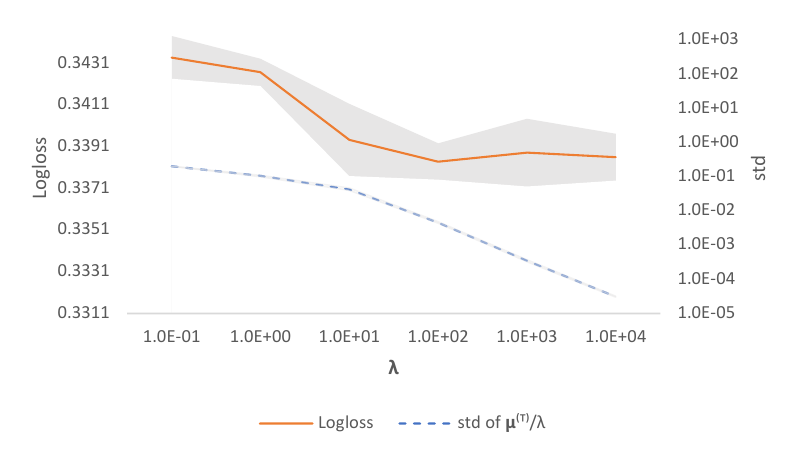}
\caption{\label{globaldep_mu1}}
\end{subfigure}
%
\begin{subfigure}[b]{0.5\textwidth}
\includegraphics[trim=0 0 12 0, clip=true,width=1.0\columnwidth]{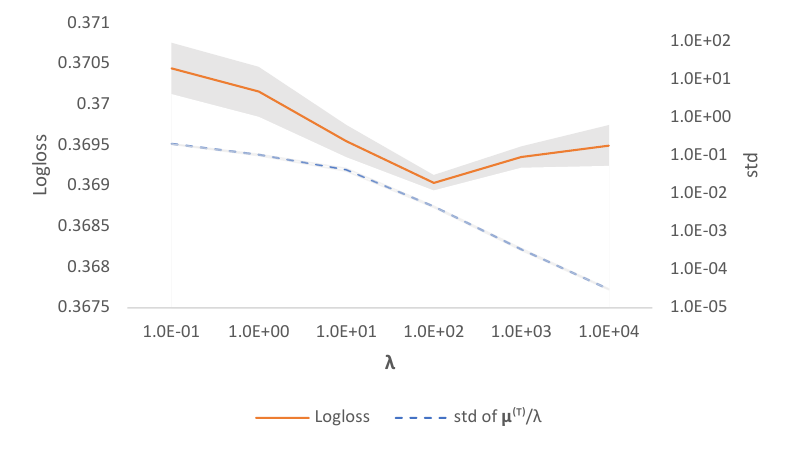}
\caption{\label{globaldep_mu2}}
\end{subfigure}
%
\begin{subfigure}[b]{0.5\textwidth}
\includegraphics[trim=0 0 12 0, clip=true,width=1.0\columnwidth]{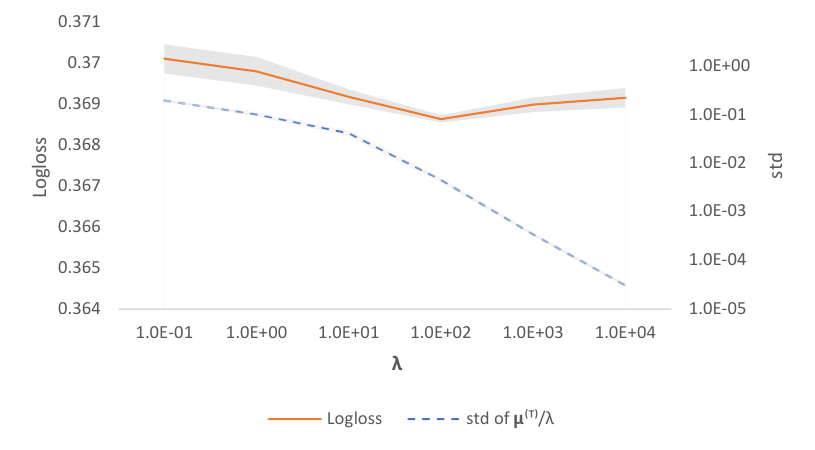}
\caption{\label{globaldep_mu3}}
\end{subfigure}
\caption{The curves for Logloss and standard deviation (std) of $\boldsymbol{\mu}^{(T)}/\lambda$ \wrt $\lambda$ on the train (Fig.~\ref{globaldep_mu1}), val (Fig.~\ref{globaldep_mu2}), and test (Fig.~\ref{globaldep_mu3}) sets of Avazu.}
\vspace{-0.2cm}
\end{figure}

\begin{figure}[t]
\vspace{-0.3cm}
\centering
\begin{subfigure}[b]{0.5\textwidth}
  \includegraphics[trim=0 6 12 0, clip=true,width=1.0\columnwidth]{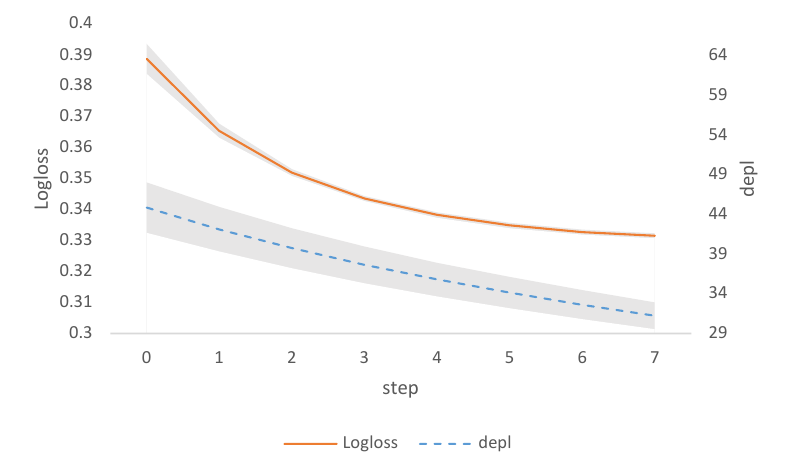}
  \caption{\label{depabl3}}
\end{subfigure}
\begin{subfigure}[b]{0.5\textwidth}
  \includegraphics[trim=0 6 12 0, clip=true,width=1.0\columnwidth]{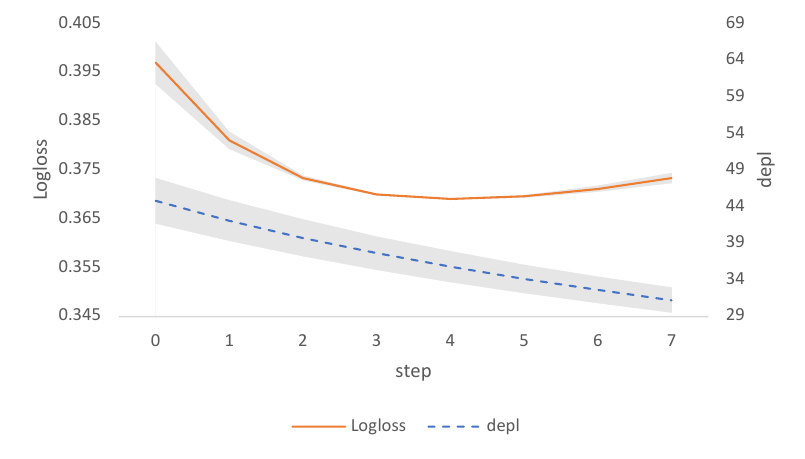}
  \caption{\label{depabl2}}
\end{subfigure}
\begin{subfigure}[b]{0.5\textwidth}
  \includegraphics[trim=0 6 12 0, clip=true,width=1.0\columnwidth]{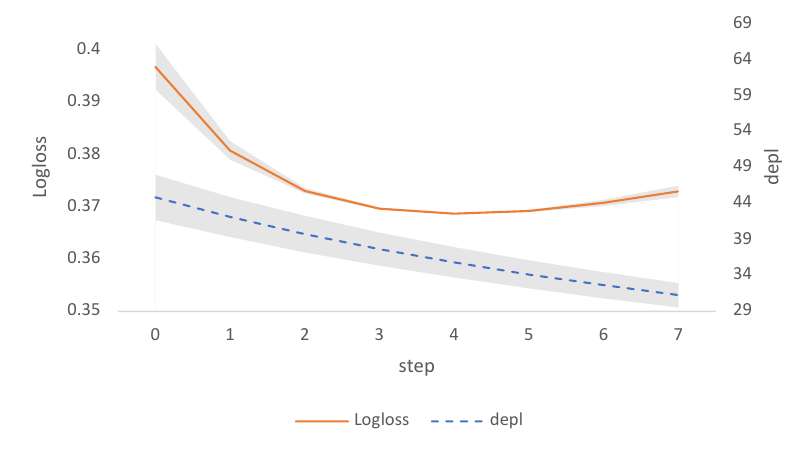}
  \caption{\label{depabl1}}
\end{subfigure}
\caption{{\ZB We train the MDL model on the train set of Avazu with $T=4$ and evaluate it on the train (Fig.~\ref{depabl3}), validation (Fig.~\ref{depabl2}) and test (Fig.~\ref{depabl1}) sets of Avazu with different $T$ for Logloss and dependency loss ({\em depl}). The dependency loss is decreasing with the number of steps $T$ on all sets, and Logloss attains its best value at $T=4$ in both validation and test, which is the same step ($T=4$) as in training.
}}
\end{figure}


\begin{figure*}[!htbp]
\hspace{-0.8cm}
\begin{subfigure}[b]{0.252\textwidth}
  \includegraphics[trim=10 75 110 110, clip=true,height=4.7cm]{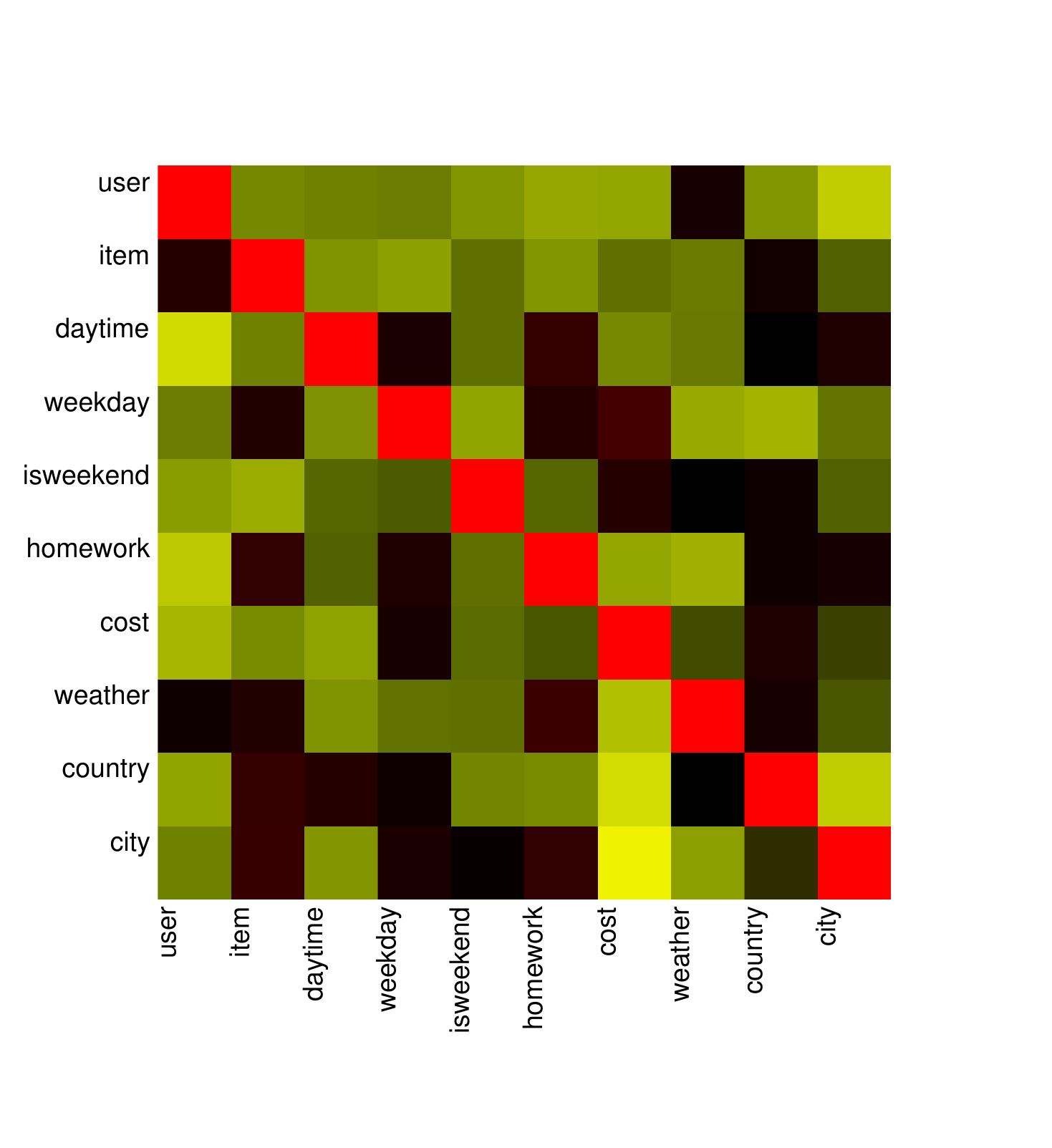}
  \caption{\label{w0}}
\end{subfigure}
\begin{subfigure}[b]{0.252\textwidth}
  \includegraphics[trim=10 75 110 110, clip=true,height=4.7cm]{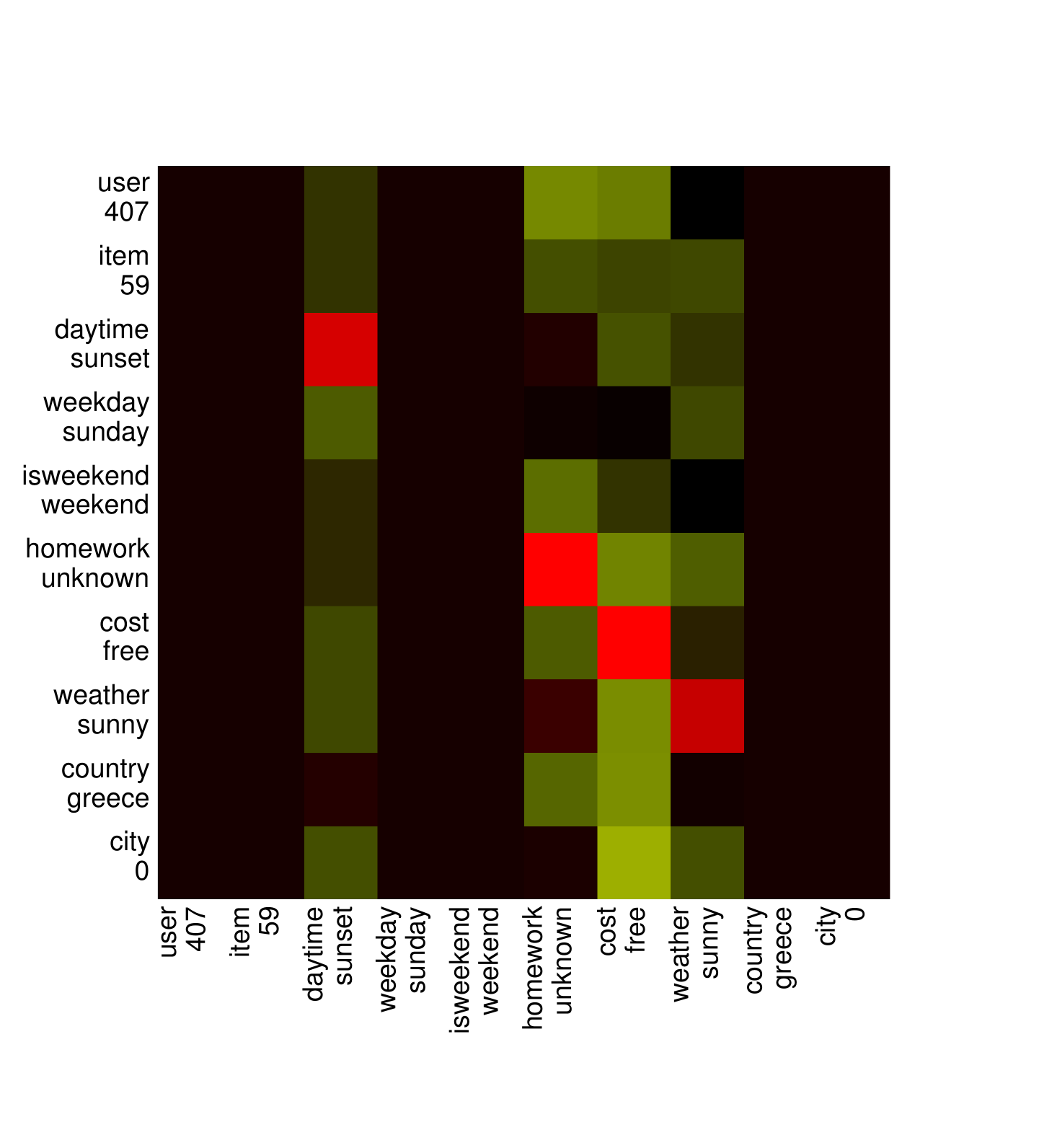}
  \caption{\label{muwA}}
\end{subfigure}
\begin{subfigure}[b]{0.252\textwidth}
 \includegraphics[trim=10 75 110 110, clip=true,height=4.7cm]{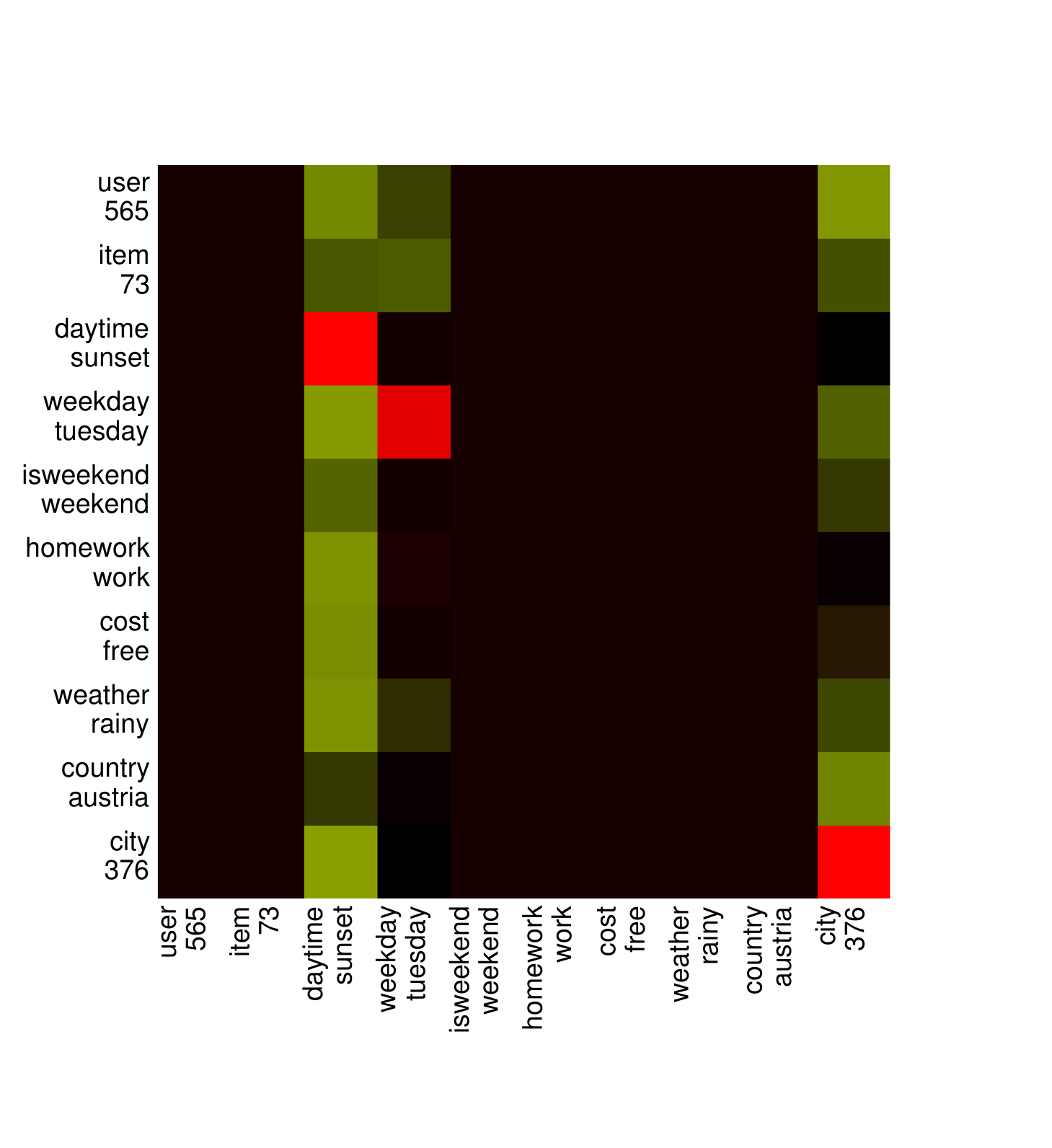}
\caption{\label{muwB}}
\end{subfigure}
\begin{subfigure}[b]{0.252\textwidth}
 \includegraphics[trim=10 75 110 110, clip=true,height=4.7cm]{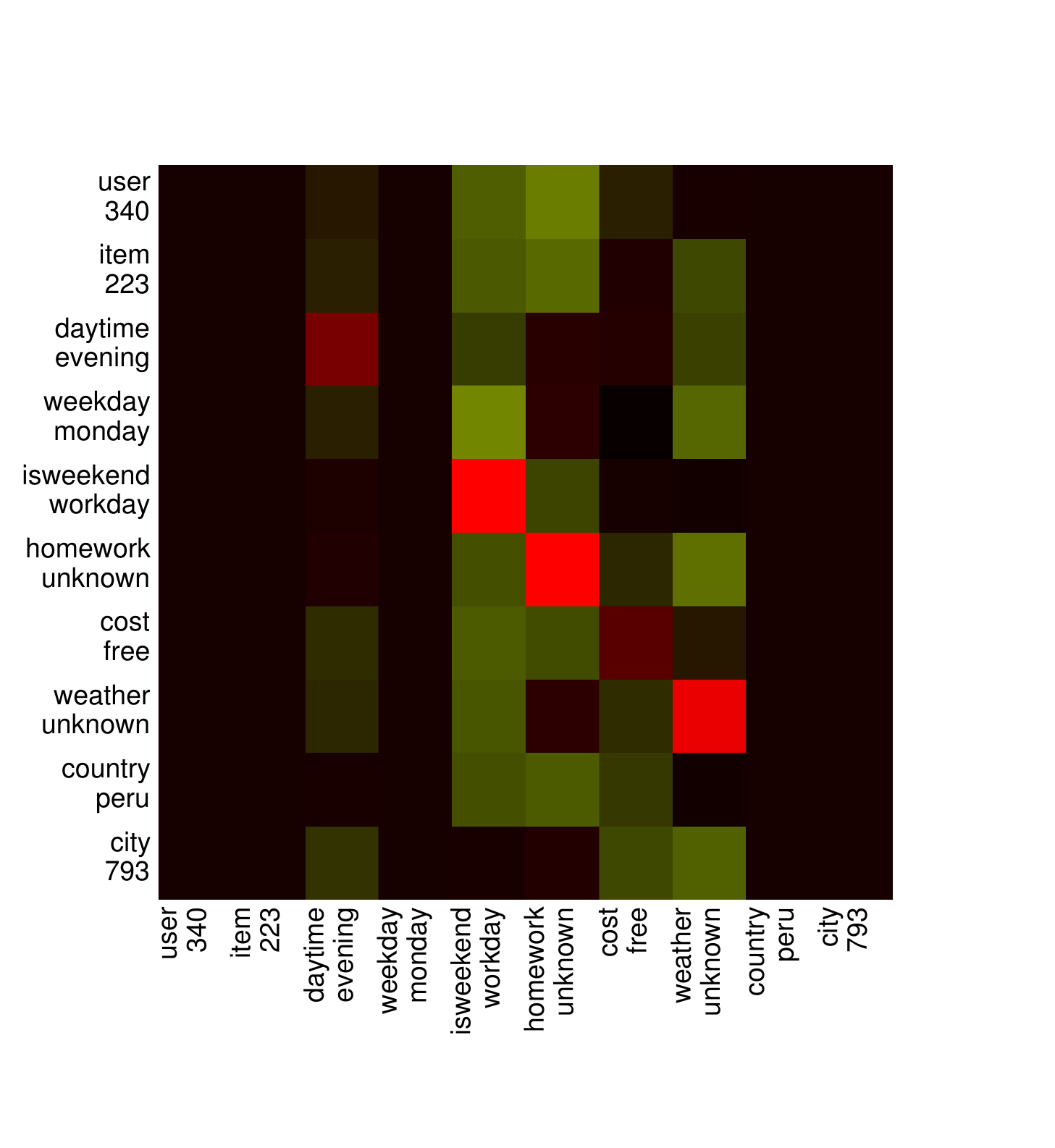}
\caption{\label{muwC}}
\end{subfigure}
\hspace{0.05cm}
\begin{subfigure}[b]{0.01\textwidth}
 \includegraphics[trim=550 40 110 128, clip=true,height=5.27cm]{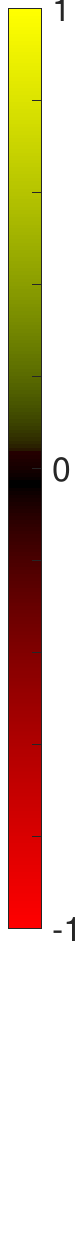}
\end{subfigure}
\caption{Illustration of dependency matrices before and after the refinement step in the inner loop of MDL on Frappe which contains 10 categorical fields. Fig. \ref{w0} shows $\mathbf{W}^{(0)}$ obtained after training is complete. Fig.~\ref{muwA}, \ref{muwB}, and \ref{muwC} show ${\boldsymbol{\mu}_A^{(T)}}\!\circ\!\mathbf{W}_A^{(T)}$,  ${\boldsymbol{\mu}_B^{(T)}}\!\circ\!\mathbf{W}_B^{(T)}$ and ${\boldsymbol{\mu}_C^{(T)}}\!\circ\!\mathbf{W}_C^{(T)}$ for three testing samples $A$, $B$ and $C$ after $T=4$ steps of refinement with Algorithm \ref{algotest}. The three samples are labelled as `{\em used}', `{\em unused}', and `{\em used}', respectively. Matrices were normalized to $(-1, 1)$ range.}\label{matw}
\end{figure*}

Fig.~\ref{globaldep_mu1}, \ref{globaldep_mu2} and \ref{globaldep_mu3} show the trend of Logloss \wrt the hyper-parameter $\lambda$ which controls the closeness of $\boldsymbol{\mu}^{(T)}/\lambda$ to the uniform distribution $\frac{1}{m}\mathbf{1}_m$. Other hyper-parameters are fixed/optimal. As  $\boldsymbol{\mu}^{(T)}$ may contain zeros, KL-divergence may yield infinity. Thus, we use the standard deviation of the elements in $\boldsymbol{\mu}^{(T)}/\lambda$ and average them across all samples as a measure. The figures show that $\lambda$  effectively controls the variance of elements in $\boldsymbol{\mu}^{(T)}/\lambda$ as seen on train, validation and test sets. As $\lambda$ increases, the standard deviation of $\boldsymbol{\mu}^{(T)}/\lambda$ is decreasing, indicating $\boldsymbol{\mu}^{(T)}/\lambda$ is getting close to the uniform distribution $\frac{1}{m}\mathbf{1}_m$. In train set, Logloss is decreasing as $\lambda$ increases, and becomes stable after certain point, but in test and validation set it increases after certain point. This may be caused by overfitting: as $\lambda$ increases, the non-zero elements in $\boldsymbol{\mu}^{(T)}$ increase, leading to denser $\boldsymbol{\mu}^{(T)}$. Compared to small $\lambda$, where sparse $\boldsymbol{\mu}^{(T)}$ may be produced, non-sparse $\boldsymbol{\mu}^{(T)}$ is more prone to overfitting 
despite being closer to the uniform distribution. Thus,  tuning $\lambda$ is beneficial for our MDL model.
%
%

\vspace{0.1cm}
\noindent\textbf{Relation of field dependencies to classification.}
To analyse how field dependencies aid learning informative embeddings and dependencies for the classification task, we present the curve reflecting the correlation between the dependency loss and Logloss generated by means of: 1) training the model on train data using the best hyper-parameters ($T=4$); and 2) testing the model on the train/validation/test data, but with a varying number of steps $T$ for MDL. Subsequently, we plot curves for the Logloss and the average dependency loss per instance \wrt the number of steps. 

Fig.~\ref{depabl3} shows that on the train set both the dependency loss and task-specific Logloss decrease with the step $T$, {\ZB showing correlation of Logloss with  the dependency loss.} Fig.~\ref{depabl2} and \ref{depabl1} show on validation and test sets that the dependency loss decreases together with Logloss  to a certain point {\ZB (to the point of $T=4$ which we use in training the model)} before overfitting gradually emerges, and thus Logloss starts increasing. This means both losses can align to some degree, and our MDL generalises well to unseen (validation/test) sets {\ZB on the step $T$}. As expected, the AUC follows the same trend as Logloss (to be precise, the maximum of AUC coincides with the minimum of Logloss). Plots of AUC are in Appendices.  The results imply that the dependency loss acts as a surrogate for the classification task up to the point where overfitting to the surrogate task becomes apparent.  

\subsection{Case Study}
Fig.~\ref{matw} illustrates the refinement step dependency matrices before and after the refinement step in the inner loop of MDL. This qualitative study is performed on Frappe (with 10 categorical fields). Fig.~\ref{w0} shows $\mathbf{W}^{(0)}$ after training with Algorithm \ref{algo}. 

Subsequently, we choose test samples $A$, $B$ and $C$ whose labels are `{\em used}', `{\em unused}', and `{\em used}', respectively. Fig.~\ref{muwA}, \ref{muwB}, and \ref{muwC} show  dependency matrices $\mathbf{W}_A^{(T)}$, $\mathbf{W}_B^{(T)}$ and $\mathbf{W}_C^{(T)}$ modulated by weight vectors ${\boldsymbol{\mu}_A^{(T)}}$, ${\boldsymbol{\mu}_B^{(T)}}$ and ${\boldsymbol{\mu}_C^{(T)}}$ after $T=4$ steps of refinement via Algorithm \ref{algotest}. 

Notice that column `{\em daytime: sunset}' in Fig.~\ref{muwB} correlates well with several categorical fields/features, and with label `{\em unused}'. In contrast,  columns `{\em daytime: sunset}' and `{\em daytime: evening}' in Fig.~\ref{muwA} and \ref{muwC} are less correlated with other categorical fields/features, which suggests a decline in correlation with labels `{\em used}' of both samples.

Moreover, column `{\em cost: free}' in Fig.~\ref{muwB} is completely zeroed (no correlation), whereas columns denoted `{\em cost: free}' in Fig.~\ref{muwA} and \ref{muwC} are non-zero, indicating correlation with several categorical fields, and labels `{\em used}' of both samples. Naturally, categorical field/feature `{\em cost: free}'  correlates well with labels `{\em used }' of `{\em user: 407}' and  `{\em user: 340}', indicating that ${\boldsymbol{\mu}_A^{(T)}}$ and ${\boldsymbol{\mu}_C^{(T)}}$  activated correctly column `{\em cost: free}' for samples $A$ and $C$, whereas ${\boldsymbol{\mu}_B^{(T)}}$  suppressed correctly column `{\em cost: free}' for sample $B$ with label `{\em unused}'. Both examples shows that MDL discovers localized dependencies well.

\section{Conclusions}
In this paper, we have explored a new way of learning on categorical data. We have shown that the dependency loss can be used as a surrogate for the target task via the meta-learning regime. The learned dependencies are beneficial to the performance of the proposed model. Instead of modelling feature dependencies at the global level, we model the field dependencies with a much lower number of parameters and adjust per instance to capture nuanced local field dependencies. Importantly, our ablation studies confirm that both losses align to certain degree, which is especially helpful in the testing stage, where labels are not available for the adaptation to the incoming data. 

{\ZB
Most of methods have their limitations, and ours is no exception. The meta-learning algorithm used in this paper introduces several additional parameters, which require tuning. The optimal step size $\eta$ involved in the inner steps can vary a lot depending on other hyper-parameters and datasets. We searched the values of $\eta$ in a small set which might be insufficient. This however provides a future direction of our research on developing a meta-learning regime that help generalise the hyper-parameters across instances, altogether with the generalisation of loss function in current work.
}

\section*{Acknowledgments}

This work was funded by CSIRO's Machine Learning and Artificial Intelligence Future Science Platform (MLAI FSP).

\bibliographystyle{IEEEtranN}
\bibliography{ref}

\begin{thebibliography}{76}
\providecommand{\natexlab}[1]{#1}
\providecommand{\url}[1]{#1}
\csname url@samestyle\endcsname
\providecommand{\newblock}{\relax}
\providecommand{\bibinfo}[2]{#2}
\providecommand{\BIBentrySTDinterwordspacing}{\spaceskip=0pt\relax}
\providecommand{\BIBentryALTinterwordstretchfactor}{4}
\providecommand{\BIBentryALTinterwordspacing}{\spaceskip=\fontdimen2\font plus
\BIBentryALTinterwordstretchfactor\fontdimen3\font minus
  \fontdimen4\font\relax}
\providecommand{\BIBforeignlanguage}[2]{{%
\expandafter\ifx\csname l@#1\endcsname\relax
\typeout{** WARNING: IEEEtranN.bst: No hyphenation pattern has been}%
\typeout{** loaded for the language `#1'. Using the pattern for}%
\typeout{** the default language instead.}%
\else
\language=\csname l@#1\endcsname
\fi
#2}}
\providecommand{\BIBdecl}{\relax}
\BIBdecl

\bibitem[Chapelle et~al.(2014)Chapelle, Manavoglu, and
  Rosales]{chapelle2014simple}
O.~Chapelle, E.~Manavoglu, and R.~Rosales, ``Simple and scalable response
  prediction for display advertising,'' \emph{ACM Transactions on Intelligent
  Systems and Technology (TIST)}, vol.~5, no.~4, pp. 1--34, 2014.

\bibitem[Agichtein et~al.(2006)Agichtein, Brill, Dumais, and
  Ragno]{agichtein2006learning}
E.~Agichtein, E.~Brill, S.~Dumais, and R.~Ragno, ``Learning user interaction
  models for predicting web search result preferences,'' in \emph{Proceedings
  of the 29th annual international ACM SIGIR conference on Research and
  development in information retrieval}, 2006, pp. 3--10.

\bibitem[Koren et~al.(2009)Koren, Bell, and Volinsky]{koren2009matrix}
Y.~Koren, R.~Bell, and C.~Volinsky, ``Matrix factorization techniques for
  recommender systems,'' \emph{Computer}, vol.~42, no.~8, pp. 30--37, 2009.

\bibitem[Liu et~al.(2023)Liu, Liu, Zhang, To, Nasrallah, and
  Chandra]{LIU2023120267}
L.~Liu, S.~Liu, L.~Zhang, X.~V. To, F.~Nasrallah, and S.~S. Chandra, ``Cascaded
  multi-modal mixing transformers for alzheimer’s disease classification with
  incomplete data,'' \emph{NeuroImage}, vol. 277, p. 120267, 2023.

\bibitem[Cheng et~al.(2016)Cheng, Koc, Harmsen, Shaked, Chandra, Aradhye,
  Anderson, Corrado, Chai, Ispir, et~al.]{cheng2016wide}
H.-T. Cheng, L.~Koc, J.~Harmsen, T.~Shaked, T.~Chandra, H.~Aradhye,
  G.~Anderson, G.~Corrado, W.~Chai, M.~Ispir \emph{et~al.}, ``Wide \& deep
  learning for recommender systems,'' in \emph{Proceedings of the 1st workshop
  on deep learning for recommender systems}, 2016, pp. 7--10.

\bibitem[Finn et~al.(2017)Finn, Abbeel, and Levine]{finn2017model}
C.~Finn, P.~Abbeel, and S.~Levine, ``Model-agnostic meta-learning for fast
  adaptation of deep networks,'' in \emph{International Conference on Machine
  Learning}.\hskip 1em plus 0.5em minus 0.4em\relax PMLR, 2017, pp. 1126--1135.

\bibitem[Metz et~al.(2018)Metz, Maheswaranathan, Cheung, and
  Sohl-Dickstein]{metz2018meta}
L.~Metz, N.~Maheswaranathan, B.~Cheung, and J.~Sohl-Dickstein, ``Meta-learning
  update rules for unsupervised representation learning,'' in
  \emph{International Conference on Learning Representations}, 2018.

\bibitem[adc()]{adc}
``2020 tencent advertising algorithm contest,''
  \url{https://algo.qq.com/archive.html?&lang=en}, accessed: 2022-05-05.

\bibitem[Zhang et~al.(2022{\natexlab{a}})Zhang, Shen, Zhang, Xu, Li, Yao, and
  Yu]{ZhangSZXLYY22}
L.~Zhang, J.~Shen, J.~Zhang, J.~Xu, Z.~Li, Y.~Yao, and L.~Yu, ``Multimodal
  marketing intent analysis for effective targeted advertising,'' \emph{{IEEE}
  Trans. Multim.}, vol.~24, pp. 1830--1843, 2022.

\bibitem[Pourahmadi(2013)]{pourahmadi2013high}
M.~Pourahmadi, \emph{High-dimensional covariance estimation: with
  high-dimensional data}.\hskip 1em plus 0.5em minus 0.4em\relax John Wiley \&
  Sons, 2013, vol. 882.

\bibitem[Lin et~al.(2018)Lin, Maji, and Koniusz]{Lin_2018_ECCV}
T.-Y. Lin, S.~Maji, and P.~Koniusz, ``Second-order democratic aggregation,'' in
  \emph{Proceedings of the European Conference on Computer Vision}, September
  2018.

\bibitem[Zhang et~al.(2023)Zhang, Zhu, Song, Koniusz, and King]{sfa_yifei}
Y.~Zhang, H.~Zhu, Z.~Song, P.~Koniusz, and I.~King, ``Spectral feature
  augmentation for graph contrastive learning and beyond,'' in
  \emph{Proceedings of the AAAI Conference on Artificial Intelligence}, 2023.

\bibitem[Rahman et~al.(2023)Rahman, Koniusz, Wang, Zhou, Moghadam, and
  Sun]{simon_isice}
S.~Rahman, P.~Koniusz, L.~Wang, L.~Zhou, P.~Moghadam, and C.~Sun, ``Learning
  partial correlation based deep visual representation for image
  classification,,'' in \emph{Proceedings of the IEEE/CVF Conference on
  Computer Vision and Pattern Recognition}, 2023.

\bibitem[Levitin and Polyak(1966)]{levitin1966constrained}
E.~Levitin and B.~Polyak, ``Constrained minimization methods,'' \emph{USSR
  Computational Mathematics and Mathematical Physics}, vol.~6, no.~5, pp.
  1--50, 1966.

\bibitem[ApS(2019)]{mosek}
\BIBentryALTinterwordspacing
M.~ApS, \emph{The MOSEK optimization toolbox for MATLAB manual. Version 9.0.},
  2019. [Online]. Available: \url{http://docs.mosek.com/9.0/toolbox/index.html}
\BIBentrySTDinterwordspacing

\bibitem[Cplex(2009)]{cplex2009v12}
I.~I. Cplex, ``V12. 1: User’s manual for cplex,'' \emph{International
  Business Machines Corporation}, vol.~46, no.~53, p. 157, 2009.

\bibitem[Paszke et~al.(2019)Paszke, Gross, Massa, Lerer, Bradbury, Chanan,
  Killeen, Lin, Gimelshein, Antiga, et~al.]{paszke2019pytorch}
A.~Paszke, S.~Gross, F.~Massa, A.~Lerer, J.~Bradbury, G.~Chanan, T.~Killeen,
  Z.~Lin, N.~Gimelshein, L.~Antiga \emph{et~al.}, ``Pytorch: An imperative
  style, high-performance deep learning library,'' \emph{Advances in Neural
  Information Processing Systems}, vol.~32, pp. 8026--8037, 2019.

\bibitem[Laha et~al.(2018)Laha, Chemmengath, Agrawal, Khapra, Sankaranarayanan,
  and Ramaswamy]{laha2018controllable}
A.~Laha, S.~A. Chemmengath, P.~Agrawal, M.~Khapra, K.~Sankaranarayanan, and
  H.~G. Ramaswamy, ``On controllable sparse alternatives to softmax,''
  \emph{Advances in neural information processing systems}, vol.~31, 2018.

\bibitem[Abadi et~al.(2016)Abadi, Barham, Chen, Chen, Davis, Dean, Devin,
  Ghemawat, Irving, Isard, et~al.]{abadi2016tensorflow}
M.~Abadi, P.~Barham, J.~Chen, Z.~Chen, A.~Davis, J.~Dean, M.~Devin,
  S.~Ghemawat, G.~Irving, M.~Isard \emph{et~al.}, ``Tensorflow: A system for
  large-scale machine learning,'' in \emph{12th $\{$USENIX$\}$ symposium on
  operating systems design and implementation ($\{$OSDI$\}$ 16)}, 2016, pp.
  265--283.

\bibitem[Liu et~al.(2020)Liu, Liu, Zhang, and Chen]{liu2020dnn2lr}
Z.~Liu, Q.~Liu, H.~Zhang, and Y.~Chen, ``Dnn2lr: Interpretation-inspired
  feature crossing for real-world tabular data,'' \emph{arXiv preprint
  arXiv:2008.09775}, 2020.

\bibitem[Liu et~al.(2021)Liu, Liu, Zhang, Chen, and Zhu]{liu2021mining}
Q.~Liu, Z.~Liu, H.~Zhang, Y.~Chen, and J.~Zhu, ``Mining cross features for
  financial credit risk assessment,'' in \emph{Proceedings of the 30th ACM
  international conference on information \& knowledge management}, 2021, pp.
  1069--1078.

\bibitem[Popov et~al.(2020)Popov, Morozov, and Babenko]{Popov2020Neural}
\BIBentryALTinterwordspacing
S.~Popov, S.~Morozov, and A.~Babenko, ``Neural oblivious decision ensembles for
  deep learning on tabular data,'' in \emph{International Conference on
  Learning Representations}, 2020. [Online]. Available:
  \url{https://openreview.net/forum?id=r1eiu2VtwH}
\BIBentrySTDinterwordspacing

\bibitem[Arik and Pfister(2021)]{arik2021tabnet}
S.~{\"O}. Arik and T.~Pfister, ``Tabnet: Attentive interpretable tabular
  learning,'' in \emph{Proceedings of the AAAI Conference on Artificial
  Intelligence}, vol.~35, no.~8, 2021, pp. 6679--6687.

\bibitem[Huang et~al.(2020)Huang, Khetan, Cvitkovic, and
  Karnin]{huang2020tabtransformer}
X.~Huang, A.~Khetan, M.~Cvitkovic, and Z.~Karnin, ``Tabtransformer: Tabular
  data modeling using contextual embeddings,'' \emph{arXiv preprint
  arXiv:2012.06678}, 2020.

\bibitem[Somepalli et~al.(2022)Somepalli, Schwarzschild, Goldblum, Bruss, and
  Goldstein]{somepalli2022saint}
G.~Somepalli, A.~Schwarzschild, M.~Goldblum, C.~B. Bruss, and T.~Goldstein,
  ``Saint: Improved neural networks for tabular data via row attention and
  contrastive pre-training,'' in \emph{NeurIPS 2022 First Table Representation
  Workshop}, 2022.

\bibitem[Shwartz-Ziv and Armon(2022)]{shwartz2022tabular}
R.~Shwartz-Ziv and A.~Armon, ``Tabular data: Deep learning is not all you
  need,'' \emph{Information Fusion}, vol.~81, pp. 84--90, 2022.

\bibitem[Gorishniy et~al.(2021)Gorishniy, Rubachev, Khrulkov, and
  Babenko]{gorishniy2021revisiting}
Y.~Gorishniy, I.~Rubachev, V.~Khrulkov, and A.~Babenko, ``Revisiting deep
  learning models for tabular data,'' \emph{Advances in Neural Information
  Processing Systems}, vol.~34, pp. 18\,932--18\,943, 2021.

\bibitem[Rendle(2010)]{rendle2010factorization}
S.~Rendle, ``Factorization machines,'' in \emph{2010 IEEE International
  Conference on Data Mining}.\hskip 1em plus 0.5em minus 0.4em\relax IEEE,
  2010, pp. 995--1000.

\bibitem[Blondel et~al.(2016)Blondel, Fujino, Ueda, and
  Ishihata]{blondel2016higher}
M.~Blondel, A.~Fujino, N.~Ueda, and M.~Ishihata, ``Higher-order factorization
  machines,'' in \emph{Proceedings of the 30th International Conference on
  Neural Information Processing Systems}, 2016, pp. 3359--3367.

\bibitem[Blondel et~al.(2015)Blondel, Fujino, and Ueda]{blondel2015convex}
M.~Blondel, A.~Fujino, and N.~Ueda, ``Convex factorization machines,'' in
  \emph{Joint European Conference on Machine Learning and Knowledge Discovery
  in Databases}.\hskip 1em plus 0.5em minus 0.4em\relax Springer, 2015, pp.
  19--35.

\bibitem[Pan et~al.(2018)Pan, Xu, Ruiz, Zhao, Pan, Sun, and Lu]{pan2018field}
J.~Pan, J.~Xu, A.~L. Ruiz, W.~Zhao, S.~Pan, Y.~Sun, and Q.~Lu, ``Field-weighted
  factorization machines for click-through rate prediction in display
  advertising,'' in \emph{Proceedings of the 2018 World Wide Web Conference},
  2018, pp. 1349--1357.

\bibitem[Lu et~al.(2020)Lu, Yu, Chang, Wang, Li, and Yuan]{lu2020dual}
W.~Lu, Y.~Yu, Y.~Chang, Z.~Wang, C.~Li, and B.~Yuan, ``A dual input-aware
  factorization machine for ctr prediction,'' in \emph{Proceedings of the 29th
  International Joint Conference on Artificial Intelligence}, 2020, pp.
  3139--3145.

\bibitem[Yu et~al.(2019)Yu, Wang, and Yuan]{yu2019input}
Y.~Yu, Z.~Wang, and B.~Yuan, ``An input-aware factorization machine for sparse
  prediction,'' in \emph{IJCAI}, 2019.

\bibitem[Xiao et~al.(2017)Xiao, Ye, He, Zhang, Wu, and
  Chua]{xiao2017attentional}
J.~Xiao, H.~Ye, X.~He, H.~Zhang, F.~Wu, and T.-S. Chua, ``Attentional
  factorization machines: Learning the weight of feature interactions via
  attention networks,'' in \emph{IJCAI}, 2017.

\bibitem[Wang et~al.(2020)Wang, Ma, Zhang, Wang, Ren, and
  Sun]{wang2020attention}
Z.~Wang, J.~Ma, Y.~Zhang, Q.~Wang, J.~Ren, and P.~Sun,
  ``Attention-over-attention field-aware factorization machine,'' in
  \emph{Proceedings of the AAAI Conference on Artificial Intelligence},
  vol.~34, 2020, pp. 6323--6330.

\bibitem[Cheng et~al.(2020)Cheng, Shen, and Huang]{cheng2020adaptive}
W.~Cheng, Y.~Shen, and L.~Huang, ``Adaptive factorization network: Learning
  adaptive-order feature interactions,'' in \emph{Proceedings of the AAAI
  Conference on Artificial Intelligence}, 2020, pp. 3609--3616.

\bibitem[He and Chua(2017)]{he2017neural}
X.~He and T.-S. Chua, ``Neural factorization machines for sparse predictive
  analytics,'' in \emph{Proceedings of the 40th International ACM SIGIR
  conference on Research and Development in Information Retrieval}, 2017, pp.
  355--364.

\bibitem[Guo et~al.(2017)Guo, Tang, Ye, Li, and He]{guo2017deepfm}
H.~Guo, R.~Tang, Y.~Ye, Z.~Li, and X.~He, ``Deepfm: a factorization-machine
  based neural network for ctr prediction,'' in \emph{Proceedings of the 26th
  International Joint Conference on Artificial Intelligence}, 2017, pp.
  1725--1731.

\bibitem[Lian et~al.(2018)Lian, Zhou, Zhang, Chen, Xie, and
  Sun]{lian2018xdeepfm}
J.~Lian, X.~Zhou, F.~Zhang, Z.~Chen, X.~Xie, and G.~Sun, ``xdeepfm: Combining
  explicit and implicit feature interactions for recommender systems,'' in
  \emph{Proceedings of the 24th ACM SIGKDD International Conference on
  Knowledge Discovery \& Data Mining}, 2018, pp. 1754--1763.

\bibitem[Hong et~al.(2019)Hong, Huang, and Chen]{hong2019interaction}
F.~Hong, D.~Huang, and G.~Chen, ``Interaction-aware factorization machines for
  recommender systems,'' in \emph{Proceedings of the AAAI Conference on
  Artificial Intelligence}, vol.~33, 2019, pp. 3804--3811.

\bibitem[Liang et~al.(2020)Liang, Xu, Sun, and Honavar]{liang2020lmlfm}
J.~Liang, D.~Xu, Y.~Sun, and V.~Honavar, ``Lmlfm: Longitudinal multi-level
  factorization machine,'' in \emph{Proceedings of the AAAI Conference on
  Artificial Intelligence}, 2020, pp. 4811--4818.

\bibitem[Liu et~al.(2017)Liu, Zhang, Zhao, Zhou, and Sun]{liu2017locally}
C.~Liu, T.~Zhang, P.~Zhao, J.~Zhou, and J.~Sun, ``Locally linear factorization
  machines.'' in \emph{IJCAI}, 2017, pp. 2294--2300.

\bibitem[Chen et~al.(2019{\natexlab{a}})Chen, Zheng, Zhao, Jiang, Ma, and
  Huang]{chen2019generalized}
X.~Chen, Y.~Zheng, P.~Zhao, Z.~Jiang, W.~Ma, and J.~Huang, ``A generalized
  locally linear factorization machine with supervised variational encoding,''
  \emph{IEEE Transactions on Knowledge and Data Engineering}, vol.~32, no.~6,
  pp. 1036--1049, 2019.

\bibitem[Juan et~al.(2016)Juan, Zhuang, Chin, and Lin]{juan2016field}
Y.~Juan, Y.~Zhuang, W.-S. Chin, and C.-J. Lin, ``Field-aware factorization
  machines for ctr prediction,'' in \emph{Proceedings of the 10th ACM
  conference on recommender systems}, 2016, pp. 43--50.

\bibitem[Qu et~al.(2018)Qu, Fang, Zhang, Tang, Niu, Guo, Yu, and
  He]{qu2018product}
Y.~Qu, B.~Fang, W.~Zhang, R.~Tang, M.~Niu, H.~Guo, Y.~Yu, and X.~He,
  ``Product-based neural networks for user response prediction over multi-field
  categorical data,'' \emph{ACM Transactions on Information Systems (TOIS)},
  vol.~37, no.~1, pp. 1--35, 2018.

\bibitem[Chen et~al.(2019{\natexlab{b}})Chen, Zheng, Wang, Ma, and
  Huang]{chen2019rafm}
X.~Chen, Y.~Zheng, J.~Wang, W.~Ma, and J.~Huang, ``Rafm: rank-aware
  factorization machines,'' in \emph{International Conference on Machine
  Learning}.\hskip 1em plus 0.5em minus 0.4em\relax PMLR, 2019, pp. 1132--1140.

\bibitem[Li et~al.(2020)Li, Zhang, Gong, Yao, and Wu]{li2020field}
Z.~Li, J.~Zhang, Y.~Gong, Y.~Yao, and Q.~Wu, ``Field-wise learning for
  multi-field categorical data,'' \emph{Advances in Neural Information
  Processing Systems}, vol.~33, 2020.

\bibitem[Azur et~al.(2011)Azur, Stuart, Frangakis, and Leaf]{azur2011multiple}
M.~J. Azur, E.~A. Stuart, C.~Frangakis, and P.~J. Leaf, ``Multiple imputation
  by chained equations: what is it and how does it work?'' \emph{International
  journal of methods in psychiatric research}, vol.~20, no.~1, pp. 40--49,
  2011.

\bibitem[Yoon et~al.(2020)Yoon, Zhang, Jordon, and van~der
  Schaar]{yoon2020vime}
J.~Yoon, Y.~Zhang, J.~Jordon, and M.~van~der Schaar, ``Vime: Extending the
  success of self-and semi-supervised learning to tabular domain,''
  \emph{Advances in Neural Information Processing Systems}, vol.~33, 2020.

\bibitem[Ucar et~al.(2021)Ucar, Hajiramezanali, and Edwards]{ucar2021subtab}
T.~Ucar, E.~Hajiramezanali, and L.~Edwards, ``Subtab: Subsetting features of
  tabular data for self-supervised representation learning,'' \emph{Advances in
  Neural Information Processing Systems}, vol.~34, 2021.

\bibitem[Zheng et~al.(2014)Zheng, Zhang, and Larochelle]{zheng2014topic}
Y.~Zheng, Y.-J. Zhang, and H.~Larochelle, ``Topic modeling of multimodal data:
  an autoregressive approach,'' in \emph{Proceedings of the IEEE conference on
  computer vision and pattern recognition}, 2014, pp. 1370--1377.

\bibitem[Zheng et~al.(2015)Zheng, Zhang, and Larochelle]{zheng2015deep}
------, ``A deep and autoregressive approach for topic modeling of multimodal
  data,'' \emph{IEEE transactions on pattern analysis and machine
  intelligence}, vol.~38, no.~6, pp. 1056--1069, 2015.

\bibitem[Snell et~al.(2017)Snell, Swersky, and Zemel]{protonet}
J.~Snell, K.~Swersky, and R.~Zemel, ``Prototypical networks for few-shot
  learning,'' in \emph{Advances in Neural Information Processing Systems},
  2017.

\bibitem[Sung et~al.(2018)Sung, Yang, Zhang, Xiang, Torr, and
  Hospedales]{relationnet}
F.~Sung, Y.~Yang, L.~Zhang, T.~Xiang, P.~H.~S. Torr, and T.~M. Hospedales,
  ``Learning to compare: Relation network for few-shot learning,'' in
  \emph{Proceedings of the IEEE/CVF Conference on Computer Vision and Pattern
  Recognition}, 2018.

\bibitem[Simon et~al.(2022)Simon, Koniusz, and Harandi]{simon2022meta}
C.~Simon, P.~Koniusz, and M.~Harandi, ``Meta-learning for multi-label few-shot
  classification,'' in \emph{Proceedings of the IEEE/CVF Winter Conference on
  Applications of Computer Vision}, 2022.

\bibitem[Shi et~al.(2022)Shi, Yu, Liu, Campbell, Koniusz, and
  Li]{Shi2022Accurate3C}
Y.~Shi, X.~Yu, L.~Liu, D.~Campbell, P.~Koniusz, and H.~Li, ``Accurate 3-dof
  camera geo-localization via ground-to-satellite image matching,'' \emph{IEEE
  Transactions on Pattern Analysis and Machine Intelligence}, vol.~45, pp.
  2682--2697, 2022.

\bibitem[Zhang et~al.(2022{\natexlab{b}})Zhang, Torr, and
  Koniusz]{hongguangACCV}
H.~Zhang, P.~H.~S. Torr, and P.~Koniusz, ``Improving few-shot learning by
  spatially-aware matching and crosstransformer,'' in \emph{Proceedings of the
  Asian Conference on Computer Vision}, 2022.

\bibitem[Lee et~al.(2019)Lee, Maji, Ravichandran, and Soatto]{lee2019meta}
K.~Lee, S.~Maji, A.~Ravichandran, and S.~Soatto, ``Meta-learning with
  differentiable convex optimization,'' in \emph{Proceedings of the IEEE/CVF
  Conference on Computer Vision and Pattern Recognition}, 2019, pp.
  10\,657--10\,665.

\bibitem[Zhu and Koniusz(2022)]{zhu2022ease}
H.~Zhu and P.~Koniusz, ``{EASE}: Unsupervised discriminant subspace learning
  for transductive few-shot learning,'' in \emph{Proceedings of the IEEE/CVF
  Conference on Computer Vision and Pattern Recognition}, 2022, pp. 9078--9088.

\bibitem[Zhu and Koniusz(2023)]{zhu2023protolp}
------, ``Transductive few-shot learning with prototype-based label propagation
  by iterative graph refinement,'' in \emph{Proceedings of the IEEE/CVF
  Conference on Computer Vision and Pattern Recognition}, 2023.

\bibitem[Zhang et~al.(2020)Zhang, Luo, Wang, and Koniusz]{zhang2020sopaccv}
S.~Zhang, D.~Luo, L.~Wang, and P.~Koniusz, ``Few-shot object detection by
  second-order pooling,'' in \emph{Proceedings of the Asian Conference on
  Computer Vision}, 2020.

\bibitem[Zhang et~al.(2022{\natexlab{c}})Zhang, Wang, Murray, and
  Koniusz]{zhang2022kernelized}
S.~Zhang, L.~Wang, N.~Murray, and P.~Koniusz, ``Kernelized few-shot object
  detection with efficient integral aggregation,'' in \emph{Proceedings of the
  IEEE/CVF Conference on Computer Vision and Pattern Recognition}, 2022.

\bibitem[Zhang et~al.(2022{\natexlab{d}})Zhang, Murray, Wang, and
  Koniusz]{zhang2022time}
S.~Zhang, N.~Murray, L.~Wang, and P.~Koniusz, ``{Time-rEversed DiffusioN tEnsor
  Transformer: A New TENET of Few-Shot Object Detection},'' in
  \emph{Proceedings of the European Conference on Computer Vision}.\hskip 1em
  plus 0.5em minus 0.4em\relax Springer, 2022.

\bibitem[Shaban et~al.(2017)Shaban, Bansal, Liu, Essa, and
  Boots]{shaban2017oslsm}
A.~Shaban, S.~Bansal, Z.~Liu, I.~Essa, and B.~Boots, ``One-shot learning for
  semantic segmentation,'' in \emph{Proceedings of the British Machine Vision
  Conference}, 2017.

\bibitem[Kang et~al.(2023)Kang, Koniusz, Cho, and Murray]{dahyun2023seg}
D.~Kang, P.~Koniusz, M.~Cho, and N.~Murray, ``Distilling self-supervised vision
  transformers for weakly-supervised few-shot classification \& segmentation,''
  in \emph{Proceedings of the IEEE/CVF Conference on Computer Vision and
  Pattern Recognition}, 2023.

\bibitem[Wang and Koniusz(2022{\natexlab{a}})]{udtw_eccv22}
L.~Wang and P.~Koniusz, ``Uncertainty-dtw for time series and sequences,'' in
  \emph{Proceedings of the European Conference on Computer Vision}.\hskip 1em
  plus 0.5em minus 0.4em\relax Springer, 2022, pp. 176--195.

\bibitem[Wang and Koniusz(2022{\natexlab{b}})]{Wang_2022_ACCV}
------, ``Temporal-viewpoint transportation plan for skeletal few-shot action
  recognition,'' in \emph{Proceedings of the Asian Conference on Computer
  Vision}, 2022.

\bibitem[Lu and Koniusz(2022)]{lu2022few}
C.~Lu and P.~Koniusz, ``Few-shot keypoint detection with uncertainty learning
  for unseen species,'' in \emph{Proceedings of the IEEE/CVF Conference on
  Computer Vision and Pattern Recognition}, 2022.

\bibitem[Lu et~al.(2021)Lu, Ye, Gan, and Zhan]{lu2021towards}
S.~Lu, H.-J. Ye, L.~Gan, and D.-C. Zhan, ``Towards enabling meta-learning from
  target models,'' \emph{Advances in Neural Information Processing Systems},
  vol.~34, 2021.

\bibitem[Harper and Konstan(2015)]{harper2015movielens}
F.~M. Harper and J.~A. Konstan, ``The movielens datasets: History and
  context,'' \emph{Acm transactions on interactive intelligent systems (tiis)},
  vol.~5, no.~4, pp. 1--19, 2015.

\bibitem[Baltrunas et~al.(2015)Baltrunas, Church, Karatzoglou, and
  Oliver]{baltrunas2015frappe}
L.~Baltrunas, K.~Church, A.~Karatzoglou, and N.~Oliver, ``Frappe: Understanding
  the usage and perception of mobile app recommendations in-the-wild,''
  \emph{arXiv preprint arXiv:1505.03014}, 2015.

\bibitem[Wang et~al.(2017)Wang, Fu, Fu, and Wang]{wang2017deep}
R.~Wang, B.~Fu, G.~Fu, and M.~Wang, ``Deep \& cross network for ad click
  predictions,'' in \emph{Proceedings of the ADKDD'17}, 2017, pp. 1--7.

\bibitem[Dunnett(1955)]{Dunnett1955}
C.~Dunnett, ``A multiple comparison procedure for comparing several treatments
  with a control,'' \emph{Journal of the American Statistical Association},
  vol.~50, pp. 1096--1121, 1955.

\bibitem[Mao et~al.(2023)Mao, Zhu, Su, Cai, Li, and Dong]{kelongfinalmlp2023}
K.~Mao, J.~Zhu, L.~Su, G.~Cai, Y.~Li, and Z.~Dong, ``Finalmlp: An enhanced
  two-stream mlp model for ctr prediction,'' in \emph{Proceedings of the AAAI
  Conference on Artificial Intelligence}, 2023.

\bibitem[Duchi et~al.(2011)Duchi, Hazan, and Singer]{duchi2011adaptive}
J.~Duchi, E.~Hazan, and Y.~Singer, ``Adaptive subgradient methods for online
  learning and stochastic optimization.'' \emph{Journal of machine learning
  research}, vol.~12, no.~7, 2011.

\bibitem[Kingma and Ba(2015)]{adamKingmaB14}
D.~P. Kingma and J.~Ba, ``Adam: {A} method for stochastic optimization,'' in
  \emph{3rd International Conference on Learning Representations}, 2015.

\end{thebibliography}


\begin{thebibliography}{1}
\providecommand{\natexlab}[1]{#1}
\providecommand{\url}[1]{#1}
\csname url@samestyle\endcsname
\providecommand{\newblock}{\relax}
\providecommand{\bibinfo}[2]{#2}
\providecommand{\BIBentrySTDinterwordspacing}{\spaceskip=0pt\relax}
\providecommand{\BIBentryALTinterwordstretchfactor}{4}
\providecommand{\BIBentryALTinterwordspacing}{\spaceskip=\fontdimen2\font plus
\BIBentryALTinterwordstretchfactor\fontdimen3\font minus
  \fontdimen4\font\relax}
\providecommand{\BIBforeignlanguage}[2]{{%
\expandafter\ifx\csname l@#1\endcsname\relax
\typeout{** WARNING: IEEEtranN.bst: No hyphenation pattern has been}%
\typeout{** loaded for the language `#1'. Using the pattern for}%
\typeout{** the default language instead.}%
\else
\language=\csname l@#1\endcsname
\fi
#2}}
\providecommand{\BIBdecl}{\relax}
\BIBdecl

\bibitem[Hothorn et~al.(2008)Hothorn, Bretz, and
  Westfall]{hothorn2008simultaneous_supp}
T.~Hothorn, F.~Bretz, and P.~Westfall, ``Simultaneous inference in general
  parametric models,'' \emph{Biometrical Journal: Journal of Mathematical
  Methods in Biosciences}, vol.~50, no.~3, pp. 346--363, 2008.

\end{thebibliography}


\vspace{-1.5cm}
\begin{IEEEbiography}[{\includegraphics[width=1in,height=1.25in,clip,keepaspectratio]{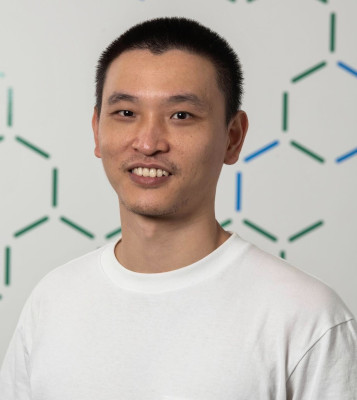}}]{Zhibin Li} is a postdoctoral research fellow at Machine Learning \& Artificial Intelligence Future Science Platform, CSIRO, Australia. He obtained his Ph.D. degree in Information Systems from the University of Technology Sydney in 2021. His research interests include spatio-temporal data mining, traffic prediction, and recommender systems. He has published over 10 papers in top journals/refereed conference proceedings such as IJCV, TKDE, NeurIPS, KDD, AAAI and IJCAI.  
\end{IEEEbiography}
\vspace{-4.95cm}

\vspace{-4.95cm}
\begin{IEEEbiography}[{\includegraphics[width=1in,height=1.25in,clip,keepaspectratio]{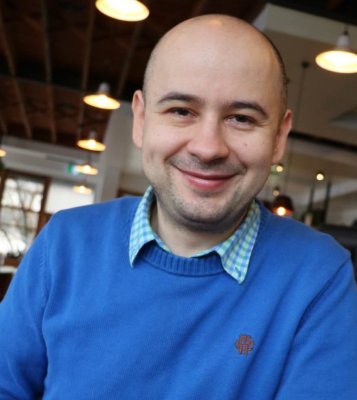}}]{Piotr Koniusz.}
A Senior Researcher in Machine Learning Research Group at Data61, CSIRO, and a Senior Honorary Lecturer at the Australian National University (ANU). He was a postdoctoral researcher in the team LEAR, INRIA, France. He received his BSc in Telecommunications and Software Engineering in 2004 from the Warsaw University of Technology, Poland, and completed his PhD in Computer Vision in 2013 at CVSSP, University of Surrey, UK. 
His current interests include contrastive and few-shot learning. He has received several awards such as the Sang Uk Lee Best Student Paper Award from ACCV'22, the Runner-up APRS/IAPR Best Student Paper Award from DICTA'22, outstanding Area Chair by ICLR 2021--2023. He  serves as a Workshop Program Chair and Senior Area Chair for NeurIPS'23.
\end{IEEEbiography}
\vspace{-4.95cm}

\vspace{-4.95cm}
\begin{IEEEbiography}[{\includegraphics[width=1in,height=1.25in,clip,keepaspectratio]{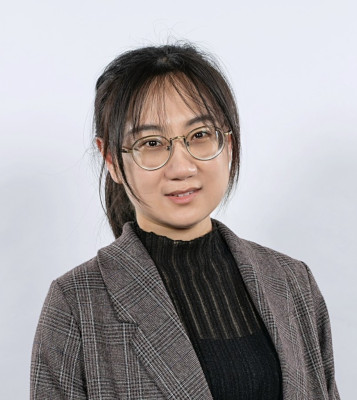}}]{Lu Zhang} received her Ph.D. degree from the Faculty of Engineering and Information Technology, University of Technology Sydney, Australia. She is currently a postdoctoral researcher in the Queensland Brain Institute \& Faculty of Engineering, Architecture, and Information Technology, University of Queensland, Australia. Her research interests include machine learning and deep learning for multimedia data analysis.
\end{IEEEbiography}
\vspace{-4.95cm}

\vspace{-4.95cm}
\begin{IEEEbiography}[{\includegraphics[width=1in,height=1.25in,clip,keepaspectratio]{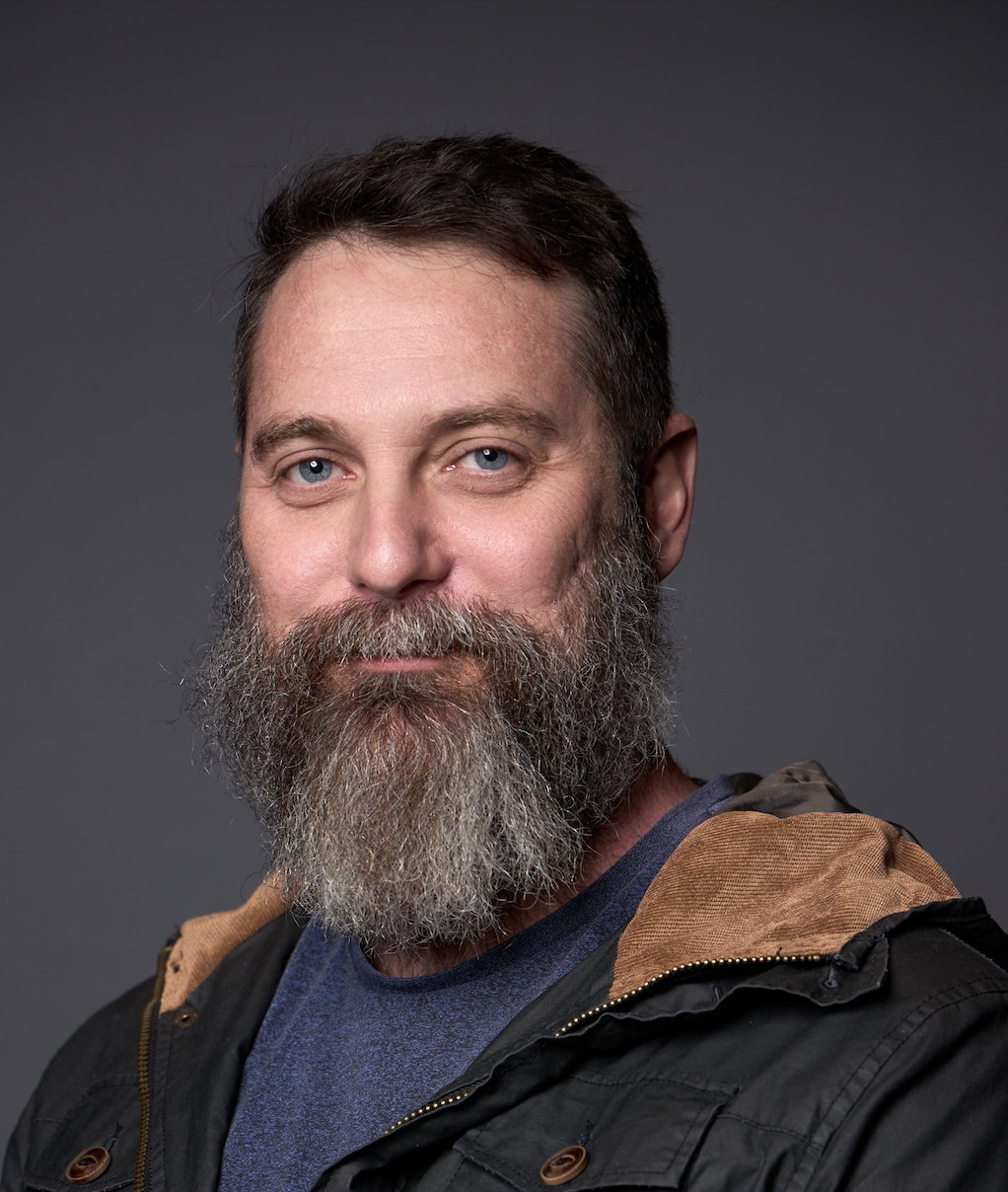}}]{Daniel Edward Pagendam} is a senior research scientist at Data61, CSIRO where he works on problems at the interface of the environmental sciences and statistics. Dan  received his PhD in Mathematics and Statistics from the University of Queensland, Australia. His main research interests are in stochastic modelling, Bayesian statistical methods and Machine Learning. 
\end{IEEEbiography}
\vspace{-4.95cm}

\vspace{-4.95cm}
\begin{IEEEbiography}[{\includegraphics[width=1in,height=1.25in,clip,keepaspectratio]{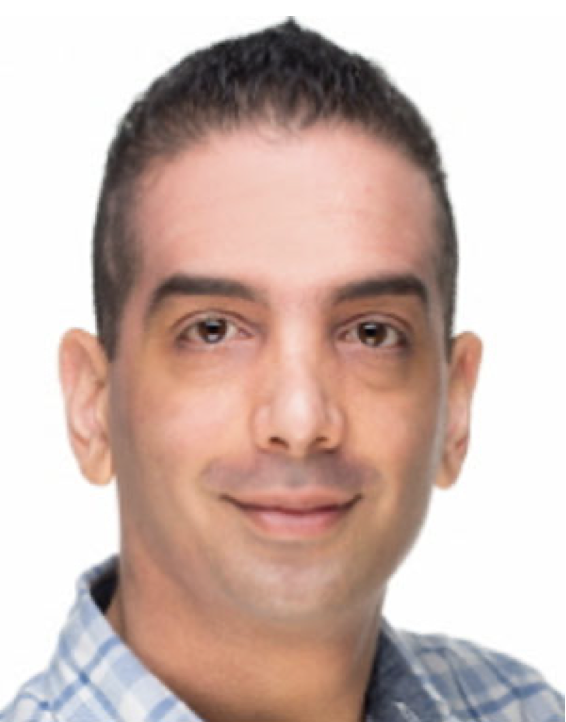}}]{Peyman Moghadam}
received his Ph.D. degree in electrical and electronic engineering from Nanyang Technological University, Singapore in 2012. He is currently a Principal Research Scientist and Cluster Leader at the Robotics and Autonomous Systems Group, Commonwealth Scientific and Industrial Research Organisation (CSIRO), Australia. He is also an Adjunct Professor at the Queensland University of Technology (QUT). His current research interests include Embodied AI, robotics and machine learning.
\end{IEEEbiography}
\vspace{-4.95cm}



\clearpage
\def\myappendix{1}

\title{Exploiting Field Dependencies for Learning on Categorical Data (Supplementary Material)}

\author{Zhibin~Li, 
        Piotr Koniusz, 
        Lu Zhang, 
        Daniel Edward Pagendam, 
        Peyman Moghadam
        \vspace{-1.0cm}
        }

\IEEEtitleabstractindextext{%
}


\maketitle

\renewcommand{\appendixname}{Appendices}
\appendix

\subsection*{A. Additional Plots on the Avazu Dataset}
{\ZB Additional plots on the Avazu dataset regarding AUC are presented in Fig.~\ref{avazuappend1} and \ref{avazuappend2}.}

\begin{figure}[!h]
\vspace{-0.3cm}
\centering
\begin{subfigure}[b]{0.5\textwidth}
  \includegraphics[trim=0 6 12 0, clip=true,width=1.0\columnwidth]{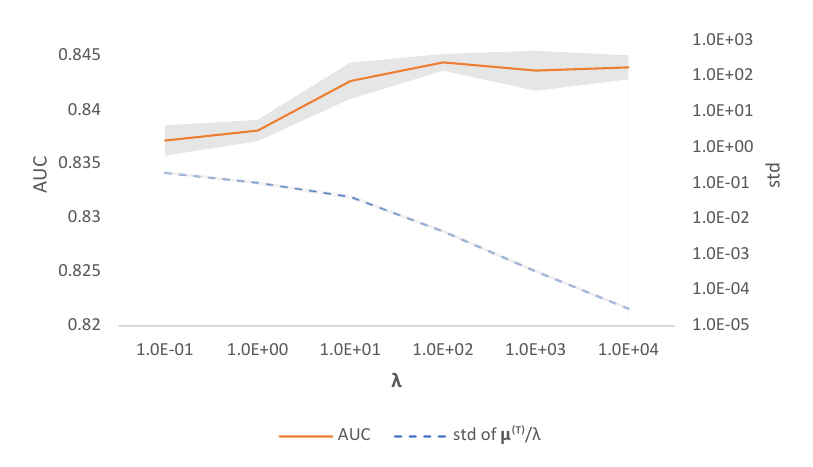}
  \caption{\label{avazu_etaabl6}}
\end{subfigure}
\begin{subfigure}[b]{0.5\textwidth}
  \includegraphics[trim=0 6 12 0, clip=true,width=1.0\columnwidth]{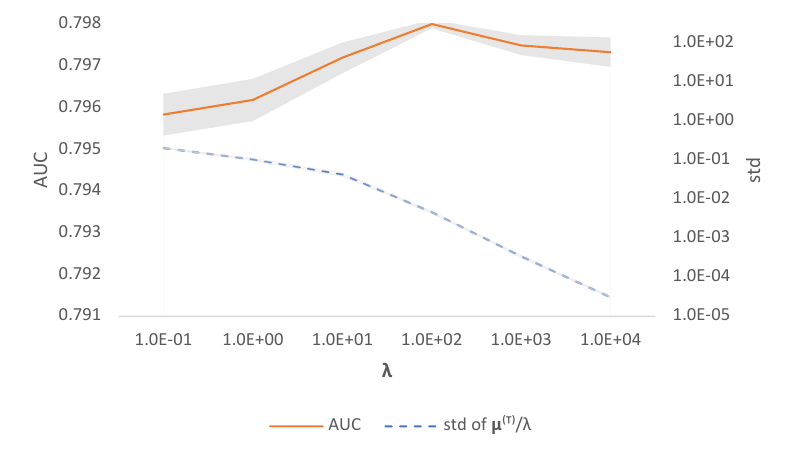}
  \caption{\label{avazu_etaabl5}}
\end{subfigure}
\begin{subfigure}[b]{0.5\textwidth}
 \includegraphics[trim=0 6 12 0, clip=true,width=1.0\columnwidth]{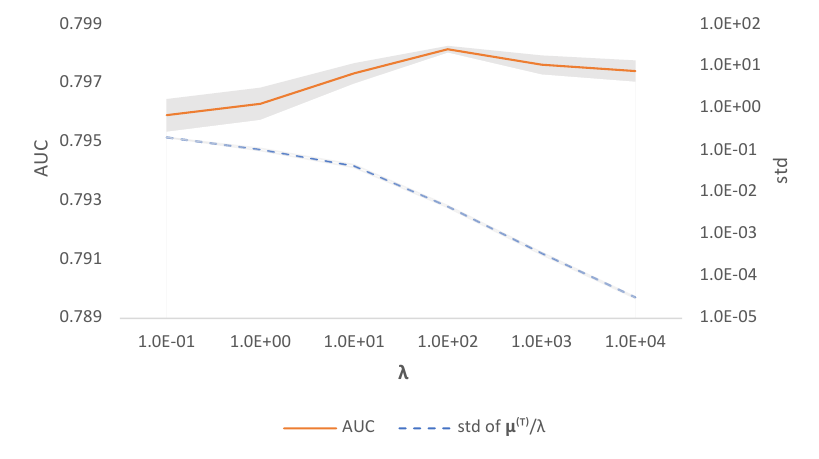}
\caption{\label{avazu_etaabl4}}
\end{subfigure}
\caption{The curves for AUC and standard deviation (std) of $\boldsymbol{\mu}^{(T)}/\lambda$ \wrt $\lambda$ on the train (Fig.~\ref{avazu_etaabl6}), validation (Fig.~\ref{avazu_etaabl5}) and test (Fig.~\ref{avazu_etaabl4}) sets of Avazu.
}\label{avazuappend1}
\vspace{-0.5cm}
\end{figure}

\begin{figure}[!b]
\centering
\begin{subfigure}[b]{0.5\textwidth}
  \includegraphics[trim=0 6 12 0, clip=true,width=1.0\columnwidth]{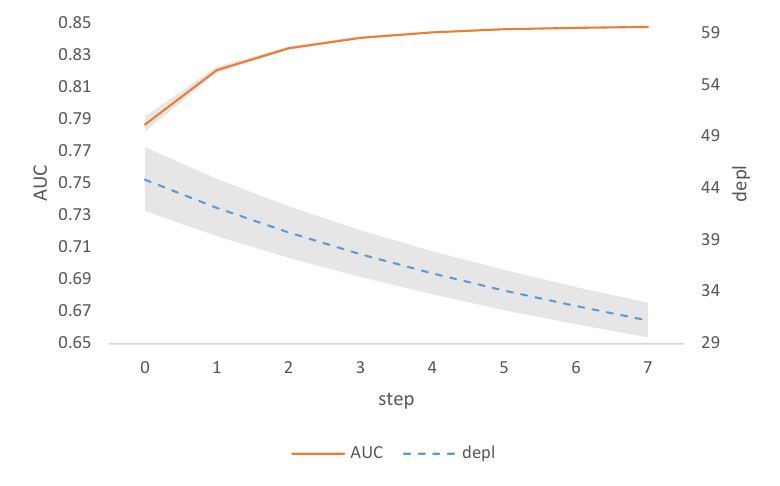}
  \caption{\label{depabl6}}
\end{subfigure}
\begin{subfigure}[b]{0.5\textwidth}
  \includegraphics[trim=0 6 12 0, clip=true,width=1.0\columnwidth]{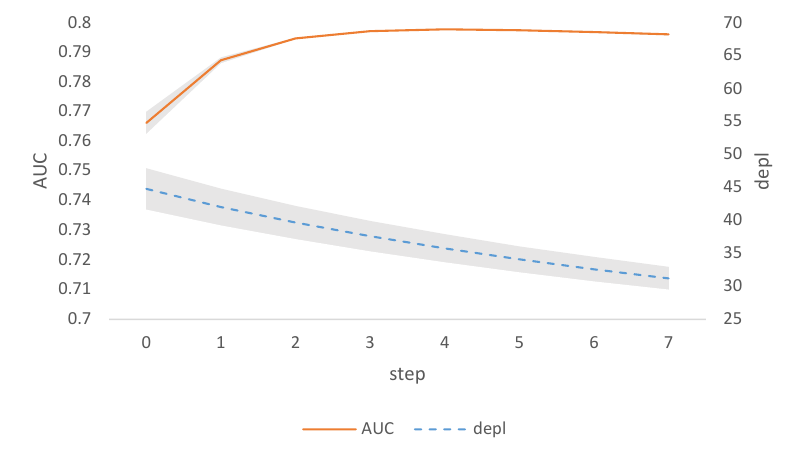}
  \caption{\label{depabl5}}
\end{subfigure}
\begin{subfigure}[b]{0.5\textwidth}
 \includegraphics[trim=0 6 12 0, clip=true,width=1.0\columnwidth]{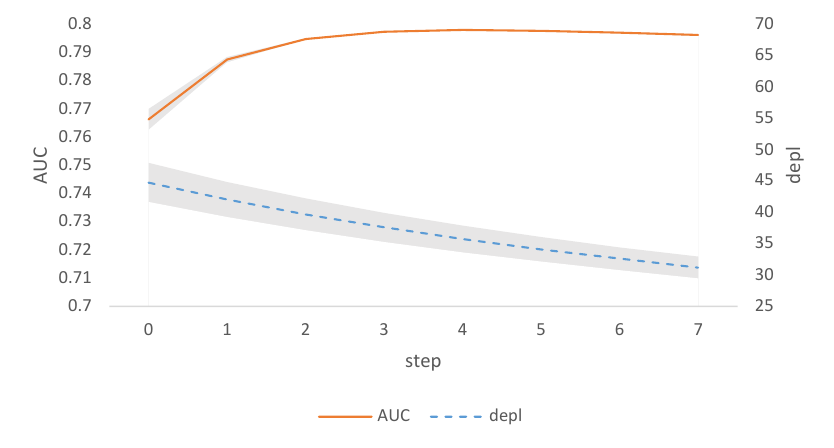}
\caption{\label{depabl4}}
\end{subfigure}
\caption{\ZB We train the MDL model on the train set of Avazu with $T=4$ and evaluate it on the train (Fig.~\ref{depabl6}), validation (Fig.~\ref{depabl5}) and test (Fig.~\ref{depabl4}) sets of Avazu with different $T$ for AUC and dependency loss ({\em depl}).
}\label{avazuappend2}
\end{figure}

\clearpage

\subsection*{B. Additional Plots on the ML-Tag Dataset}
{\ZB Additional plots on the ML-Tag dataset regarding AUC and Logloss are presented in Fig.~\ref{mltagappend1}, \ref{mltagappend2}, \ref{mltagappend3}, and \ref{mltagappend4}.}

\begin{figure}[!h]
\centering
\begin{subfigure}[b]{0.5\textwidth}
  \includegraphics[trim=16 6 14 0, clip=true,width=1.0\columnwidth]{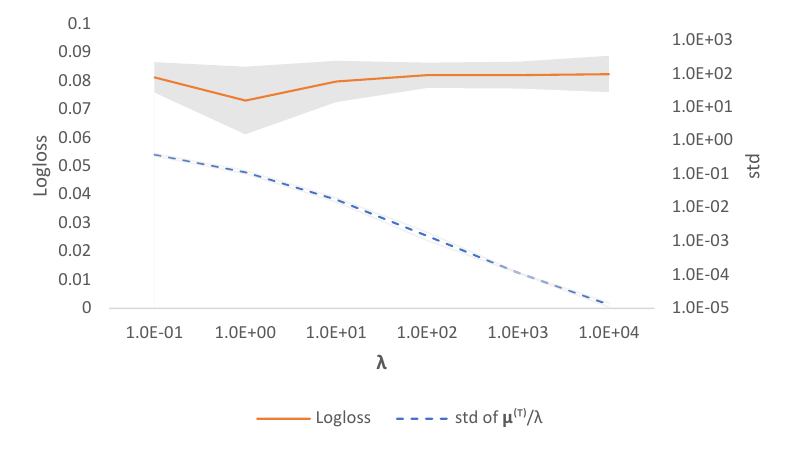}
  \caption{\label{mltag_eta_logloss_train}}
\end{subfigure}
\begin{subfigure}[b]{0.5\textwidth}
 \includegraphics[trim=16 6 14 0, clip=true,width=1.0\columnwidth]{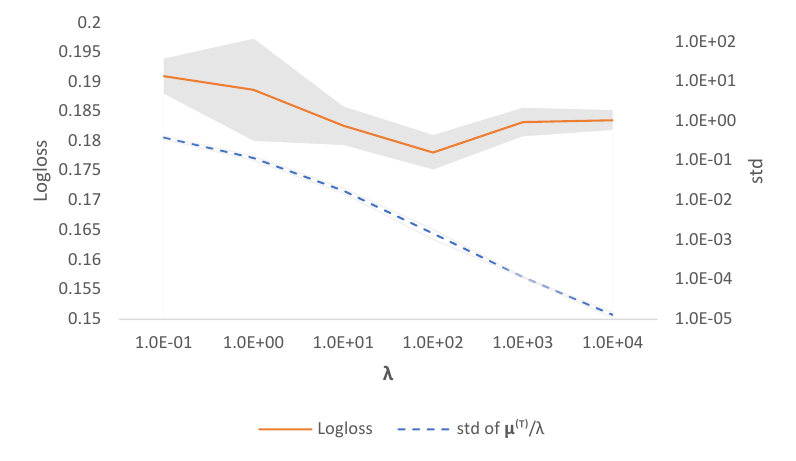}
\caption{\label{mltag_eta_logloss_val}}
\end{subfigure}
\begin{subfigure}[b]{0.5\textwidth}
  \includegraphics[trim=16 6 14 0, clip=true,width=1.0\columnwidth]{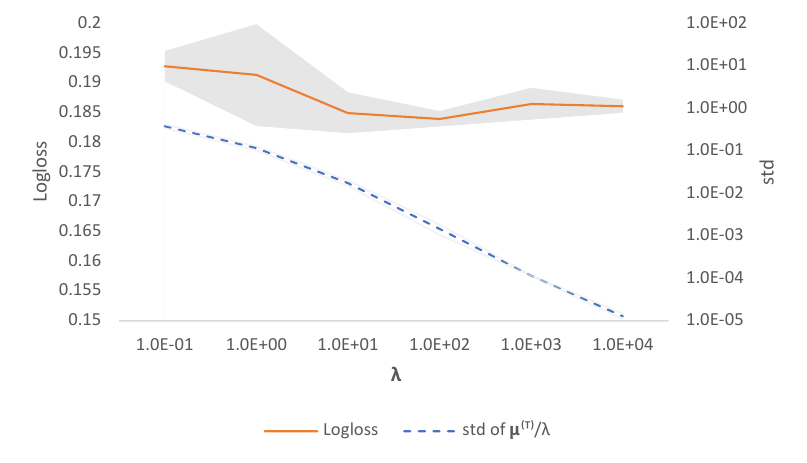}
  \caption{\label{mltag_eta_logloss_test}}
\end{subfigure}
\caption{\ZB The curves for Logloss and standard deviation (std) of $\boldsymbol{\mu}^{(T)}/\lambda$ \wrt $\lambda$ on the train (Fig.~\ref{mltag_eta_logloss_train}), val (Fig.~\ref{mltag_eta_logloss_val}), and test (Fig.~\ref{mltag_eta_logloss_test}) sets of ML-Tag.
}\label{mltagappend1}
\end{figure}

\begin{figure}[!h]
\centering
\begin{subfigure}[b]{0.5\textwidth}
  \includegraphics[trim=16 6 14 0, clip=true,width=1.0\columnwidth]{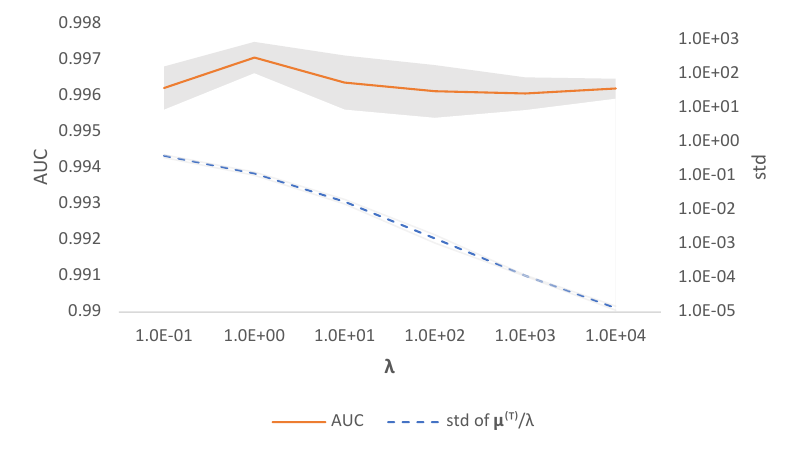}
  \caption{\label{mltag_eta_auc_train}}
\end{subfigure}\\
\begin{subfigure}[b]{0.5\textwidth}
  \includegraphics[trim=16 6 14 0, clip=true,width=1.0\columnwidth]{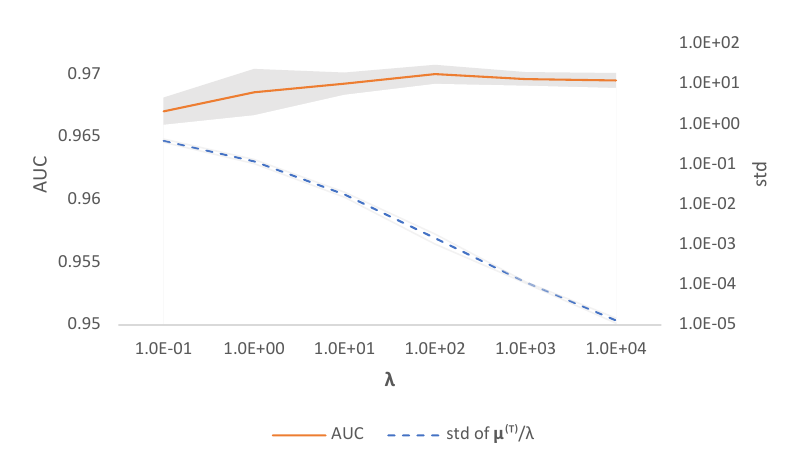}
  \caption{\label{mltag_eta_auc_val}}
\end{subfigure}
\begin{subfigure}[b]{0.5\textwidth}
 \includegraphics[trim=16 6 14 0, clip=true,width=1.0\columnwidth]{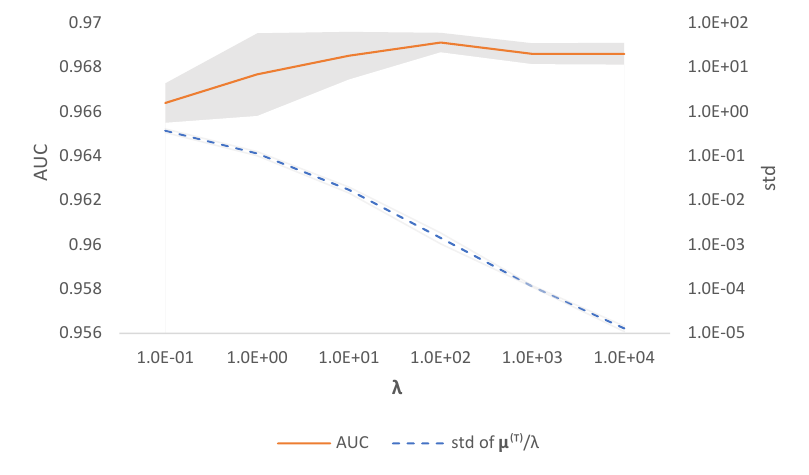}
\caption{\label{mltag_eta_auc_test}}
\end{subfigure}
\caption{\ZB The curves for AUC and standard deviation (std) of $\boldsymbol{\mu}^{(T)}/\lambda$ \wrt $\lambda$ on the train (Fig.~\ref{mltag_eta_auc_train}), validation (Fig.~\ref{mltag_eta_auc_val}) and test (Fig.~\ref{mltag_eta_auc_test}) sets of ML-Tag.
}\label{mltagappend2}
\end{figure}

\begin{figure}[!h]
\centering
\begin{subfigure}[b]{0.5\textwidth}
  \includegraphics[trim=16 6 14 0, clip=true,width=1.0\columnwidth]{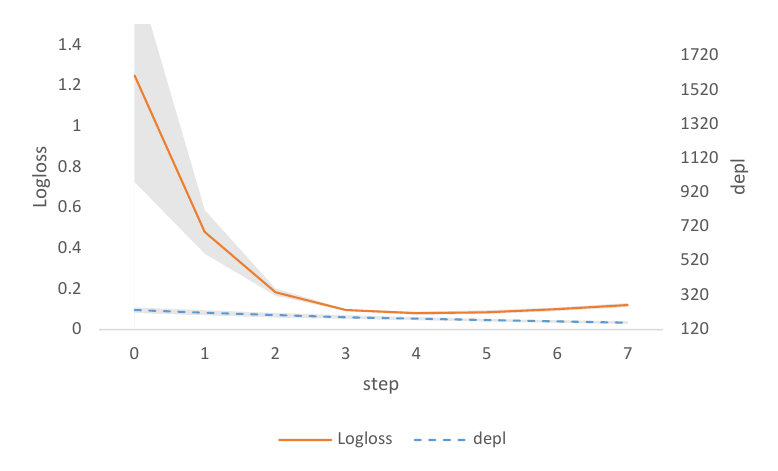}
  \caption{\label{mltag_depl_logloss_train}}
\end{subfigure}
\begin{subfigure}[b]{0.5\textwidth}
 \includegraphics[trim=16 6 14 0, clip=true,width=1.0\columnwidth]{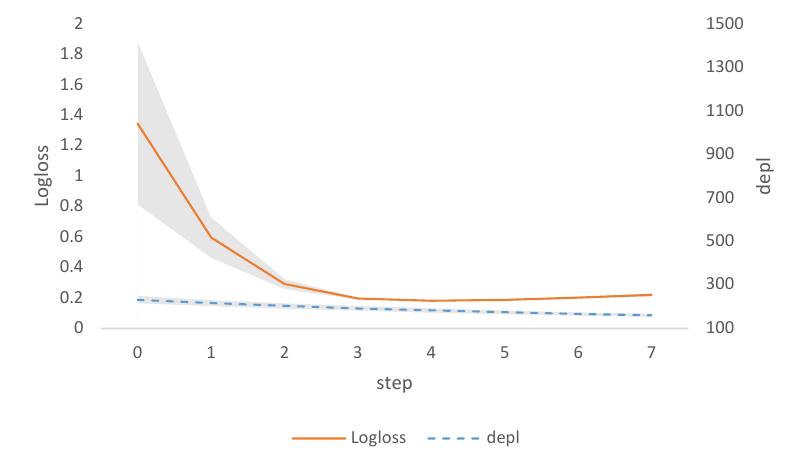}
\caption{\label{mltag_depl_logloss_val}}
\end{subfigure}
\begin{subfigure}[b]{0.5\textwidth}
  \includegraphics[trim=16 6 14 0, clip=true,width=1.0\columnwidth]{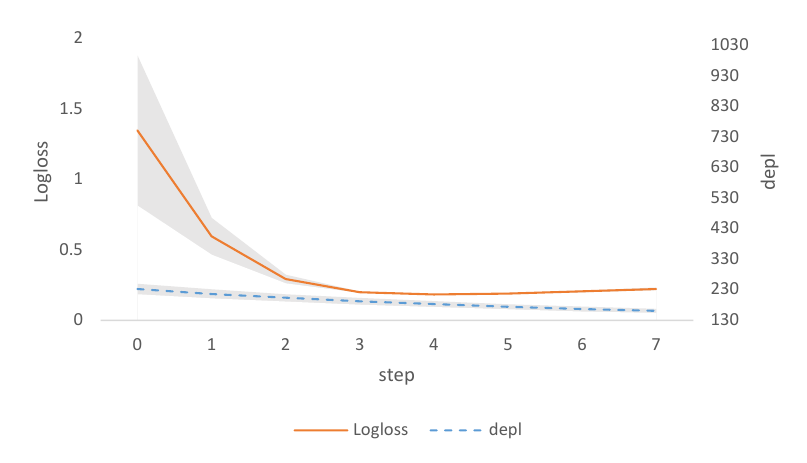}
  \caption{\label{mltag_depl_logloss_test}}
\end{subfigure}
\caption{\ZB We train the MDL model on train set of ML-Tag with $T=4$ and evaluate it on the train (Fig.~\ref{mltag_depl_logloss_train}), validation (Fig.~\ref{mltag_depl_logloss_val}) and test (Fig.~\ref{mltag_depl_logloss_test}) sets of ML-Tag with different $T$ for Logloss and dependency loss ({\em depl}).
}\label{mltagappend3}
\end{figure}

\begin{figure}[!h]
\centering
\begin{subfigure}[b]{0.5\textwidth}
  \includegraphics[trim=16 6 14 0, clip=true,width=1.0\columnwidth]{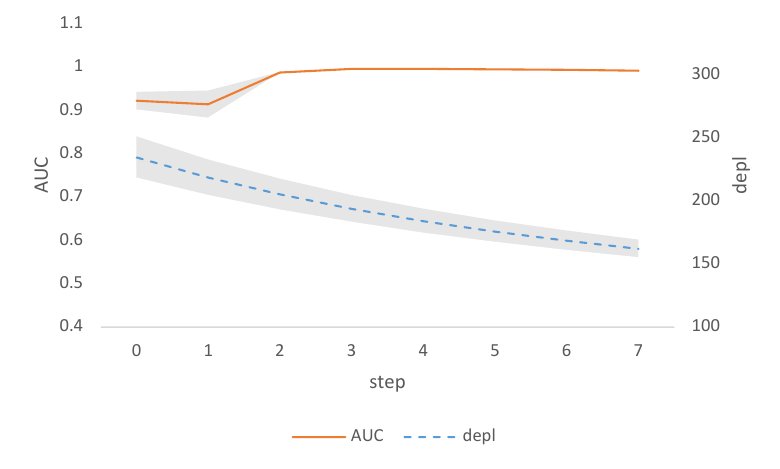}
  \caption{\label{mltag_depl_auc_train}}
\end{subfigure}
\begin{subfigure}[b]{0.5\textwidth}
  \includegraphics[trim=16 6 14 0, clip=true,width=1.0\columnwidth]{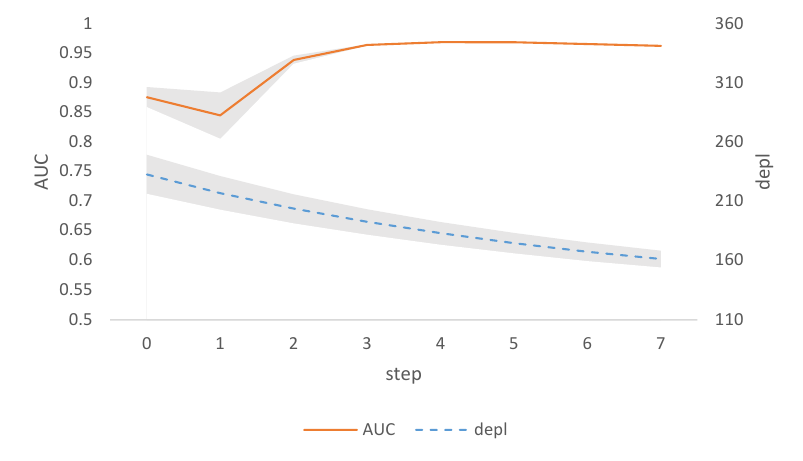}
  \caption{\label{mltag_depl_auc_val}}
\end{subfigure}
\begin{subfigure}[b]{0.5\textwidth}
 \includegraphics[trim=16 6 14 0, clip=true,width=1.0\columnwidth]{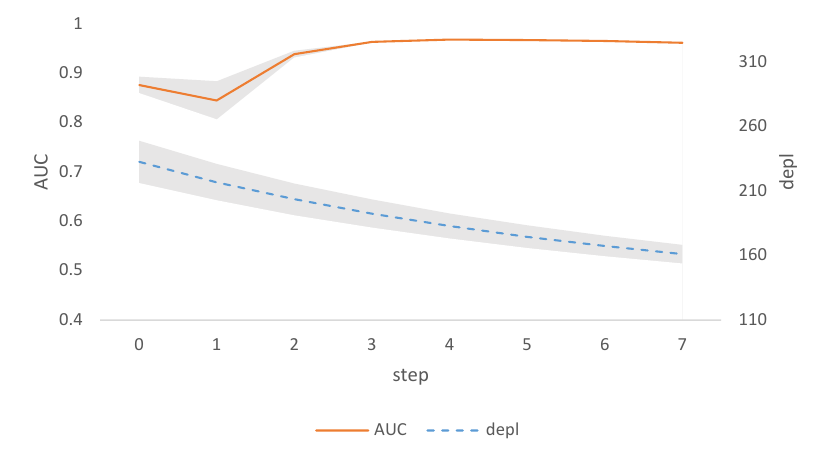}
\caption{\label{mltag_depl_auc_test}}
\end{subfigure}
\caption{\ZB We train the MDL model on train set of ML-Tag with $T=4$ and evaluate it on the train ( Fig.~\ref{mltag_depl_auc_train}), validation (Fig.~\ref{mltag_depl_auc_val}) and test (Fig.~\ref{mltag_depl_auc_test}) sets of ML-Tag with different $T$ for AUC and dependency loss ({\em depl}).
}\label{mltagappend4}
\end{figure}

\clearpage

\subsection*{C. Runtimes}
Some runtime is presented in Table~\ref{runtimes1} and \ref{runtimes2}.

\begin{table}[!h]
\centering
\small
\caption{The runtime (in seconds) of MDL per dataset.}\label{runtimes1}
\begin{tabular}{@{}cccccc@{}}
\toprule
                       KDD2012 & Avazu & Criteo & ML-Tag & Frappe & ML-Rating \\ \midrule
 18557   & 5809  & 21713  & 120     & 20     & 112       \\ \bottomrule
\end{tabular}
\end{table}

{\ZB 
\begin{table}[!h]
\caption{The runtimes on Avazu for MDL and all other methods used in our comparisons.}\label{runtimes2}
\begin{tabular}{@{}ll@{}}
\toprule
Method & Runtime \\ \midrule
LR & 1h36m \\
FM & 1h8m \\
FFM & 7h20m \\
xDFM & 1h \\
AFN & 50m \\
IFM & 40m \\
INN & 1h53m \\
LLFM & 30h45m \\
RaFM & 20h50m \\
FWL & 1h31m \\
W\&D & 2h36m \\
DeepFM & 3h18m \\
NODE & 3h50m \\
TabNet & 3h18m \\
TabTr & 20h55m \\
SAINT & 2h36m \\
FMLP & 26m \\
MDL (ours) & 1h36m \\ \bottomrule
\end{tabular}
\end{table}
}

\subsection*{D. Ablations on Hyper-Parameters}
{\ZB The impact of $T$ and the number of MLP layers on results of validation set of Avazu are presented in Table~\ref{hyperptable}.}
\begin{table}
\centering
\setlength{\tabcolsep}{2pt}
\small
\caption{\ZB Ablations on hyper-parameters $T$ and the number of MLP layers on the validation set of Avazu.}\label{hyperptable}
\label{}
\begin{tabular}{@{}lcccccccc@{}}
\toprule
 & \multicolumn{4}{c}{T (step)} &  & \multicolumn{3}{c}{Number of Layers} \\ \cmidrule(lr){2-5} \cmidrule(l){7-9} 
 & 2 & 4 & 6 & 8 &  & 2 & 3 & 4 \\ \midrule
Logloss & 0.3692 & \cellcolor{LightCyan}\textbf{0.3690} & 0.3691 & 0.3692 &  & 0.3697 & \cellcolor{LightCyan}\textbf{0.3690} & 0.3695 \\
AUC & 0.7975 & \cellcolor{LightCyan}\textbf{0.7980} & 0.7978 & 0.7976 &  & 0.7970 & \cellcolor{LightCyan}\textbf{0.7980} & 0.7971 \\ \bottomrule
\end{tabular}
\end{table}

\subsection*{E. Standard Deviations of Results}
{\ZB Standard deviations of experimental results are presented in Tables~\ref{stdofAUC}, \ref{stdoflogloss}, and \ref{stdofmlrating}.}

\begin{table}[ht]
\setlength{\tabcolsep}{2pt}
\caption{\ZB Standard deviations of AUC calculated for replicated experiments across different datasets and methods.}\label{stdofAUC}
\centering
\begin{tabular}{rrrrrr}
  \hline
 & Avazu & Criteo & Frappe & KDD2012 & ML-Tag \\ 
  \hline
LR & 1.913E-04 & 6.116E-05 & 3.624E-04 & 2.691E-04 & 2.691E-04 \\ 
  FM & 3.095E-04 & 9.937E-05 & 3.218E-04 & 1.740E-04 & 1.740E-04 \\ 
  FFM & 4.092E-04 & 6.282E-05 & 5.128E-04 & 1.995E-04 & 1.995E-04 \\ 
  xDFM & 2.138E-04 & 2.992E-05 & 8.672E-04 & 3.879E-04 & 3.879E-04 \\ 
  AFN & 2.081E-04 & 8.659E-05 & 7.002E-04 & 4.718E-04 & 4.718E-04 \\ 
  IFM & 3.258E-04 & 3.288E-05 & 6.828E-04 & 2.106E-04 & 2.106E-04 \\ 
  INN & 1.682E-04 & 1.151E-04 & 2.716E-04 & 1.829E-04 & 1.829E-04 \\ 
  LLFM & 2.361E-04 & 9.480E-05 & 1.210E-03 & 2.369E-04 & 2.369E-04 \\ 
  RaFM & 1.730E-04 & 9.809E-05 & 8.760E-04 & 2.758E-04 & 2.758E-04 \\ 
  FWL & 1.849E-04 & 1.038E-04 & 4.582E-04 & 1.868E-04 & 1.868E-04 \\ 
  WD & 4.823E-04 & 1.683E-04 & - & 5.544E-04 & 4.066E-04 \\ 
  DeepFM & 4.508E-04 & 3.692E-04 & - & 6.735E-04 & 4.119E-04 \\ 
  NODE & 1.874E-04 & 3.383E-04 & - & 3.005E-04 & 1.608E-04 \\ 
  Tabnet & 2.542E-04 & 1.855E-03 & - & 3.202E-03 & 2.514E-03 \\ 
  TabTran & 5.326E-04 & 1.001E-04 & - & 1.768E-03 & 1.193E-03 \\ 
  SAINT & 2.235E-03 & 6.611E-04 & - & 1.258E-03 & 9.566E-04 \\ 
  FMLP & 2.156E-04 & 8.937E-05 & - & 3.690E-04 & 7.389E-04 \\ 
  Ours & 1.190E-04 & 8.275E-05 & 7.967E-04 & 2.773E-04 & 5.941E-04 \\ 
   \hline
\end{tabular}
\end{table}

\begin{table}[ht]
\setlength{\tabcolsep}{2pt}
\caption{\ZB Standard deviations of Logloss calculated for replicated experiments across different datasets and methods.}\label{stdoflogloss}
\begin{tabular}{rrrrrr}
  \hline
 Method & Avazu & Criteo & Frappe & KDD2012 & ML-Tag \\ 
  \hline
LR & 1.659E-04 & 6.392E-05 & 6.114E-03 & 3.772E-05 & 4.665E-04 \\ 
  FM & 1.674E-04 & 5.342E-05 & 2.421E-3 & 4.852E-05 & 4.134E-04 \\ 
  FFM & 1.588E-04 & 4.451E-05 & 4.289E-03 & 2.583E-05 & 6.233E-04 \\ 
  xDFM & 1.049E-04 & 4.173E-05 & 2.107E-03 & 5.635E-05 & 5.694E-04 \\ 
  AFN & 1.463E-04 & 4.826E-05 & 2.963E-3 & 5.886E-05 & 6.782E-04 \\ 
  IFM & 9.781E-05 & 4.639E-05 & 5.814E-03 & 6.055E-05 & 3.727E-04 \\ 
  INN & 1.446E-04 & 2.875E-05 & 3.129E-03 & 1.367E-05 & 5.531E-04 \\ 
  LLFM & 1.471E-04 & 4.703E-05 & 6.832E-03 & 6.135E-05 & 4.734E-04 \\ 
  RaFM & 1.871E-04 & 5.116E-05 & 2.645E-03 & 2.871E-05 & 5.465E-04 \\ 
  FWL & 1.498E-04 & 3.429E-05 & 1.874E-03 & 4.248E-05 & 2.129E-04 \\ 
  WD & 3.362E-04 & 1.158E-04 & - & 1.281E-04 & 1.274E-03 \\ 
  DeepFM & 2.947E-04 & 3.143E-04 & - & 1.402E-04 & 8.384E-04 \\ 
  NODE & 2.029E-04 & 6.613E-04 & - & 2.339E-04 & 4.172E-04 \\ 
  Tabnet & 1.624E-04 & 5.815E-03 & - & 3.615E-04 & 3.719E-03 \\ 
  TabTran & 3.785E-04 & 7.982E-05 & - & 3.637E-04 & 3.180E-03 \\ 
  SAINT & 8.591E-04 & 6.181E-04 & - & 2.000E-04 & 1.440E-03 \\ 
  FMLP & 9.540E-05 & 1.720E-04 & - & 7.310E-05 & 1.804E-03 \\ 
  Ours & 8.830E-05 & 9.370E-05 & 1.944E-03 & 6.448E-05 & 1.622E-03 \\ 
     \hline
\end{tabular}
\end{table}

\begin{table}[ht]
\caption{\ZB Standard deviations of MSE calculated for replicated experiments across different methods for the {\tt ML-Rating} dataset.}\label{stdofmlrating}
\begin{tabular}{@{}ll@{}}
\toprule
Method & Standard deviation \\ \midrule
LR     & 2.118E-03          \\
FM     & 1.637E-03          \\
FFM    & 1.436E-03          \\
xDFM   & 8.910E-04          \\
AFN    & 1.135E-03          \\
IFM    & 1.683E-03          \\
INN    & 1.062E-03          \\
LLFM   & 1.332E-03          \\
RaFM   & 2.493E-03          \\
FWL    & 1.624E-03          \\
Ours   & 2.112E-03          \\ \bottomrule
\end{tabular}
\end{table}

\subsection*{F. Statistical Test for Results in Tables 3-5}
{\ZB  We conduct ANOVA (Dunnett's test) on each group of results to see if results with our model are significantly batter than the baselines. This is better than conducting simple t-test where results of each baseline and our model are compared pair-wise and independently. Dunnett's test is a statistical method for controlling the Type I error rate, when making multiple hypothesis tests that involve comparing a variety of methods to an experimental standard or new treatment.  In this work, we compare many existing methods to our method (comparisons called ``planned contrasts'').  Using Dunnett's test provides a conservative test of significance to ensure that the probability of falsely rejecting the null hypothesis (i.e. detecting that our method is better than other methods, when it is not in reality) does not grow with the number of hypotheses tested.  Since we are interested in testing the hypothesis that our method is better than alternatives, we use one-sided statistical tests.

We used ANOVA undertaken in R and using the package {\tt multcomp}~\citelatex{hothorn2008simultaneous_supp} to perform multiple comparisons.
In the following tables, The column ``Estimate" shows the estimated difference between the mean Logloss/AUC for each baseline and our method; ``Std. Error" is the standard deviation of the estimate; ``t value" is the t-statistic for Dunnett's test; and Pr($<$t) (Pr($>$t)) is the probability of obtaining a t-statistic less (greater) than the observed t-statistic.}

\begin{table}[ht]
\setlength{\tabcolsep}{2pt}
\caption{Dunnett's test (one-sided) \wrt Table 3, AUC, {\tt Avazu} dataset.}
\centering
\begin{tabular}{rrrrr}
  \hline
 & Estimate & Std. Error & t value & Pr($<$t) \\ 
  \hline
LR minus Ours & -2.183E-02 & 1.546E-04 & -1.412E+02 & 0.000E+00 \\ 
  FM minus Ours & -1.269E-02 & 1.546E-04 & -8.209E+01 & 0.000E+00 \\ 
  FFM minus Ours & -6.603E-03 & 1.546E-04 & -4.271E+01 & 0.000E+00 \\ 
  xDFM minus Ours & -6.631E-03 & 1.546E-04 & -4.289E+01 & 0.000E+00 \\ 
  AFN minus Ours & -9.455E-03 & 1.546E-04 & -6.116E+01 & 0.000E+00 \\ 
  IFM minus Ours & -7.619E-03 & 1.546E-04 & -4.928E+01 & 0.000E+00 \\ 
  INN minus Ours & -2.203E-02 & 1.546E-04 & -1.425E+02 & 0.000E+00 \\ 
  LLFM minus Ours & -1.186E-02 & 1.546E-04 & -7.674E+01 & 0.000E+00 \\ 
  RaFM minus Ours & -7.790E-03 & 1.546E-04 & -5.039E+01 & 0.000E+00 \\ 
  FWL minus Ours & -3.496E-03 & 1.546E-04 & -2.261E+01 & 0.000E+00 \\ 
   \hline
\end{tabular}
\end{table}

\begin{table}[ht]
\setlength{\tabcolsep}{2pt}
\caption{Dunnett's test (one-sided) \wrt Table 5, AUC, {\tt Avazu} dataset.}
\centering
\begin{tabular}{rrrrr}
  \hline
 & Estimate & Std. Error & t value & Pr($<$t) \\ 
  \hline
WD minus Ours & -9.373E-03 & 5.419E-04 & -1.730E+01 & 0.000E+00 \\ 
  DeepFM minus Ours & -8.943E-03 & 5.419E-04 & -1.650E+01 & 0.000E+00 \\ 
  NODE minus Ours & -3.302E-02 & 5.419E-04 & -6.093E+01 & 0.000E+00 \\ 
  Tabnet minus Ours & -1.468E-02 & 5.419E-04 & -2.708E+01 & 0.000E+00 \\ 
  TabTran minus Ours & -1.309E-02 & 5.419E-04 & -2.416E+01 & 0.000E+00 \\ 
  SAINT minus Ours & -2.049E-02 & 5.419E-04 & -3.782E+01 & 0.000E+00 \\ 
  FMLP minus Ours & -8.048E-03 & 5.419E-04 & -1.485E+01 & 0.000E+00 \\ 
   \hline
\end{tabular}
\end{table}

\begin{table}[ht]
\setlength{\tabcolsep}{2pt}
\caption{Dunnett's test (one-sided) \wrt Table 3, AUC, {\tt Criteo} dataset.}
\centering
\begin{tabular}{rrrrr}
  \hline
 & Estimate & Std. Error & t value & Pr($<$t) \\ 
  \hline
LR minus Ours & -2.037E-02 & 5.275E-05 & -3.861E+02 & 0.000E+00 \\ 
  FM minus Ours & -6.306E-03 & 5.275E-05 & -1.195E+02 & 0.000E+00 \\ 
  FFM minus Ours & -3.768E-03 & 5.275E-05 & -7.142E+01 & 0.000E+00 \\ 
  xDFM minus Ours & -3.263E-03 & 5.275E-05 & -6.186E+01 & 0.000E+00 \\ 
  AFN minus Ours & -2.002E-03 & 5.275E-05 & -3.795E+01 & 0.000E+00 \\ 
  IFM minus Ours & -3.070E-03 & 5.275E-05 & -5.819E+01 & 0.000E+00 \\ 
  INN minus Ours & -7.782E-03 & 5.275E-05 & -1.475E+02 & 0.000E+00 \\ 
  LLFM minus Ours & -5.057E-03 & 5.275E-05 & -9.587E+01 & 0.000E+00 \\ 
  RaFM minus Ours & -4.004E-03 & 5.275E-05 & -7.589E+01 & 0.000E+00 \\ 
  FWL minus Ours & -1.525E-03 & 5.275E-05 & -2.891E+01 & 0.000E+00 \\ 
   \hline
\end{tabular}
\end{table}

\begin{table}[ht]
\setlength{\tabcolsep}{2pt}
\caption{Dunnett's test (one-sided) \wrt Table 5, AUC, {\tt Criteo} dataset.}
\centering
\begin{tabular}{rrrrr}
  \hline
 & Estimate & Std. Error & t value & Pr($<$t) \\ 
  \hline
WD minus Ours & -3.914E-03 & 4.573E-04 & -8.560E+00 & 1.132E-08 \\ 
  DeepFM minus Ours & -5.277E-03 & 4.573E-04 & -1.154E+01 & 2.076E-14 \\ 
  NODE minus Ours & -3.486E-02 & 4.573E-04 & -7.624E+01 & 0.000E+00 \\ 
  Tabnet minus Ours & -1.502E-02 & 4.573E-04 & -3.285E+01 & 0.000E+00 \\ 
  TabTran minus Ours & -6.646E-03 & 4.573E-04 & -1.454E+01 & 0.000E+00 \\ 
  SAINT minus Ours & -7.954E-03 & 4.573E-04 & -1.740E+01 & 0.000E+00 \\ 
  FMLP minus Ours & -8.065E-03 & 4.573E-04 & -1.764E+01 & 0.000E+00 \\ 
   \hline
\end{tabular}
\end{table}

\begin{table}[ht]
\setlength{\tabcolsep}{2pt}
\caption{Dunnett's test (one-sided) \wrt Table 3, AUC, {\tt KDD2012} dataset.}
\centering
\begin{tabular}{rrrrr}
  \hline
 & Estimate & Std. Error & t value & Pr($<$t) \\ 
  \hline
LR minus Ours & -2.809E-02 & 1.745E-04 & -1.609E+02 & 0.000E+00 \\ 
  FM minus Ours & -1.343E-02 & 1.745E-04 & -7.699E+01 & 0.000E+00 \\ 
  FFM minus Ours & -1.466E-02 & 1.745E-04 & -8.404E+01 & 0.000E+00 \\ 
  xDFM minus Ours & -9.294E-03 & 1.745E-04 & -5.326E+01 & 0.000E+00 \\ 
  AFN minus Ours & -7.489E-03 & 1.745E-04 & -4.291E+01 & 0.000E+00 \\ 
  IFM minus Ours & -3.051E-03 & 1.745E-04 & -1.748E+01 & 0.000E+00 \\ 
  INN minus Ours & -2.216E-02 & 1.745E-04 & -1.270E+02 & 0.000E+00 \\ 
  LLFM minus Ours & -1.262E-02 & 1.745E-04 & -7.233E+01 & 0.000E+00 \\ 
  RaFM minus Ours & -2.020E-02 & 1.745E-04 & -1.158E+02 & 0.000E+00 \\ 
  FWL minus Ours & -7.082E-03 & 1.745E-04 & -4.058E+01 & 0.000E+00 \\
   \hline
\end{tabular}
\end{table}

\begin{table}[ht]
\setlength{\tabcolsep}{2pt}
\caption{Dunnett's test (one-sided) \wrt Table 5, AUC, {\tt KDD2012} dataset.}
\centering
\begin{tabular}{rrrrr}
  \hline
 & Estimate & Std. Error & t value & Pr($<$t) \\ 
  \hline
WD minus Ours & -7.221E-03 & 9.769E-04 & -7.392E+00 & 3.217E-06 \\ 
  DeepFM minus Ours & -8.086E-03 & 9.769E-04 & -8.277E+00 & 3.443E-07 \\ 
  NODE minus Ours & -3.087E-02 & 9.769E-04 & -3.160E+01 & 0.000E+00 \\ 
  Tabnet minus Ours & -5.262E-02 & 9.769E-04 & -5.387E+01 & 0.000E+00 \\ 
  TabTran minus Ours & -3.127E-02 & 9.769E-04 & -3.201E+01 & 0.000E+00 \\ 
  SAINT minus Ours & -2.024E-02 & 9.769E-04 & -2.072E+01 & 0.000E+00 \\ 
  FMLP minus Ours & -7.713E-03 & 9.769E-04 & -7.896E+00 & 3.354E-07 \\ 
   \hline
\end{tabular}
\end{table}

\begin{table}[ht]
\setlength{\tabcolsep}{2pt}
\caption{Dunnett's test (one-sided) \wrt Table 4, AUC, {\tt ML-Tag} dataset. }
\centering
\begin{tabular}{rrrrr}
  \hline
 & Estimate & Std. Error & t value & Pr($<$t) \\ 
  \hline
LR minus Ours & -4.026E-02 & 2.346E-04 & -1.716E+02 & 0.000E+00 \\ 
  FM minus Ours & -8.514E-03 & 2.346E-04 & -3.629E+01 & 0.000E+00 \\ 
  FFM minus Ours & -4.911E-03 & 2.346E-04 & -2.093E+01 & 0.000E+00 \\ 
  xDFM minus Ours & -5.850E-03 & 2.346E-04 & -2.493E+01 & 0.000E+00 \\ 
  AFN minus Ours & -9.904E-03 & 2.346E-04 & -4.221E+01 & 0.000E+00 \\ 
  IFM minus Ours & -1.970E-02 & 2.346E-04 & -8.395E+01 & 0.000E+00 \\ 
  INN minus Ours & -1.707E-02 & 2.346E-04 & -7.274E+01 & 0.000E+00 \\ 
  LLFM minus Ours & -1.239E-02 & 2.346E-04 & -5.282E+01 & 0.000E+00 \\ 
  RaFM minus Ours & -1.023E-02 & 2.346E-04 & -4.362E+01 & 0.000E+00 \\ 
  FWL minus Ours & -6.383E-03 & 2.346E-04 & -2.721E+01 & 0.000E+00 \\ 
   \hline
\end{tabular}
\end{table}

\begin{table}[ht]
\setlength{\tabcolsep}{2pt}
\caption{Dunnett's test (one-sided) \wrt Table 5, AUC, {\tt ML-Tag} dataset. }
\centering
\begin{tabular}{rrrrr}
  \hline
 & Estimate & Std. Error & t value & Pr($<$t) \\ 
  \hline
WD minus Ours & -1.496E-02 & 7.043E-04 & -2.124E+01 & 0.000E+00 \\ 
  DeepFM minus Ours & -8.884E-03 & 7.043E-04 & -1.261E+01 & 9.992E-16 \\ 
  NODE minus Ours & -2.681E-02 & 7.043E-04 & -3.806E+01 & 0.000E+00 \\ 
  Tabnet minus Ours & -3.177E-02 & 7.043E-04 & -4.510E+01 & 0.000E+00 \\ 
  TabTran minus Ours & -2.176E-02 & 7.043E-04 & -3.089E+01 & 0.000E+00 \\ 
  SAINT minus Ours & -3.244E-02 & 7.043E-04 & -4.606E+01 & 0.000E+00 \\ 
  FMLP minus Ours & -5.762E-03 & 7.043E-04 & -8.182E+00 & 5.051E-09 \\
   \hline
\end{tabular}
\end{table}

\begin{table}[ht]
\setlength{\tabcolsep}{2pt}
\caption{Dunnett's test (one-sided) \wrt Table 4, AUC, {\tt Frappe} dataset.}
\centering
\begin{tabular}{rrrrr}
  \hline
 & Estimate & Std. Error & t value & Pr($<$t) \\ 
  \hline
LR minus Ours & -5.417E-02 & 4.413E-04 & -1.228E+02 & 0.000E+00 \\ 
  FM minus Ours & -2.632E-03 & 4.413E-04 & -5.964E+00 & 1.322E-06 \\ 
  FFM minus Ours & -2.869E-03 & 4.413E-04 & -6.502E+00 & 2.042E-07 \\ 
  xDFM minus Ours & -2.707E-03 & 4.413E-04 & -6.134E+00 & 1.633E-06 \\ 
  AFN minus Ours & -1.334E-03 & 4.413E-04 & -3.024E+00 & 1.557E-02 \\ 
  IFM minus Ours & -2.333E-03 & 4.413E-04 & -5.287E+00 & 1.620E-05 \\ 
  INN minus Ours & -1.160E-02 & 4.413E-04 & -2.628E+01 & 0.000E+00 \\ 
  LLFM minus Ours & -3.413E-03 & 4.413E-04 & -7.735E+00 & 9.542E-10 \\ 
  RaFM minus Ours & -1.174E-03 & 4.413E-04 & -2.660E+00 & 3.747E-02 \\ 
  FWL minus Ours & -2.098E-03 & 4.413E-04 & -4.754E+00 & 1.036E-04 \\ 
   \hline
\end{tabular}
\end{table}


\begin{table}[ht]
\setlength{\tabcolsep}{2pt}
\caption{Dunnett's test (one-sided) \wrt Table 3, Logloss, {\tt Avazu} dataset.}
\centering
\begin{tabular}{rrrrr}
  \hline
 & Estimate & Std. Error & t value & Pr($>$t) \\ 
  \hline
LR minus Ours & 1.326E-02 & 9.156E-05 & 1.448E+02 & 0.000E+00 \\ 
  FM minus Ours & 8.364E-03 & 9.156E-05 & 9.135E+01 & 0.000E+00 \\ 
  FFM minus Ours & 5.137E-03 & 9.156E-05 & 5.610E+01 & 0.000E+00 \\ 
  xDFM minus Ours & 4.105E-03 & 9.156E-05 & 4.483E+01 & 0.000E+00 \\ 
  AFN minus Ours & 6.793E-03 & 9.156E-05 & 7.419E+01 & 0.000E+00 \\ 
  IFM minus Ours & 4.698E-03 & 9.156E-05 & 5.131E+01 & 0.000E+00 \\ 
  INN minus Ours & 1.304E-02 & 9.156E-05 & 1.425E+02 & 0.000E+00 \\ 
  LLFM minus Ours & 8.141E-03 & 9.156E-05 & 8.891E+01 & 0.000E+00 \\ 
  RaFM minus Ours & 5.512E-03 & 9.156E-05 & 6.020E+01 & 0.000E+00 \\ 
  FWL minus Ours & 2.775E-03 & 9.156E-05 & 3.030E+01 & 0.000E+00 \\ 
   \hline
\end{tabular}
\end{table}

\begin{table}[ht]
\setlength{\tabcolsep}{2pt}
\caption{Dunnett's test (one-sided) \wrt Table 5, Logloss, {\tt Avazu} dataset.}
\centering
\begin{tabular}{rrrrr}
  \hline
 & Estimate & Std. Error & t value & Pr($>$t) \\ 
  \hline
WD minus Ours & 6.061E-03 & 2.414E-04 & 2.511E+01 & 0.000E+00 \\ 
  DeepFM minus Ours & 6.062E-03 & 2.414E-04 & 2.511E+01 & 0.000E+00 \\ 
  NODE minus Ours & 1.979E-02 & 2.414E-04 & 8.197E+01 & 0.000E+00 \\ 
  Tabnet minus Ours & 9.386E-03 & 2.414E-04 & 3.888E+01 & 0.000E+00 \\ 
  TabTran minus Ours & 8.463E-03 & 2.414E-04 & 3.506E+01 & 0.000E+00 \\ 
  SAINT minus Ours & 1.311E-02 & 2.414E-04 & 5.429E+01 & 0.000E+00 \\ 
  FMLP minus Ours & 5.051E-03 & 2.414E-04 & 2.092E+01 & 0.000E+00 \\ 
   \hline
\end{tabular}
\end{table}

\begin{table}[ht]
\setlength{\tabcolsep}{2pt}
\caption{Dunnett's test (one-sided) \wrt Table 3, Logloss, {\tt Criteo} dataset.}
\centering
\begin{tabular}{rrrrr}
  \hline
 & Estimate & Std. Error & t value & Pr($>$t) \\ 
  \hline
LR minus Ours & 1.856E-02 & 3.344E-05 & 5.550E+02 & 0.000E+00 \\ 
  FM minus Ours & 4.482E-03 & 3.344E-05 & 1.340E+02 & 0.000E+00 \\ 
  FFM minus Ours & 3.770E-03 & 3.344E-05 & 1.127E+02 & 0.000E+00 \\ 
  xDFM minus Ours & 3.047E-03 & 3.344E-05 & 9.111E+01 & 0.000E+00 \\ 
  AFN minus Ours & 2.010E-03 & 3.344E-05 & 6.009E+01 & 0.000E+00 \\ 
  IFM minus Ours & 2.800E-03 & 3.344E-05 & 8.374E+01 & 0.000E+00 \\ 
  INN minus Ours & 7.232E-03 & 3.344E-05 & 2.162E+02 & 0.000E+00 \\ 
  LLFM minus Ours & 5.002E-03 & 3.344E-05 & 1.496E+02 & 0.000E+00 \\ 
  RaFM minus Ours & 4.002E-03 & 3.344E-05 & 1.197E+02 & 0.000E+00 \\ 
  FWL minus Ours & 1.520E-03 & 3.344E-05 & 4.545E+01 & 0.000E+00 \\ 
   \hline
\end{tabular}
\end{table}

\begin{table}[ht]
\setlength{\tabcolsep}{2pt}
\caption{Dunnett's test (one-sided) \wrt Table 5, Logloss, {\tt Criteo} dataset.}
\centering
\begin{tabular}{rrrrr}
  \hline
 & Estimate & Std. Error & t value & Pr($>$t) \\ 
  \hline
WD minus Ours & 3.713E-03 & 1.319E-03 & 2.815E+00 & 2.204E-02 \\ 
  DeepFM minus Ours & 5.140E-03 & 1.319E-03 & 3.898E+00 & 1.376E-03 \\ 
  NODE minus Ours & 3.099E-02 & 1.319E-03 & 2.350E+01 & 0.000E+00 \\ 
  Tabnet minus Ours & 2.217E-02 & 1.319E-03 & 1.681E+01 & 0.000E+00 \\ 
  TabTran minus Ours & 6.374E-03 & 1.319E-03 & 4.833E+00 & 9.159E-05 \\ 
  SAINT minus Ours & 7.554E-03 & 1.319E-03 & 5.728E+00 & 4.669E-06 \\ 
  FMLP minus Ours & 7.637E-03 & 1.319E-03 & 5.790E+00 & 1.405E-05 \\ 
   \hline
\end{tabular}
\end{table}

\begin{table}[ht]
\setlength{\tabcolsep}{2pt}
\caption{Dunnett's test (one-sided) \wrt Table 3, Logloss, {\tt KDD2012} dataset.}
\centering
\begin{tabular}{rrrrr}
  \hline
 & Estimate & Std. Error & t value & Pr($>$t) \\ 
  \hline
LR minus Ours & 6.618E-03 & 3.044E-05 & 2.174E+02 & 0.000E+00 \\ 
  FM minus Ours & 2.899E-03 & 3.044E-05 & 9.522E+01 & 0.000E+00 \\ 
  FFM minus Ours & 3.508E-03 & 3.044E-05 & 1.152E+02 & 0.000E+00 \\ 
  xDFM minus Ours & 2.294E-03 & 3.044E-05 & 7.534E+01 & 0.000E+00 \\ 
  AFN minus Ours & 1.935E-03 & 3.044E-05 & 6.355E+01 & 0.000E+00 \\ 
  IFM minus Ours & 7.236E-04 & 3.044E-05 & 2.377E+01 & 0.000E+00 \\ 
  INN minus Ours & 8.719E-03 & 3.044E-05 & 2.864E+02 & 0.000E+00 \\ 
  LLFM minus Ours & 2.800E-03 & 3.044E-05 & 9.196E+01 & 0.000E+00 \\ 
  RaFM minus Ours & 4.628E-03 & 3.044E-05 & 1.520E+02 & 0.000E+00 \\ 
  FWL minus Ours & 2.136E-03 & 3.044E-05 & 7.015E+01 & 0.000E+00 \\ 
   \hline
\end{tabular} 
\end{table}

\begin{table}[ht]
\setlength{\tabcolsep}{2pt}
\caption{Dunnett's test (one-sided) \wrt Table 5, Logloss, {\tt KDD2012} dataset.}
\centering
\begin{tabular}{rrrrr}
  \hline
 & Estimate & Std. Error & t value & Pr($>$t) \\ 
  \hline
WD minus Ours & 1.625E-03 & 1.554E-04 & 1.046E+01 & 1.400E-09 \\ 
  DeepFM minus Ours & 1.812E-03 & 1.554E-04 & 1.166E+01 & 3.488E-11 \\ 
  NODE minus Ours & 7.372E-03 & 1.554E-04 & 4.745E+01 & 0.000E+00 \\ 
  Tabnet minus Ours & 1.079E-02 & 1.554E-04 & 6.942E+01 & 0.000E+00 \\ 
  TabTran minus Ours & 6.774E-03 & 1.554E-04 & 4.360E+01 & 0.000E+00 \\ 
  SAINT minus Ours & 4.295E-03 & 1.554E-04 & 2.764E+01 & 0.000E+00 \\ 
  FMLP minus Ours & 1.465E-03 & 1.554E-04 & 9.427E+00 & 3.801E-08 \\ 
   \hline
\end{tabular}
\end{table}

\begin{table}[ht]
\setlength{\tabcolsep}{2pt}
\caption{Dunnett's test (one-sided) \wrt Table 4, Logloss, {\tt ML-Tag} dataset.}
\centering
\begin{tabular}{rrrrr}
  \hline
 & Estimate & Std. Error & t value & Pr($>$t) \\ 
  \hline
LR minus Ours & 1.227E-01 & 4.351E-04 & 2.820E+02 & 0.000E+00 \\ 
  FM minus Ours & 2.713E-02 & 4.351E-04 & 6.236E+01 & 0.000E+00 \\ 
  FFM minus Ours & 2.124E-02 & 4.351E-04 & 4.882E+01 & 0.000E+00 \\ 
  xDFM minus Ours & 3.775E-02 & 4.351E-04 & 8.676E+01 & 0.000E+00 \\ 
  AFN minus Ours & 4.439E-02 & 4.351E-04 & 1.020E+02 & 0.000E+00 \\ 
  IFM minus Ours & 6.052E-02 & 4.351E-04 & 1.391E+02 & 0.000E+00 \\ 
  INN minus Ours & 6.809E-02 & 4.351E-04 & 1.565E+02 & 0.000E+00 \\ 
  LLFM minus Ours & 4.501E-02 & 4.351E-04 & 1.035E+02 & 0.000E+00 \\ 
  RaFM minus Ours & 4.064E-02 & 4.351E-04 & 9.341E+01 & 0.000E+00 \\ 
  FWL minus Ours & 2.880E-02 & 4.351E-04 & 6.621E+01 & 0.000E+00 \\ 
   \hline
\end{tabular}
\end{table}

\begin{table}[ht]
\setlength{\tabcolsep}{2pt}
\caption{Dunnett's test (one-sided) \wrt Table 5, Logloss, {\tt ML-Tag} dataset.}
\centering
\begin{tabular}{rrrrr}
  \hline
 & Estimate & Std. Error & t value & Pr($>$t) \\ 
  \hline
WD minus Ours & 6.193E-02 & 1.312E-03 & 4.722E+01 & 0.000E+00 \\ 
  DeepFM minus Ours & 4.385E-02 & 1.312E-03 & 3.344E+01 & 0.000E+00 \\ 
  NODE minus Ours & 8.995E-02 & 1.312E-03 & 6.858E+01 & 0.000E+00 \\ 
  Tabnet minus Ours & 9.838E-02 & 1.312E-03 & 7.501E+01 & 0.000E+00 \\ 
  TabTran minus Ours & 8.086E-02 & 1.312E-03 & 6.165E+01 & 0.000E+00 \\ 
  SAINT minus Ours & 1.021E-01 & 1.312E-03 & 7.787E+01 & 0.000E+00 \\ 
  FMLP minus Ours & 4.093E-02 & 1.312E-03 & 3.121E+01 & 0.000E+00 \\ 
   \hline
\end{tabular}
\end{table}

\begin{table}[ht]
\setlength{\tabcolsep}{2pt}
\caption{Dunnett's test (one-sided) \wrt Table 4, Logloss, {\tt Frappe} dataset.}
\centering
\begin{tabular}{rrrrr}
  \hline
 & Estimate & Std. Error & t value & Pr($>$t) \\ 
  \hline
LR minus Ours & 1.892E-01 & 2.554E-03 & 7.406E+01 & 0.000E+00 \\ 
  FM minus Ours & 2.961E-03 & 2.554E-03 & 1.159E+00 & 4.656E-01 \\ 
  FFM minus Ours & 1.331E-02 & 2.554E-03 & 5.212E+00 & 2.281E-05 \\ 
  xDFM minus Ours & 1.693E-02 & 2.554E-03 & 6.627E+00 & 7.710E-08 \\ 
  AFN minus Ours & 9.372E-03 & 2.554E-03 & 3.669E+00 & 2.668E-03 \\ 
  IFM minus Ours & 1.035E-02 & 2.554E-03 & 4.054E+00 & 8.672E-04 \\ 
  INN minus Ours & 8.327E-02 & 2.554E-03 & 3.260E+01 & 0.000E+00 \\ 
  LLFM minus Ours & 1.107E-02 & 2.554E-03 & 4.334E+00 & 4.068E-04 \\ 
  RaFM minus Ours & -1.945E-03 & 2.554E-03 & -7.614E-01 & 9.885E-01 \\ 
  FWL minus Ours & 1.067E-02 & 2.554E-03 & 4.179E+00 & 5.820E-04 \\ 
   \hline
\end{tabular}
\end{table}

\begin{table}[ht]
\setlength{\tabcolsep}{2pt}
\caption{Dunnett's test (one-sided) \wrt Table 4, MSE, {\tt ML-Rating} dataset.}
\centering
\begin{tabular}{rrrrr}
  \hline
 & Estimate & Std. Error & t value & Pr($>$t) \\ 
  \hline
LR minus Ours & 8.887E-02 & 1.051E-03 & 8.458E+01 & 0.000E+00 \\ 
  FM minus Ours & 1.424E-02 & 1.051E-03 & 1.356E+01 & 0.000E+00 \\ 
  FFM minus Ours & 2.936E-02 & 1.051E-03 & 2.795E+01 & 0.000E+00 \\ 
  xDFM minus Ours & 2.335E-02 & 1.051E-03 & 2.223E+01 & 0.000E+00 \\ 
  AFN minus Ours & 3.198E-02 & 1.051E-03 & 3.044E+01 & 0.000E+00 \\ 
  IFM minus Ours & 1.533E-02 & 1.051E-03 & 1.459E+01 & 0.000E+00 \\ 
  INN minus Ours & 1.335E-02 & 1.051E-03 & 1.271E+01 & 0.000E+00 \\ 
  LLFM minus Ours & 1.094E-02 & 1.051E-03 & 1.041E+01 & 4.219E-14 \\ 
  RaFM minus Ours & 2.834E-02 & 1.051E-03 & 2.697E+01 & 0.000E+00 \\ 
  FWL minus Ours & 2.818E-02 & 1.051E-03 & 2.682E+01 & 0.000E+00 \\ 
   \hline
\end{tabular} 
\end{table}

\bibliographystylelatex{IEEEtranN}
\bibliographylatex{latex}

\end{document}